%% file: main.tex
\documentclass[11pt, a4paper, logo, copyright]{pluralisresearch}
\input{math_commands.tex}

\pdftrailerid{redacted}

\makeatletter
\renewcommand\bibentry[1]{\nocite{#1}{\frenchspacing\@nameuse{BR@r@#1\@extra@b@citeb}}}
\makeatother

\usepackage{kantlipsum, lipsum}
\usepackage{dsfont}
\usepackage{algorithm}
\usepackage{algpseudocode}
\usepackage{adjustbox}
\usepackage{multirow}
\usepackage{nicefrac}

\usepackage{wrapfig}

\usepackage{microtype}
\usepackage{graphicx}
\usepackage{float}
\usepackage{booktabs} 
\usepackage{bm}      

\newcommand{\A}{\mathbf{A}}
\newcommand{\B}{\mathbf{B}}

\newcommand{\I}{\mathbf{I}}

\newcommand{\M}{\mathbf{M}}

\newcommand{\Pp}{\mathbf{P}} 

\newcommand{\T}{\mathbf{T}}
\newcommand{\U}{\mathbf{U}}
\newcommand{\V}{\mathbf{V}}
\newcommand{\W}{\mathbf{W}}
\newcommand{\X}{\mathbf{X}}

\newcommand{\g}{\mathbf{g}}

\usepackage{hyperref}

\usepackage{amsmath}
\usepackage{amssymb}
\usepackage{mathtools}
\usepackage{amsthm}
\usepackage{subcaption}
\usepackage{float}
\usepackage{multirow} 
\usepackage[capitalize,noabbrev]{cleveref}

\theoremstyle{plain}
\newtheorem{theorem}{Theorem}[section]
\newtheorem{proposition}[theorem]{Proposition}

\newtheorem{lemma}[theorem]{Lemma}

\theoremstyle{definition}

\theoremstyle{remark}

\usepackage{booktabs} 
\usepackage{array}    

\usepackage[textsize=tiny]{todonotes}
\usepackage[most]{tcolorbox}
\tcbuselibrary{theorems}     

\usepackage{tcolorbox}
\usepackage{cleveref}

\newtcolorbox[auto counter,number within=section,crefname={box}{boxes}]{pabox}[2][]{%
title=Statement 1 
colback=blue!5!white,
colframe=blue!75!black,
colbacktitle=cyan!50!green,
coltitle=blue!25!black,
fonttitle=\bfseries,
subtitle style={boxrule=0.4pt,colback=cyan!50!red!25!white},title=Statement ~\thetcbcounter, label={#1}}

\usepackage[utf8]{inputenc} 
\usepackage[T1]{fontenc}    
\usepackage{hyperref}       
\usepackage{url}            
\usepackage{booktabs}       
\usepackage{amsfonts}       
\usepackage{nicefrac}       
\usepackage{microtype}      

\newcolumntype{R}[2]{%
    >{\adjustbox{angle=#1,lap=\width-(#2)}\bgroup}%
    l%
    <{\egroup}%
}

\usepackage{pifont}

\usepackage{subcaption}
\usepackage{url}

\graphicspath{{figures/}}

\title{Subspace Networks: Scaling Decentralized Training with Communication-Efficient Model Parallelism}

\correspondingauthor{sameera@puralis.ai}


\reportnumber{} 

\renewcommand{\today}{} 

\author{Sameera Ramasinghe, Thalaiyasingam Ajanthan, Gil Avraham, Yan Zuo, and Alexander Long\\
Pluralis Research
}

\begin{abstract}
    Scaling models has led to significant advancements in deep learning, but training these models in decentralized settings remains challenging due to communication bottlenecks. While existing compression techniques are effective in data-parallel, they do not extend to model parallelism. Unlike data-parallel training, where weight gradients are exchanged, model-parallel requires compressing activations and activation gradients as they propagate through layers, accumulating compression errors. We propose a novel  compression algorithm that  compresses both forward and backward passes, enabling up to $99\%$ compression with no convergence degradation with negligible memory/compute overhead. By leveraging a recursive structure in transformer networks, we predefine a low-dimensional subspace to confine the activations and gradients, allowing full reconstruction in subsequent layers. Our method achieves up to $100\times$ improvement in communication efficiency and enables training billion-parameter-scale models over low-end GPUs connected via consumer-grade internet speeds as low as $80$Mbps, matching the convergence of centralized datacenter systems with $100$Gbps connections with model parallel. 
\end{abstract}

\begin{document}
\maketitle

\section{Introduction}
\label{sec:intro}


Scaling models and datasets has been pivotal in driving deep learning advancements, with model sizes expanding from millions of parameters~\cite{krizhevsky2012imagenet} to billions~\cite{kolesnikov2020big, dubey2024llama} and even trillions~\cite{ren2023pangu}. These larger models exceed the memory capacity of a single device, requiring distributed training approaches to manage computation across multiple devices.


A common solution is distributed data parallelism (DDP)~\cite{li2020pytorch} or its more advanced variant, fully sharded data parallelism (FSDP)~\cite{zhao2023pytorch}, which distributes data across nodes while replicating the model on each device. This enables larger batch sizes and higher throughput, but constrains the model size to the memory of a single device. Model parallelism (MP) addresses this limitation by distributing parameters across devices~\cite{krizhevsky2012imagenet, huang2019gpipe}, enabling the training of models that surpass single-node memory constraints. MP includes tensor parallelism, which splits individual layers, and pipeline parallelism, which distributes layers across devices; the latter is the focus of this work. Modern large-scale training combines DDP and MP to achieve scalability. Despite these strategies, all approaches require the transfer of large amounts of data between devices~\cite{narayanan2021efficient}, limiting training to high-performance computing clusters with fast interconnects. These infrastructures are costly and accessible only to resource-rich organizations~\cite{vergara2019scaling, powerai}, creating disparities that restrict broader research and risk centralizing innovation.

Decentralized training provides an alternative by leveraging consumer-grade devices, enabling individuals with small devices to participate in large-scale model training, reducing reliance on major corporations. This approach democratizes access to large-scale model training by leveraging underutilized GPUs in personal computers and volunteer networks~\cite{yuan2022decentralized, ryabinin2023swarm}. However, decentralized training faces significant challenges due to limited bandwidth and high latency in heterogeneous networks~\cite{yuan2022decentralized}, necessitating communication-efficient compression algorithms to minimize data transfer while preserving training performance.


Most existing communication-efficient techniques focus on DDP~\cite{lian2017can, ryabinin2021moshpit, douillard2023diloco, peng2024decoupled}, where \emph{weight gradients} are computed independently on each node (for each model replica) and then compressed before peer-to-peer communication. To this end, techniques such as sparsification~\cite{wangni2018gradient, wang2017efficient, lin2017deep}, quantization~\cite{wang2023cocktailsgd, alistarh2017qsgd, bernstein2018signsgd}, and low-rank approximations~\cite{zhao2024galore} exploit the redundancy in weight gradients to reduce communication. However, in MP, information must be passed between \emph{layers}, requiring the communication of \emph{activations} and \emph{activation gradients}. Unlike weight gradients, activations lack inherent redundancy and approximation errors accumulate across layers, leading to degraded convergence~\cite{bian2024does, rudakov2023activations}. These challenges prevent the straightforward application of DDP compression techniques to MP, and hence to date, MP decentralized training remains infeasible, resulting in massive slowdowns over centralized training.

To bridge this gap, we propose a novel compression algorithm tailored for MP. We show that as training progresses, weight matrices exhibit rank collapse, converging to low-rank subspaces. Thus, by \emph{explicitly} constraining specific weight matrices to such low-rank subspaces and leveraging a recursive structure inherent in transformer networks, we demonstrate that transformer layer activations—despite their high rank—can be decomposed into a dynamic low-rank component and a static high-rank component. This decomposition enables efficient compression of information passed between layers during both the forward and backward passes, ensuring \emph{lossless} reconstruction in subsequent layers.



We validate the practical effectiveness of our approach through extensive evaluations on billion-parameter-scale models. Our compression method enables the distribution of large-scale models across consumer-grade GPUs with internet-grade connections (80Mbps) while matching the convergence performance of centralized setups with 100 Gbps interconnects. We achieve up to $100 \times$ improvement in communication efficiency without any degradation in convergence. Further, we successfully train an 8B-parameter LLaMA model \cite{dubey2024llama} with layers split across four different geographical regions, connected via the internet, and achieve convergence on par with baselines utilizing datacenter-grade connections. By addressing critical limitations in decentralized training, our method intend to remove significant barriers to scaling large models in resource-constrained environments, democratizing access to large-scale deep learning.

\section{Related works}

\paragraph{Decentralized training} involves a group of \emph{autonomous} devices (\textit{i.e.,} no central orchestrator) collaborating to train a large-scale model by leveraging MP/DDP methods. These devices, often geographically distributed and heterogeneous, are connected via networks with varying delays and bandwidth constraints. Key advancements in this area encompass both theoretical insights~\cite{lian2017can, koloskova2019decentralized, koloskova2020unified} as well as practical approaches~\cite{ryabinin2020towards,diskin2021distributed}. Despite the progress, current efforts predominantly focus on DDP \cite{lian2017can, koloskova2019decentralized, koloskova2020unified, diskin2021distributed}, which hinders scaling models beyond the memory capacity of local nodes. SWARM parallelism~\cite{ryabinin2023swarm} and Tasklets~\cite{yuan2022decentralized}, treat the problem as a scheduling challenge, with the former addressing the stochasticity inherent in decentralized cluster settings. However, all methods to date still face significant scalability challenges due to communication bottlenecks inherent in MP. Our method overcomes this limitation by introducing an effective communication compression technique for MP (specifically pipeline parallel), mitigating a major barrier in scaling decentralized training.

\paragraph{Communication compression} accelerates distributed training over bandwidth-limited networks by reducing data transfers. Key strategies include \textbf{sparsification}, which transmits only significant parameter updates~\cite{wangni2018gradient, wang2017efficient, lin2017deep}; \textbf{quantization}, which lowers communication by reducing parameter precision~\cite{dettmers20218, wang2023cocktailsgd, alistarh2017qsgd, bernstein2018signsgd, karimireddy2019error, tang20211, wu2018error}; and \textbf{low-rank projection}, which compresses gradients via projection onto lower-dimensional subspaces~\cite{zhao2024galore, vogels2019powersgd}. While successful in data-parallel settings (DDP), these techniques face difficulties in model-parallel (MP) setups, including error accumulation across layers and unstructured activations~\cite{bian2024does, rudakov2023activations}, causing degraded convergence. Recent work by~\cite{wang2021pufferfish} proposed low-rank MP communication, but required significant architectural changes, preventing training from scratch. In contrast, our approach involves minimal initialization changes without architectural modifications, enabling efficient MP training from scratch with improved scalability.

\section{Transformer block}

We provide a brief exposition of the transformer block and proceed to describe the proposed compression method. Let the input to the $l^{\text{th}}$ layer be $\X^{l} \in \mathbb{R}^{b \times n \times d}$, where $b$, $n$, and $d$ are the batch size, sequence length, and embedding dimension, respectively. Given the weight matrices of each attention head $h$ as $\W^l_{Q,h}, \W^l_{K,h}, \W^l_{V,h} \in \mathbb{R}^{d \times d_H}$, with $d_H = d/H$ where $H$ is the number of attention heads, the following computations are performed for each attention head: $\X^l_{Q,h} = \X^l \W_{Q,h}^l$, $\X^l_{K,h} = \X^l \W_{K,h}^l$, ${\X^l_{V,h} = \X^l \W_{V,h}^l}$. The rest of the computations are as follows:
\begin{equation}
\X^l_{\text{head, h}} = f_{\text{softmax}}\left(\frac{\X^l_{Q,h}\X_{K,h}^\top}{\sqrt{d_H}}\right)\X^l_{V,h}
\label{eq:trans_head}
\end{equation}
%
%
\begin{equation}\label{eq:transformer_layer}
\begin{aligned}
\X^l_{\text{concat}} &= [\X^l_{\text{head, 1}}, \dots, \X^l_{\text{head, H}}] \\
\X^l_{\text{attn}} &= \X^l_{\text{concat}}\W^l_{p_1} + \X^{l} \\
\X^l_{\text{hidden}} &= f_{\text{relu}}(\X^l_{\text{attn}}\W_1^l) \\
\X^{l+1} &= \X^l_{\text{hidden}}\W_{p_2}^l + \X^l_{\text{attn}}
\end{aligned}
\end{equation}
%
where $\W^l_{p_1} \in \mathbb{R}^{d \times d}, \W_1^l \in \mathbb{R}^{d \times d_{\text{ff}}}$, and $\W_{p_2}^l \in \mathbb{R}^{d_{\text{ff}} \times d}$.  We omit the layer norms for brevity, which does not affect any of our derivations. $d_{\text{ff}}$ is usually an integer multiple of $d$. We will refer to $\W_{p_2}^l$ and $\W_{p_1}^l$ as \emph{projection matrices} from here onward.

\begin{figure}
  \centering
  \includegraphics[width=0.5\textwidth]{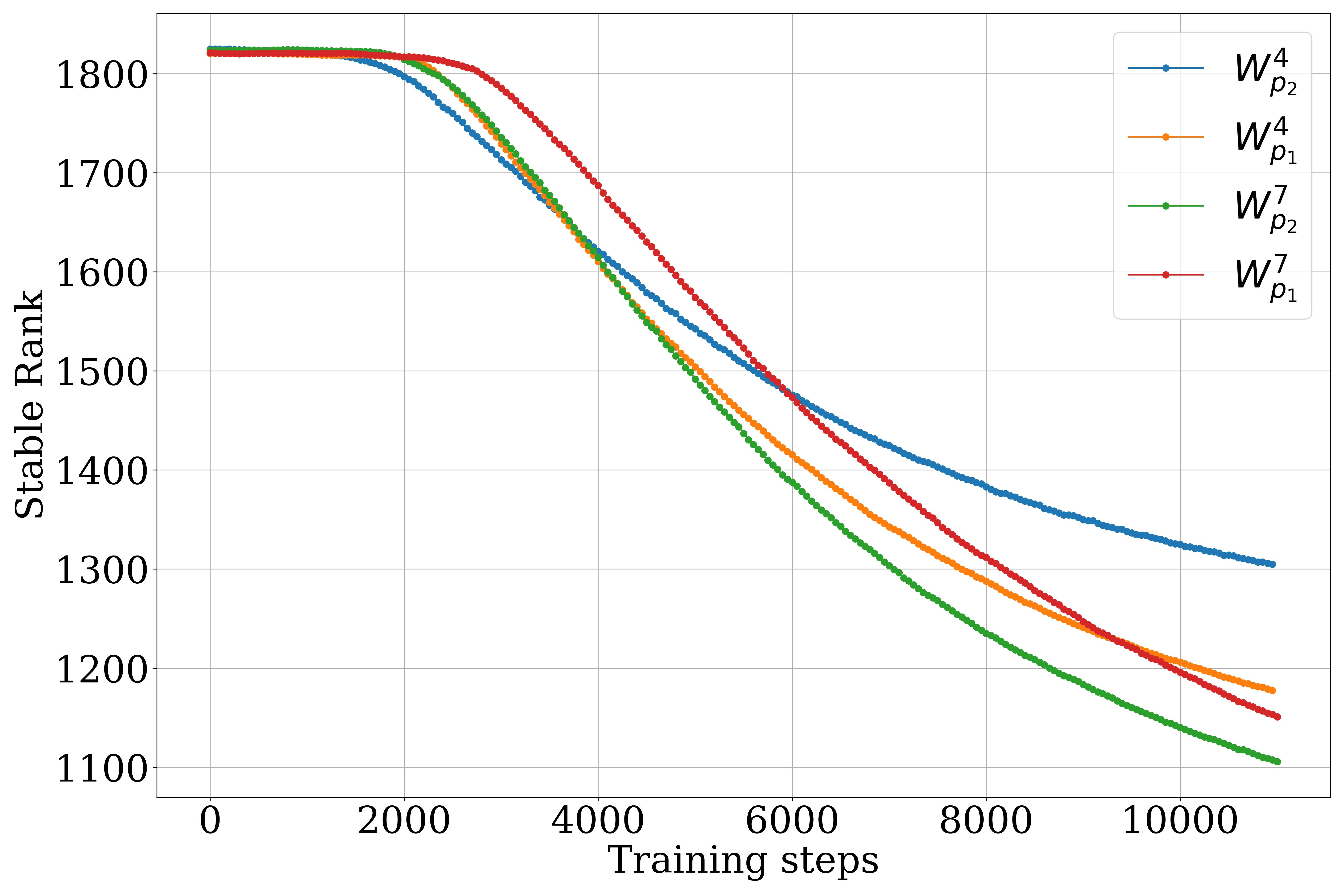}
  \caption{\textbf{Rank collapse in projection matrices.} Consistent with Statements~\ref{st:2} and~\ref{st:3} (and Theorems~\ref{th:weight_update},~\ref{th:weight_convergence} - Appendix), we empirically observe a natural rank collapse in the projection matrices (of non-compressed models). Shown is an 8-layer, 2B-parameter model, with the stable (effective) ranks of the projection matrices for the $4^{\text{th}}$ (middle) and $7^{\text{th}}$ (penultimate) layers plotted over training steps.}
  \label{fig:rank_collapse}
\end{figure}

\section{Subspace networks}

\subsection{Rank collapse of projection matrices}




We leverage the observation that weight gradients of projection matrices inherently reside in a low-dimensional subspace, as widely reported across various architectures~\cite{li2024learning, vogels2019powersgd, wang2018atomo, zhang2024convergence, zhao2021zero, cosson2023low, yang2023spectral, gur2018gradient} (see Appendix~\ref{sec:app_assumptions} for empirical validation). Together with AdamW's decoupled weight decay—which suppresses negligible gradient components—this drives projection matrices to converge toward a subspace spanned by dominant gradients, resulting in effectively low-rank structures. Formal treatment is provided in Statements~\ref{st:2},~\ref{st:3}, and their associated theorems.



To validate this phenomenon, we train an 8-layer, 2B-parameter model on WikiText~\cite{wikitext} and track the stable rank of its projection layers during training (hidden dimension of 4096 and a context length of 2048. Stable rank is computed as $\sum_i \sigma_i^2 / \max_i(\sigma_i^2)$, where $\sigma_i$ denotes the $i^{\text{th}}$ singular value. As shown in Fig.~\ref{fig:rank_collapse}, the stable rank of projection matrices sharply declines, consistent with our theory. This structure, combined with the recursive nature of transformers, allows us to design a nearly-lossless low-rank compression scheme. While prior works observe similar rank collapse in self-attention matrices~\cite{dong2021attention, abbe2024transformers, sanyal2404inheritune}, we focus specifically on projection matrices as the basis for compression. We also validate this using official checkpoints of fully-trained large scale models (Appendix \ref{app:checkpoints}).

\subsection{Investigating the activation structure}
\label{ssec_trans}

We utilize the natural rank collapse we discussed thus far for compression. Observing the transformer block in Eq. \ref{eq:trans_head} and \ref{eq:transformer_layer}, we identify that a recursive structure emerges for layer outputs due to the skip connections:
\begin{equation}
    \X^{l+1} = \X^l_{\text{hidden}}\W_{p_2}^l  + \X^l_{\text{concat}}\W^l_{p_1} + \X^{l}.
\end{equation}
which can be expanded as:
\begin{align}
\X^{l+1} = \sum_{i=1}^{l} (\X^i_{\text{hidden}}\W_{p_2}^i  + \X^i_{\text{concat}}\W^i_{p_1}) + \X^{0}
\label{eq:recursive}
\end{align}
Here, 
\begin{equation}
    \X^0 = \text{PE} + \text{TE},
\end{equation}
where $\text{PE}, \text{TE} \in \mathbb{R}^{b \times n \times d}$ represent the positional and token embeddings, respectively. Let $[p_1, p_2, \dots, p_n]$ be the sequence of positional indices, and $[t_1, t_2, \dots, t_n]$ be the corresponding token indices. We denote embedding matrices $\mathbf{P} \in \mathbb{R}^{n \times d}$ for positional embeddings and $\mathbf{T} \in \mathbb{R}^{v \times d}$ for token embeddings, where $v$ is the vocabulary length. Thus the tokens and positional indices are embedded via a lookup: $\text{PE}= \Pp[p_{1:n},:]$ and $\text{TE}=\T[t_{1:n},:]$.
%
%
Observing Eq.~\ref{eq:recursive}, we note that $\X^0 = \text{PE} + \text{TE}$ contributes as a common additive term to all layer outputs. Therefore, we consider the rank of the residual activations when $\text{PE}$ and $\text{TE}$ are subtracted:
\begin{align} \label{eq:recursive2}
\hat{\X}^{l+1}  =  \X^{l+1} - \text{PE} - \text{TE}    = \sum_{i=1}^{l} (\X^i_{\text{hidden}}\W_{p_2}^i  + \X^i_{\text{concat}}\W^i_{p_1})
\end{align}
Recall that $\text{Row}(\A \B) \subseteq \text{Row}(\B)$, for any two matrices $\A, \B$ where $\text{Row}(\cdot)$ denotes the row space. Thus, it is clear that if the rows of the projection weights \((\W_{p_2}, \W_{p_1})\) up to layer $l$ span a common low-dimensional subspace, then the rows of \(\hat{\X}^{l+1}\) is also restricted to the same subspace, since vector spaces are closed under addition.

\subsection{Compressing the forward pass}


Recall that our analysis so far indicates that the  projection matrices naturally confine themselves to a smaller subspace as training progresses. Consequently, the residual activations $\hat{\X}^{l+1}$ are also restricted to a smaller subspace, if the union of those subspaces is low-dimensional.  Based on this insight, it is intriguing to explore the feasibility of \emph{explicitly forcing} the rows of projection matrices to vary within a common low-dimensional subspace $\mathcal{S}$, throughout training, to facilitate activation compression. As discussed in Section \ref{ssec_trans}, this forces the rows of activation outputs ($\hat{\X}^{l+1}$) to span the same subspace $\mathcal{S}$. Surprisingly, we find that even with extreme low-dimensional $\mathcal{S}$, the networks can achieve almost the same convergence rates as in the unaltered ones. This allows us to significantly reduce the communication between blocks, which we show next.

Let $\text{Row}(\W^l_{p_2}), \text{Row}(\W^l_{p_1}) \subseteq \mathcal{S}$. Further, Let ${\U_k \in \mathbb{R}^{d \times k}}$ be a matrix with orthonormal columns and $\text{Col}(\U_k) = \mathcal{S}$. Then, the following holds:
\begin{align}
    \hat{\X}^{l+1} = \hat{\X}^{l+1}\U_k \U_k^\top
\end{align}
%
In other words, $\hat{X}^{l+1}$ remains unaltered by the projection, since it is already in $\mathcal{S}$. The above formulation introduces a property that can be leveraged for compression during the forward pass. Specifically, the dimensionality of \(\hat{\X}^{l+1} \U_k \in \mathbb{R}^{b \times n \times k}\), is substantially smaller than that of \(\hat{\X}^{l+1} \in \mathbb{R}^{b \times n \times d}\) since \(k \ll d\). If each node in a distributed system is initialized with a shared copy of \(\U_k\), \emph{this matrix does not need to be transmitted repeatedly}. Instead, \(\hat{\X}^{l+1} \U_k\) can  be transmitted to the next node, and the original \(\X^{l+1}\) is reconstructed as:
\[
\X^{l+1} = \hat{\X}^{l+1} \U_k^\top + \text{PE} + \text{TE}.
\]
This approach ensures exact recovery of \(\X^{l+1} \) without approximation.

\subsubsection{Decomposition of high-rank components}

For practical utilization of above approach, we still need to subtract $\text{PE}$ and $\text{TE}$ from $\X^{l+1}$ to compute $\hat{\X}^{l+1}$ (and add them back in the next layer). While $\text{PE}$ is deterministic and can be computed locally within each node, $\text{TE}$ varies depending on the batch, making it impossible to do so. 

One potential solution is restricting the embedding table $\T$ also to $\mathcal{S}$. However, we observed that this degrades network performance due to severely limiting the representation capacity of the token embeddings. Instead, we propose modeling $\text{TE}$ as a composition of a fixed high-rank component and a trainable low-rank component:
\[
\text{TE} = \T_{\text{fixed}}[t_{1:n}, :] + \T_{\mathcal{S}}[t_{1:n}, :],
\]
where we obtain a trainable low rank embedding table ${\T_{\mathcal{S}}  = \T_{\text{fixed}}\U_k\U_k^T}$. At the beginning of training, $\T_{\text{fixed}}$ is transmitted to all nodes and stored. During the forward pass, we compress the activations as:
\begin{align}
    \X^{l+1}_{\text{compressed}} = & \big(\hat{\X}^{l+1} + \T_{\mathcal{S}}[t_{1:n}, :]\big)\U_k  \nonumber \\= 
& (\X^{l+1} - \text{PE} - \T_{\text{fixed}}[t_{1:n}, :]) \U_k,
\label{eq:subtract}
\end{align}
 Note that in Eq.~\ref{eq:subtract} both $\text{PE}$ and $ \T_{\text{fixed}}[t_{1:n}, :]$---which are generally high-rank---are subtracted from $X^{l+1}$ so the remaining $(\X^{l+1} - \text{PE} -  \T_{\text{fixed}}[t_{1:n}, :]) $ is already in $\mathcal{S}$ and is low-rank. Further, since $\text{Row}( \T_{\mathcal{S}}[t_{1:n}, :])  \subseteq \text{Col}(\U_k)$, it is implicitly captured in the compressed activations $\X^{l+1}_{\text{compressed}}$. Reconstruction at the next node is then performed as:
\begin{align}
    \X^{l+1}_{\text{recovered}} =  \X^{l+1}_{\text{compressed}} \U_k^\top + \text{PE} +  \T_{\text{fixed}}[t_{1:n}, :]   \nonumber  =  \X^{l+1}
\end{align}
ensuring a \textbf{lossless} recovery of $\X^{l+1}$ while not compromising its high ranked-ness. 


A natural question arises: would explicitly restricting the projection matrices to a fixed $\mathcal{S}$, instead of allowing this property to emerge organically, adversely affect convergence? Note that this is a form of constraint optimization and there are well known convergence guarantees. However, in  Sec.~\ref{sec:theory}, we provide a convergence proof (to at least a first-order stationary point) for completeness for the above partial projection case. Furthermore, we conduct extensive experiments across a variety of settings to empirically validate the convergence. Further, since the fixed embeddings are ephemeral, they have a negligible effect on the peak GPU memory (Appendix \ref{app:memory_overhead})

\subsection{Compression in backpropagation}

In the previous section, we showed that constraining the rows of projection matrices to a shared low-dimensional subspace \(\mathcal{S}\), coupled with decomposing the embedding table into low-rank and high-rank components, facilitates compression of activations in the forward pass. This same constraint naturally facilitates lossless gradient compression in the backward pass. Specifically, let \(\nabla_L (\X^{l+1})\) denote the gradient of \(\X^{l+1}\), with respect to the loss, that needs to be propagated to the previous layer. This gradient can be compressed as:
\begin{align} \label{eq:grad_compress}
    \big(\nabla_L (\X^{l+1})\big)_{\text{compressed}} = \nabla_L (\X^{l+1}) \U_k \in \mathbb{R}^{b \times n \times k},
\end{align}
and subsequently fully recovered in the previous layer $l$ as:
\begin{align}
    \big(\nabla_L (\X^{l+1})\big)_{\text{recovered}} = \big(\nabla_L (\X^{l+1})\big)_{\text{compressed}} \U_k^\top 
    =  \nabla_L (\X^{l+1}).
\end{align}
Remarkably, this formulation ensures that the gradient flow to the computational graph prior to $\X^{l+1}$ remains lossless, with no approximation error. Intuitively, after backpropagation through the parameter matrix $\W_{p_2}$, the gradient has the form $\nabla_L (\X^{l+1})\W_{p_2}^{\top}$. Because ${\text{Row(}\W_{p_2}) \subseteq \text{Col}(\U_k) = \mathcal{S}}$, the projection $\nabla_L (\X^{l+1})\U_k\U_k^\top\W_{p_2}^{\top}$ does not alter the resulting gradient flow. Full derivation is provided in Appendix \ref{sec:abl_gradients}.



\subsection{Subspace updates using Grassmann manifold}

We observe that restricting the column spaces of projection layers to a fixed subspace, even at high-ratios, is able to maintain surprisingly adequate convergence. To further improve the convergence, we allow the subspaces to slowly drift. To align the subspace with the gradient directions, we minimize the norm of the gradient components that lie outside the subspace, as measured at the last Transformer layer. Let \( \nabla_L (\X^{\text{final}}_t)  \in \mathbb{R}^{b \times n \times d} \) denote the activation gradients at the last compressed transformer layer. The leftover gradient, which lies outside the subspace, is $ \hat{\X}^{\text{final}}_t = \nabla_L (\X^{\text{final}}_t)(\I - \mathbf{U}_k \mathbf{U}_k^\top)$, where \( \I - \mathbf{U}_k \mathbf{U}_k^\top \) is the projection operator onto $\mathcal{S}^{\perp}$. We accumulate $ \hat{\X}^{\text{final}}_t$ over $K$ iterations to obtain the metric $\mathcal{L}_{\text{Grassmann}} = \frac{1}{K}\sum_{t = k}^{k+K}  \|  \hat{\X}^{\text{final}}_t\|_F^2$, where \( \|\cdot\|_F \) is the Frobenius norm. We aim to minimize $\mathcal{L}_{\text{Grassmann}}$ over all possible $\U_k$.


A straightforward way to minimize \( \mathcal{L}_{\text{Grassmann}} \) is by performing SVD on it and updating the subspace using the left singular vectors. However, abrupt changes to the subspace can disrupt convergence. Thus, we perform smooth updates by taking steps on the Grassmann manifold. The Grassmann manifold \( \mathcal{G}(k, n) \) is the set of all \( k \)-dimensional subspaces of \( \mathbb{R}^n \). A point on \( \mathcal{G}(k, n) \) is represented by an orthonormal matrix \( \mathbf{U}_k \in \mathbb{R}^{n \times k} \), where the columns of \( \mathbf{U}_k \) form a basis for the subspace. Thus, defining $\mathcal{S}$ as a point on the Grassmann manifold enables taking smooth steps on the manifold. To minimize \( \mathcal{L}_{\text{Grassmann}}\), we employ gradient descent on \( \mathcal{G}(k, n) \). First, the Euclidean gradient of \(\mathcal{L}_{\text{Grassmann}} \) with respect to \( \mathbf{U}_k \), denoted \(\nabla_{\mathcal{L}_{\text{Grassmann}}}( \mathbf{U}_k )\), is projected onto the tangent space to obtain the Riemannian gradient:
\begin{equation}
    (\nabla_\mathcal{L}\mathbf{U}_k)_{\text{tangent}} = \nabla_{\mathcal{L}_{\text{Grassmann}}}( \mathbf{U}_k ) - \mathbf{U}_k \mathbf{U}_k^\top \nabla_{\mathcal{L}_{\text{Grassmann}}}( \mathbf{U}_k).
\end{equation}
 Then, we perform a gradient descent step ${\mathbf{U}^{\text{new}}_k = \mathbf{U}_k - \eta \,  (\nabla_\mathcal{L}\mathbf{U}_k)_{\text{tangent}}}$ where \( \eta \) is the step size.  Then, to map \( \mathbf{U}^{\text{new}}_k \) back to the manifold, we apply a retraction by orthonormalizing the columns of \( \mathbf{U}^{\text{new}}_k \) using QR decomposition $ \mathbf{U}^{\text{new}}_k, \mathbf{R} = \text{QR}(\mathbf{U}^{\text{new}}_k)$, where \( \text{QR}(\cdot) \) denotes the QR decomposition. In practice, we perform this subspace update on $\U_k$ very infrequently (per every $500$ iterations), and transmit to all the layers.

\section{Modifying AdamW for subspace networks}
\label{sec:adam}

Our approach can be framed as a constrained optimization problem in which the projection matrices must remain within the subspace $\mathcal{S}$. To ensure they stay in this feasible set throughout training, it requires projecting them onto $\mathcal{S}$ at each training iteration. We propose a modified version of the AdamW optimizer that preserves the row space of $\W^l_{p_2}$ in $\mathcal{S}$. That means, once initialized within $\mathcal{S}$, the modified optimizer ensures that the $\W^l_{p_2}$ no longer require iterative projection steps, guaranteeing they remain in $\mathcal{S}$ without incurring any approximation error. Note that however, we still need to project $\W^l_{p_1}$ onto $\mathcal{S}$ due to the nonlinearity of activations as discussed in Appendix.~\ref{sec:abl_gradient}.




 AdamW defines the momentum term as $\M_t  = \beta_1 \M_{t-1} + (1-\beta_1)\nabla_L(\W_{p_2} (t) )$ and the second order momentum $\V_t  = \beta_2 \V_{t-1} + (1-\beta_2)(\nabla_L(\W_{p_2}(t) ))^2$. Then, the weight update is defined as 

\begin{equation}
    \W_{p_2}(t+1) = \W_{p_2}(t) - \eta \alpha_t \hat{\M}_t - \lambda \W_{p_2}(t),
    \label{eq:weight_update}
\end{equation}


where $\alpha_t = \frac{1}{\sqrt{\hat{\V}_t} + \epsilon}$, $\hat{\V_t} = \frac{\V_t}{\beta_2 + \epsilon}$, $\hat{\M_t} = \frac{\M_t}{\beta_1 + \epsilon}$, and $\eta$ is the fixed learning rate. $\lambda$ is a hyperparameter controlling the adaptive weight decay. Note that in the second term on the right hand side in Eq.~\ref{eq:weight_update}, $\alpha_t$ is not a constant scaling factor across coordinates.  Thus, it changes the direction of the rows of $\hat{\M}_t$. This causes the $\W_{p_2}(t+1)$ to drift from  $\mathcal{S}$. Therefore, we make $\alpha_t$ constant row-wise so that the resulting update $ \eta \alpha_t \hat{\M}_t - \lambda \W_p(t)$ is closed within $\mathcal{S}$. Specifically, we alter $\hat{\V}_t$ as below. We first take

\begin{equation}
    \mu_{\text{row}}(\hat{\V}_t) = \frac{1}{m} \sum_{i=1}^m \hat{\V}_t({:, i}) \in \mathbb{R}^{n \times 1}.
\end{equation}

where $\mu_{\text{row}}(\cdot)$ is the row-wise mean. Then, 
\begin{equation}
    \tilde{\V}_t = \mathbf{1}_m \cdot \mu_{\text{row}}(\hat{\V}_t),
\end{equation}

where \(\mathbf{1}_m \in \mathbb{R}^{1 \times m}\) is a column vector of ones, and \(\tilde{\V} \in \mathbb{R}^{n \times m}\) has the same dimensions as \(\hat{\V}_t\). Then we substitute $\hat{\V}_t$ with \(\tilde{\V}\) in Eq.~\ref{eq:weight_update}. We only apply the above modification to $\W_{p_2}$ and keep the usual updates as it is for other weights.

\section{Computational Overhead of subspace updates}

Our compression introduces two additional computational components: \emph{weight projection} and \emph{Grassmann updates}. Here, we provide an analysis of the overhead associated with these operations, demonstrating their minimal impact on memory usage and computational efficiency.

\paragraph{Overhead of Weight Projection.}
We first assess the computational cost incurred by the weight projection step. For a practical evaluation, we consider a model with approximately 2 billion parameters (8 layers, 4k model dimension, 16 attention heads), pipelined across 8 A10G GPUs. In our experiments, a single forward pass takes approximately $4.61$s, while the weight projection step adds only about $0.05$s of computation time. Thus, the relative overhead introduced by weight projection is approximately $1\%$, indicating a negligible computational burden.

\paragraph{Overhead of Grassmann Updates.}
Updating the subspace via Grassmann manifold optimization initially appears complex; however, by deriving closed-form Euclidean gradients, the practical implementation becomes computationally inexpensive. Specifically, given a Grassmann loss defined as:
\[
\mathcal{L}_{\text{Grassmann}} = \frac{1}{K}\sum_{t=k}^{k+K}\left\|\nabla_L\left(\mathbf{X}_t^{\text{final}}\right)\left(\mathbf{I} - \mathbf{U}_k\mathbf{U}_k^\top\right)\right\|_F^2,
\]
and using the Frobenius norm expansion, we obtain:
\[
\mathcal{L}_{\text{Grassmann}} = \text{const} - \frac{1}{K}\sum_t \text{Tr}\left(\mathbf{G}_t\mathbf{U}_k\mathbf{U}_k^\top\mathbf{G}_t^\top\right),
\]
where $\mathbf{G}_t=\nabla_L(\mathbf{X}_t^{\text{final}})$. By defining the symmetric accumulation matrix $\mathbf{S}=\frac{1}{K}\sum_t\mathbf{G}_t^\top\mathbf{G}_t$, the closed-form Euclidean gradient simplifies elegantly to:
\[
\nabla_{\mathcal{L}_{\text{Grassmann}}}(\mathbf{U}_k) = -2\mathbf{S}\mathbf{U}_k.
\]

Practically, implementing this gradient involves straightforward matrix operations in PyTorch:
\begin{itemize}
    \item \textbf{Accumulation:} Compute and accumulate $\mathbf{G}_t^\top\mathbf{G}_t$ once per iteration via standard matrix multiplication.
    \item \textbf{Gradient Computation:} Multiply $\mathbf{S}$ by $\mathbf{U}_k$ once per a fixed interval (e.g., every 500 iterations).
    \item \textbf{Riemannian Projection:} Project back onto the Grassmann manifold via a basic linear algebra operation, again executed at low frequency.
\end{itemize}

These matrix operations are inexpensive and infrequent, resulting in minimal computational overhead. Consequently, our method's overall cost of maintaining and updating the subspace is negligible compared to the cost of the forward and backward passes of transformer models, affirming the efficiency of our approach.

\section{Theoretical insights}
\label{sec:theory}

This section provides several theoretical insights for the proposed method. We structure our analysis into four key statements (and put the corresponding formal theorems in the Appendix). The first statement indicates that if there is a lossy compression between the layers, as the model depth increases, the approximation error of compression can grow exponentially. The second and third statements demonstrate that, even in an uncompressed network, if the weight gradients are confined to a particular subspace, the weight matrices also naturally converge to a low-dimensional subspace with AdamW. The fourth observation establishes that explicitly enforcing a subset of weights onto a low-dimensional subspace does not harm convergence.


\begin{pabox}[st:1]{}

If the compression of activations and activation gradients between layers in a model-parallel setting introduces approximation errors, these errors can accumulate exponentially with increasing depth, provided the weight and activation norms are sufficiently large.  Refer to Theorem~\ref{th:error_accumulation} for a formal proof.
\end{pabox}


The above result suggests that extending compression techniques from DDP (which are lossy) to MP (which requires compressing information passed between adjacent layers) leads to the accumulation of approximation errors. This occurs because the compression at one layer directly impacts downstream layers, a phenomenon not present in DDP. Additionally, the lack of exploitable structure in activations and activation gradients \cite{bian2024does, rudakov2023activations} typically results in larger approximation errors compared to the gradients of weights. This limitation makes such compression methods unsuitable for MP in large-scale models.

\begin{pabox}[st:2]{}

If the gradients of a particular weight matrix in a network are constrained to a fixed subspace, then under AdamW, the weight \emph{updates} asymptotically converge to that subspace over a sufficiently large number of training steps. For a formal proof, see Theorem~\ref{th:weight_update}.
\end{pabox}



This provides a critical insight: if the gradients of an unconstrained network predominantly lie within a specific low-dimensional subspace \(\mathcal{S}\)—a property we empirically validate (see Appendix~\ref{sec:abl_gradient})—then the AdamW optimizer asymptotically restricts updates outside of \(\mathcal{S}\). While this behavior is straightforward for vanilla stochastic gradient descent (SGD), it is non-trivial for AdamW due to its adaptive learning rate mechanism.


\begin{pabox}[st:3]{}

If Statement \ref{st:2} holds, then the corresponding weight matrices asymptotically converge to the same subspace, irrespective of their initialization. For a formal proof, see Theorem~\ref{th:weight_convergence}.
\end{pabox}



Intuitively, this result indicates that the decoupled weight decay mechanism in AdamW systematically suppresses components of the weight matrix that receive negligible gradient updates. Consequently, the learned weights converge to a low-dimensional subspace defined by the gradient updates.



\begin{pabox}[st:4]{}

A network in which a subset of weights is constrained to a low-dimensional subspace converges to a first-order stationary point with a convergence rate of \(O(1/T)\). For a formal proof, see Proposition~\ref{thm:convergence}.
\end{pabox}

This result is a straightforward extension of the the convergence rate guarantees for constrained optimization using proximal gradient descent on non-convex functions. It shows that even when a \emph{subset} of parameters is restricted to a lower-dimensional subspace, the standard convergence rate of \(O(1/T)\) in terms of stationarity remains intact. 


\section{Experiments}
\label{sec:experiments}
\subsection{Experimental Setup}

We evaluate decoder-only models (based on Llama 3~\cite{dubey2024llama}) across four large-scale datasets: WikiText (WT)~\cite{wikitext}, BookCorpus (BC)~\cite{bookcorpus}, OpenWebText (OWT)~\cite{owt}, and C4~\cite{2019t5}. For WT, we use the standard splits; for BC and OWT, we randomly select $10\%$ of training data as validation; for C4, due to computational constraints, we report training loss only. The base model has a context length of $1024$, embedding dimension $4096$, $24$ heads, and $8$ layers ($\sim$2B parameters); larger models (up to 8B parameters) are noted explicitly in ablation sections. We use a base learning rate $\eta=3e\text{-}4$ (with warmup and linear decay), weight decay $0.01$, and batch size $32$, unless otherwise specified. We use GPipe~\cite{huang2019gpipe} via \texttt{torch.distributed.pipelining}, integrating our compression into all but the final transformer layer.

\begin{figure*}[ht]
    \centering
    \begin{subfigure}[b]{0.3\textwidth}
        \includegraphics[width=\textwidth]{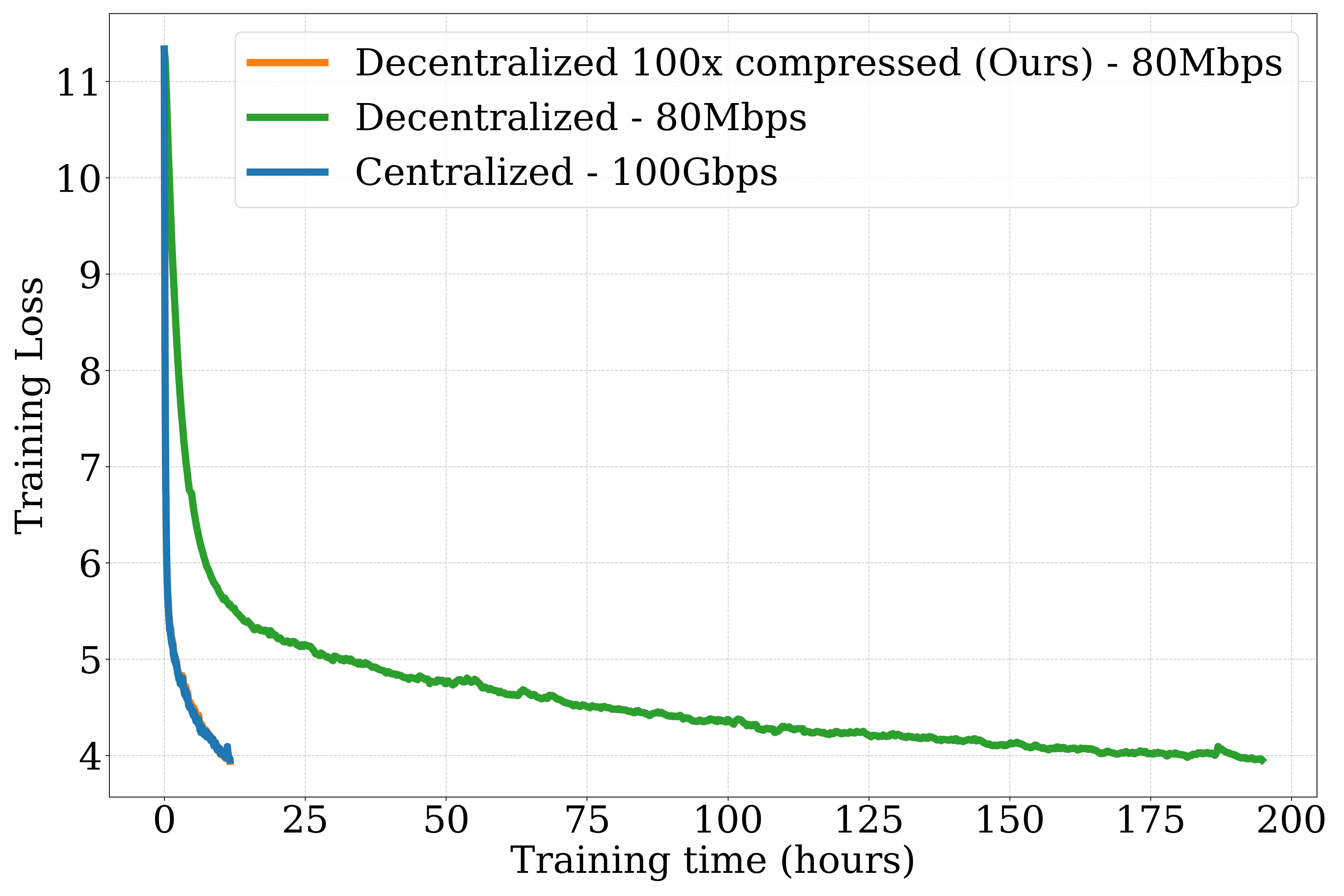}
 
        \label{fig:sub1}
    \end{subfigure}
    \hfill
    \begin{subfigure}[b]{0.3\textwidth}
        \includegraphics[width=\textwidth]{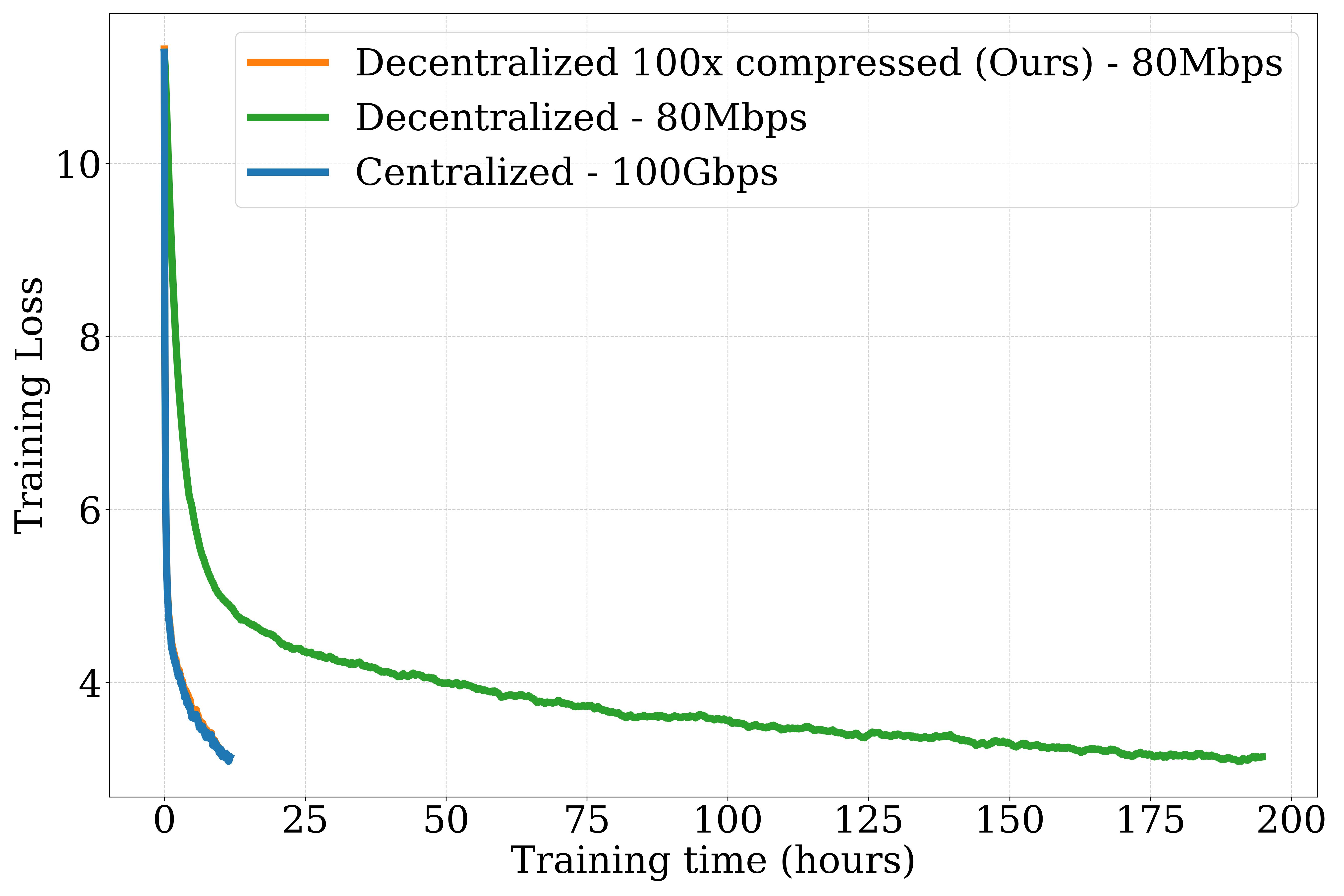}

        \label{fig:sub2}
    \end{subfigure}
    \hfill
    \begin{subfigure}[b]{0.3\textwidth}
        \includegraphics[width=\textwidth]{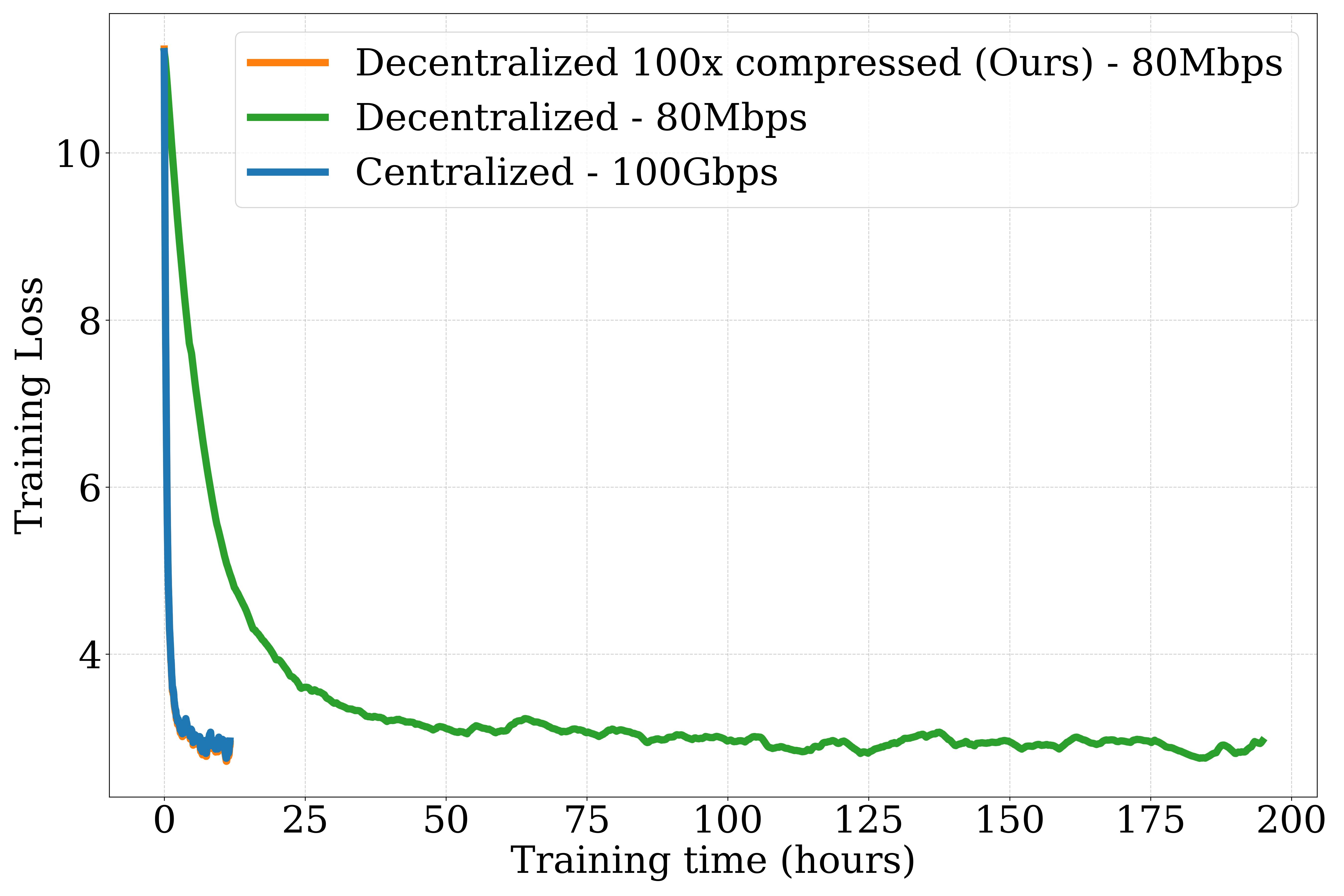}

        \label{fig:sub3}
    \end{subfigure}

    \caption{\textbf{Convergence in low-bandwidth settings.} From left to right: OpenWebText, WikiText, and BookCorpus. In each plot, the training curves are presented against wall-clock time for an 8-layer (2B) model. Decentralized models utilize $80$Mbps connections while the centralized model has datacenter-grade $100$Gbps links. Our compressed model achieves on-par convergence to the centralized model, even under a $80$Mbps bandwidth budget. In contrast, the non-compressed decentralized model with $80$Mbps links suffers from significantly slower convergence due to the communication bottleneck.  }
    \label{fig:validation}
\end{figure*}

We initialize \(\U_k\) with isotropic Gaussian noise and set \(k=40\), achieving \(100\times\) compression. Bandwidth simulations sample from \(\mathcal{N}(\mathcal{B},0.2\mathcal{B})\) per pass, defining `centralized' as 100Gbps or 16Gbps setups, with all others as `decentralized'. Experiments (except the 8B Llama run on L4 GPUs with internet-based decentralized connections) use A10g GPUs (24GB VRAM) with one layer per GPU. \textbf{Our method's effectiveness increases with faster accelerators, as slower GPUs allow more computation-communication overlap.}

\begin{figure}[ht]

    \centering
    \begin{subfigure}[b]{0.45\textwidth}
        \includegraphics[width=\textwidth]{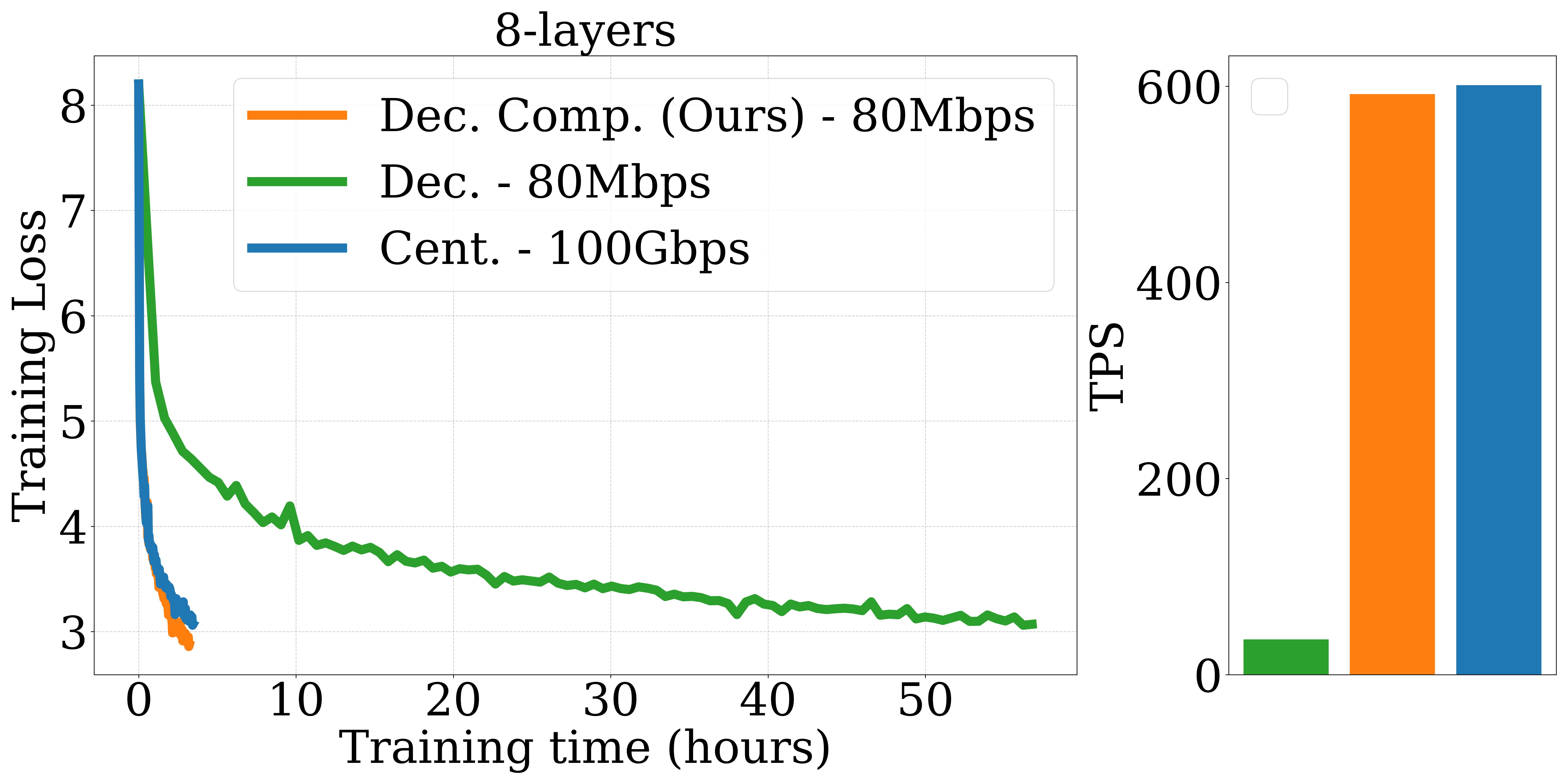}
        \label{fig:sub2}
    \end{subfigure}
    \hfill
    \begin{subfigure}[b]{0.45\textwidth}
        \includegraphics[width=\textwidth]{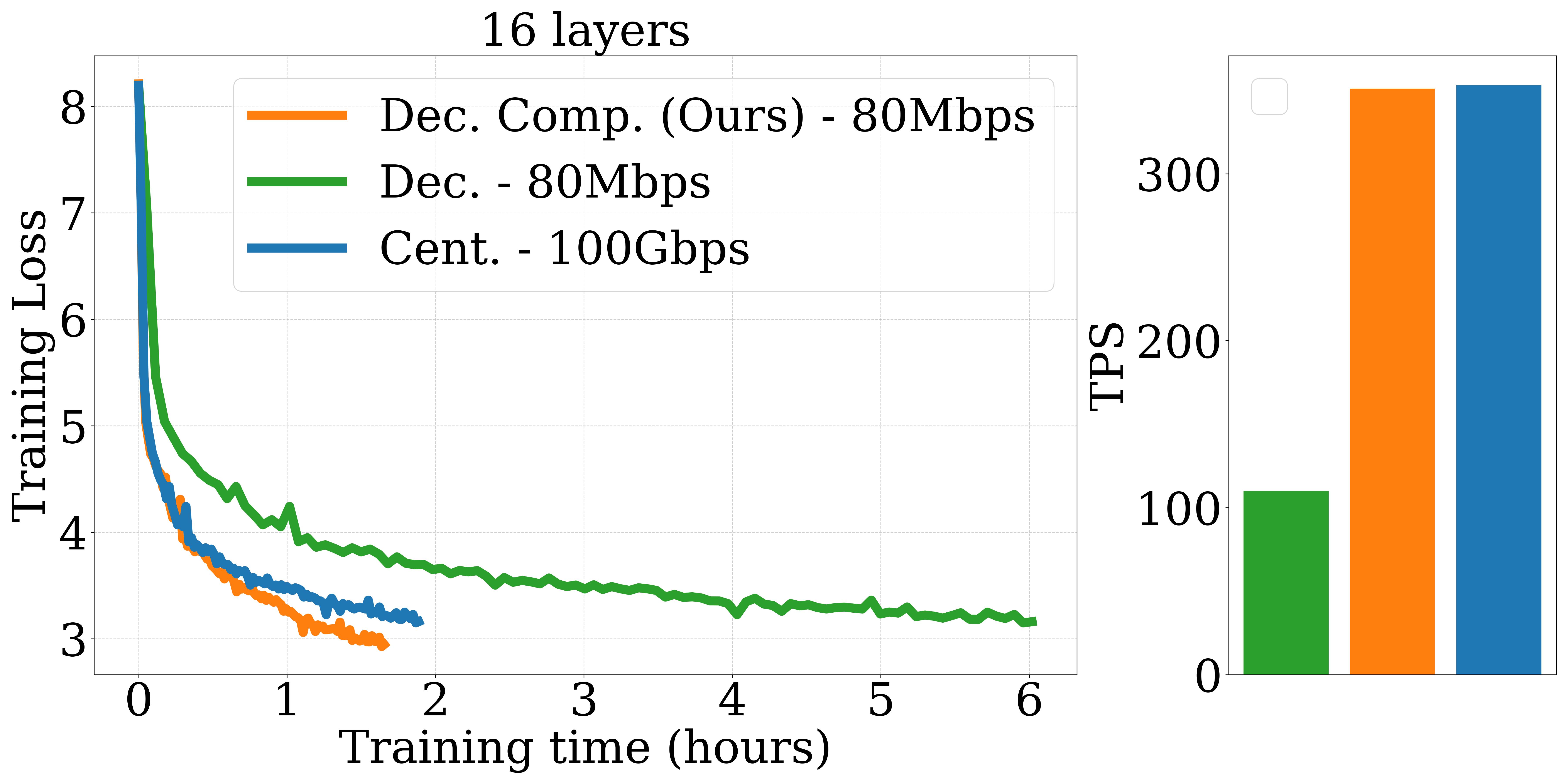}
        \label{fig:sub2}
    \end{subfigure}
  
    \caption{\textbf{Performance against depth.} Two key observations: (1)  Our compression  matches (even slightly exceeds) with centralized baseline  as the \# layers increases. Fig.~\ref{fig:llama_8b}, further validates this. (2) For the 16-layer model, we fit two layers per GPU (using A100s), increasing computation per block. As shown, the performance gap between centralized and decentralized models slightly decreases, as the ratio between computation and bandwidth bottlenecks is reduced—consistent with the square-cube law. 
  }
    \label{fig:depth}
\end{figure}

\begin{figure}[t]
    \centering
    \begin{minipage}{0.38\columnwidth}
        \centering
        \includegraphics[width=\linewidth]{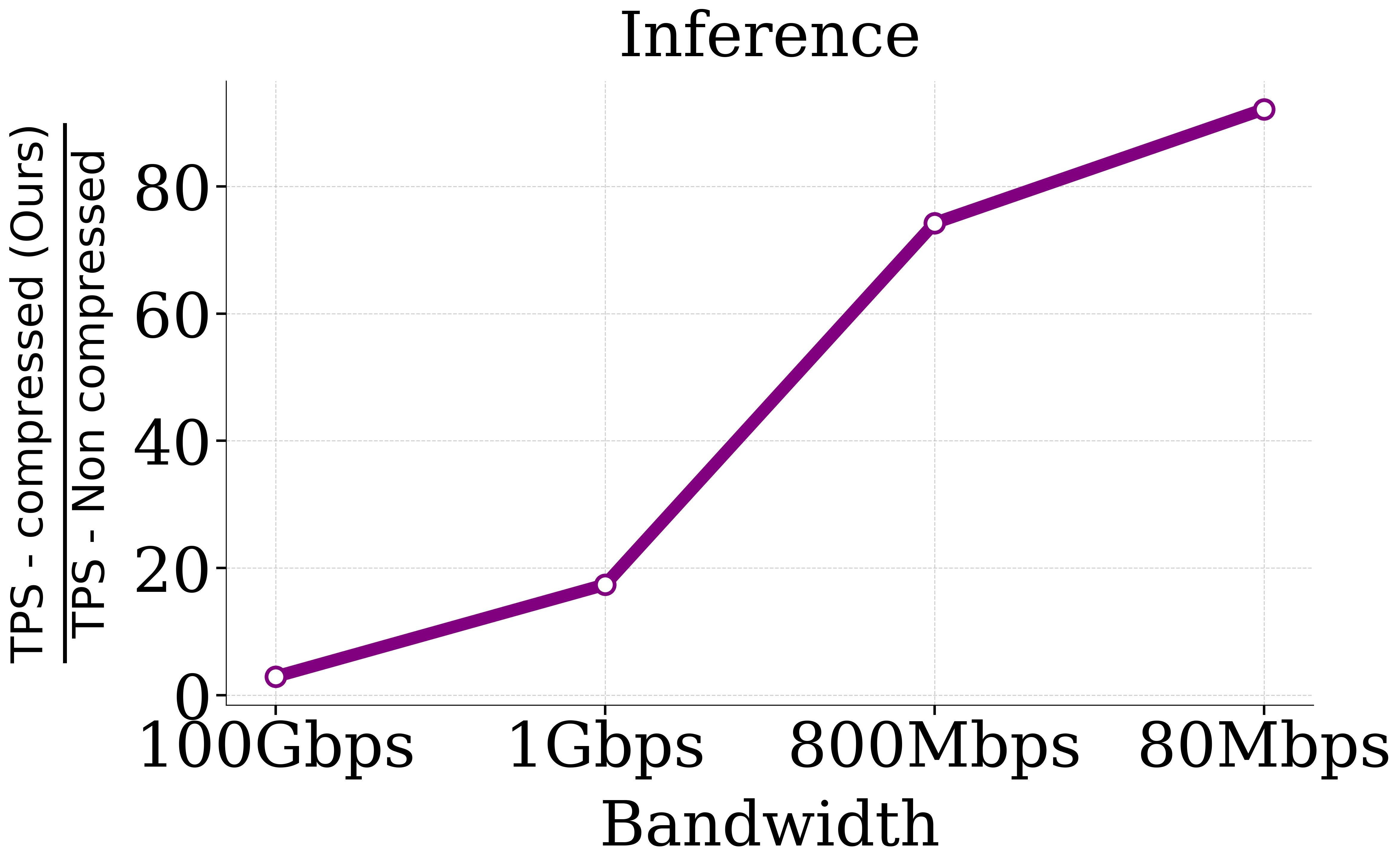}
        \label{fig:sub2}
    \end{minipage}
    \hspace{1mm}
    \begin{minipage}{0.38\columnwidth}
        \centering
        \includegraphics[width=\linewidth]{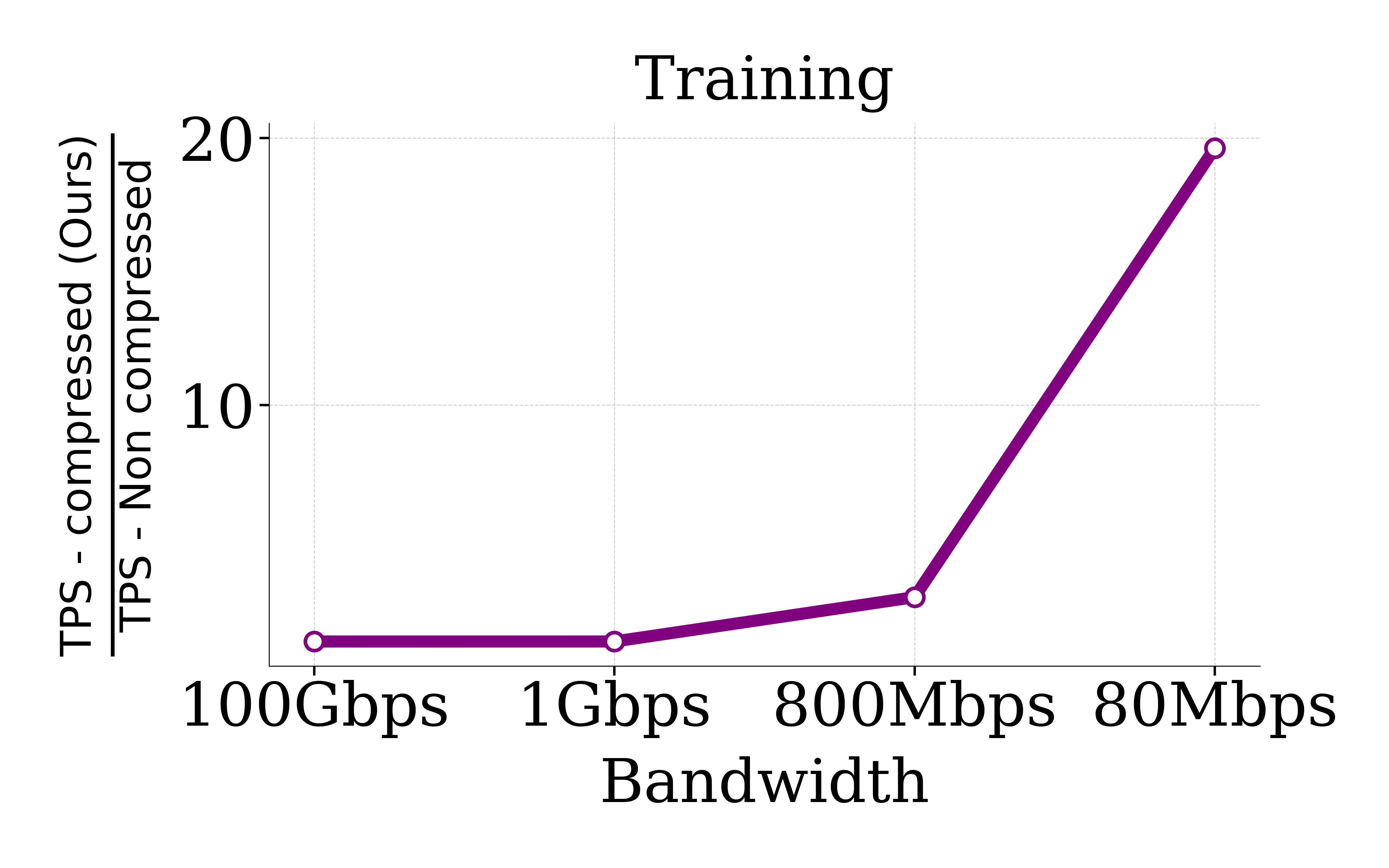}
        \label{fig:sub2}
    \end{minipage}

    \caption{\textbf{Throughput Gain.} As bandwidth becomes increasingly limited, the compressed models achieve a significantly higher throughput gain in both inference (left) and training (right). Results are shown for 8-layer (2B) models.}
    \label{fig:throughput_gain}
\end{figure}

\paragraph{Square-Cube Law.}
Square-cube law \cite{ryabinin2023swarm} states that in distributed training, computation scales cubically with model size per node, while communication grows only quadratically. This partially offsets communication bottlenecks with computational overhead. Thus, $c$-times slower communication does not lead to a $c$-times slower convergence. Hence, by improving communication efficiency by $100\times$, we achieve convergence speeds comparable to $100$Gbps setups, even with $80$Mbps links.

\begin{figure}[ht!]
    \centering
    \begin{minipage}[t]{0.49\textwidth}
        \centering
        \includegraphics[width=0.95\linewidth]{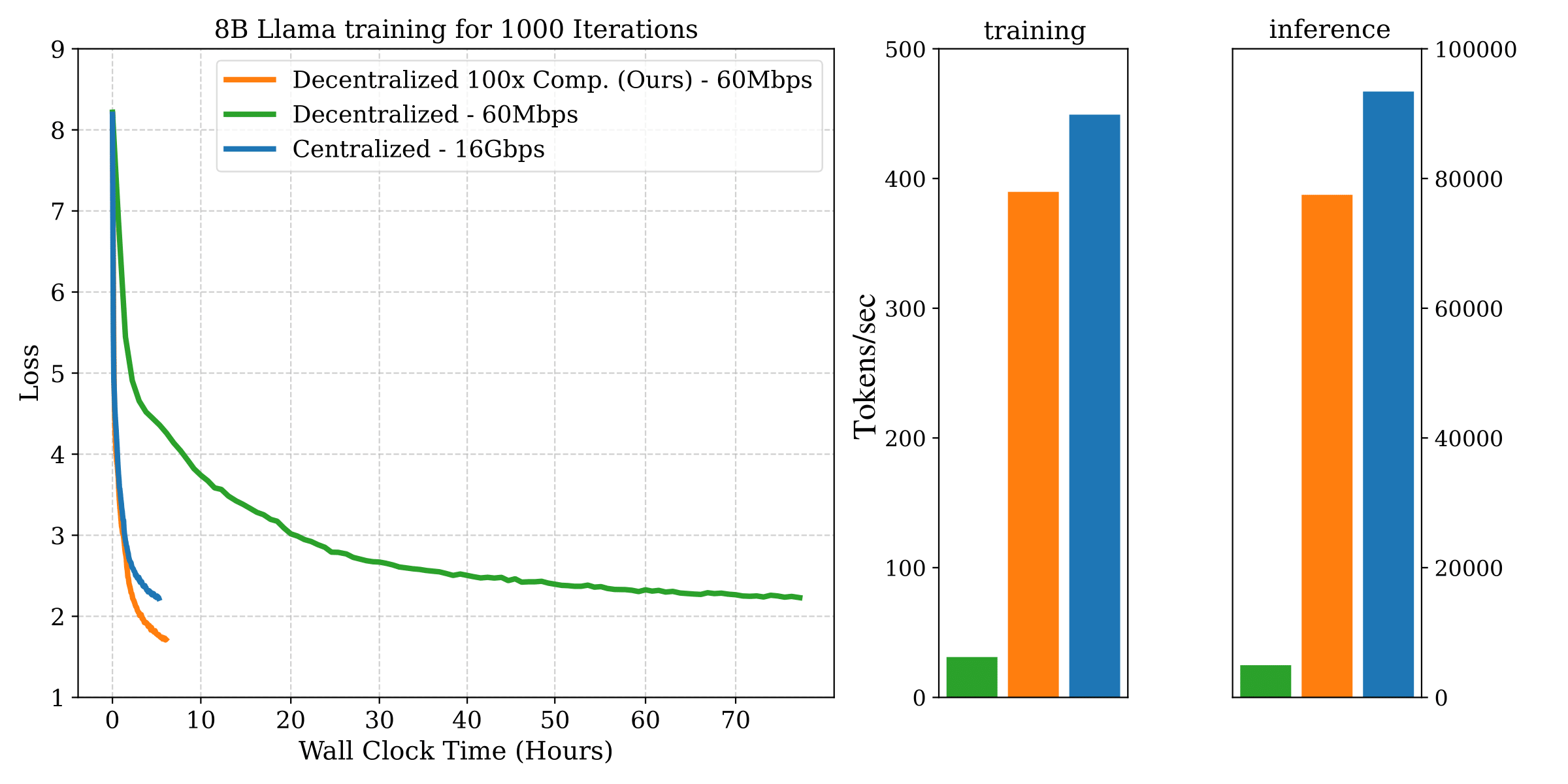}
\caption{\textbf{Training convergence and throughput on an 8B LLaMA model.} All runs use 64 L4 GPUs distributed over 8 instances. Centralized instances reside within one region (min bandwidth 16\,Gbps), while decentralized instances span 4 regions (min bandwidth 60\,Mbps), highlighting pipeline parallel bottlenecks due to reduced inter-node bandwidth.}
        \label{fig:llama_8b}
    \end{minipage}
    \hfill
    \begin{minipage}[t]{0.49\textwidth}
        \centering
        \includegraphics[width=0.80\linewidth]{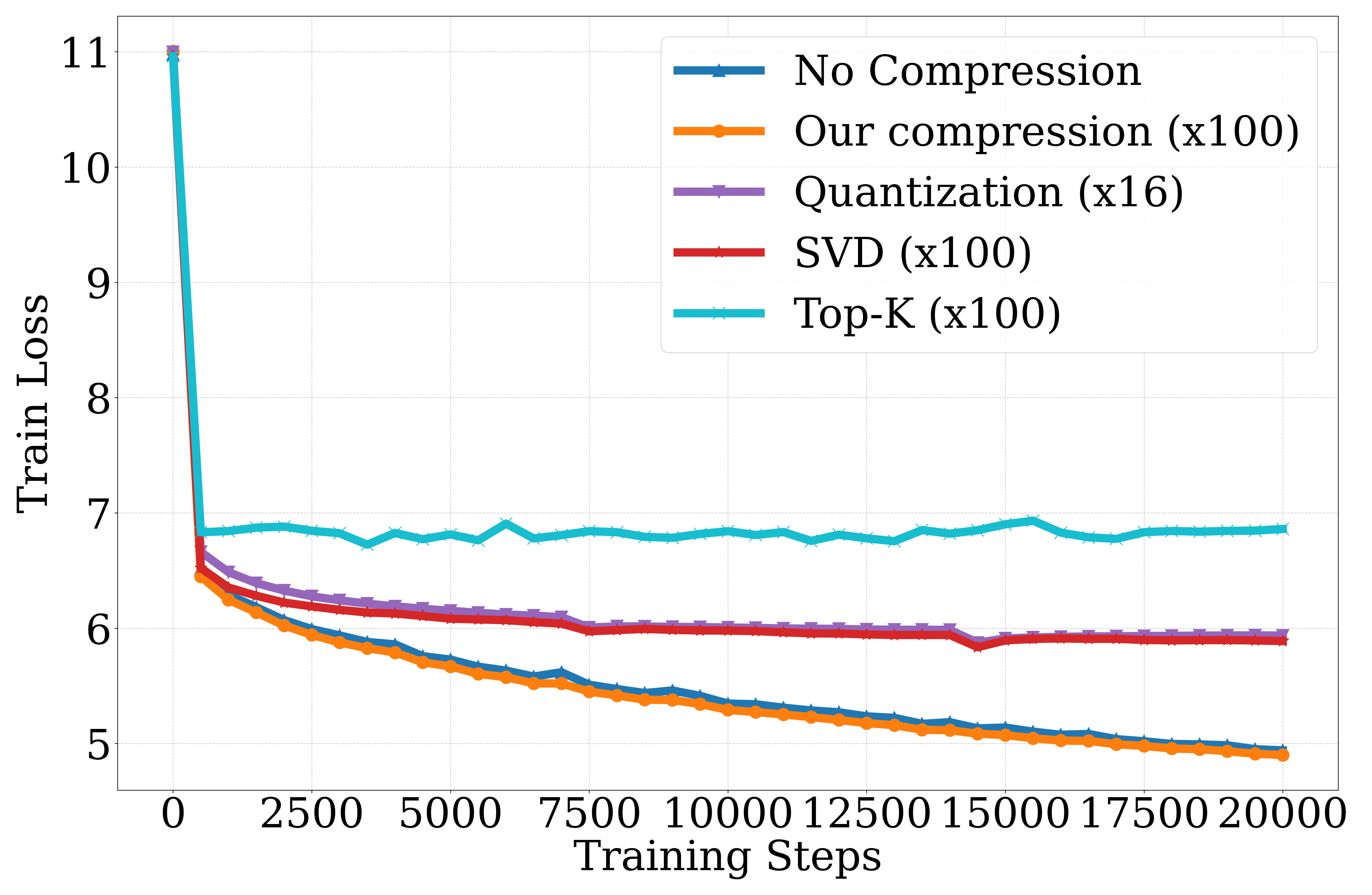}
  \caption{\textbf{Comparison against lossy compression methods.} Top-k, low-rank (SVD), and quantization fail to converge at $100\times$ compression, with quantization additionally limited by numerical precision. Our method matches the convergence rate of the uncompressed baseline.}
        \label{fig:compression_abl}
    \end{minipage}
 
\end{figure}

\subsection{Convergence in low-bandwidth settings}

\begin{table}
\centering
\small
\caption{\textbf{Perplexity scores}. Models trained for 12 hours on OpenWebText (OWT), BookCorpus (BC), and WikiText (WT). Bandwidth (B/W) and tokens per second (TPS) are reported. Our method outperforms even the centralized model, achieving significantly higher TPS compared to the non-compressed decentralized baseline.}  
\label{tab:perplexity}
\begin{tabular}{l|l|rrr|r}
\toprule
\textbf{Model} & \textbf{B/W} & \textbf{OWT}$\downarrow$ & \textbf{BC}$\downarrow$ & \textbf{WT}$\downarrow$ & \textbf{TPS}$\uparrow$ \\ 
\midrule
Decentralized & 80Mbps  & 925.19 & 108.85 & 601.84 & 36.12  \\
Decentralized Compressed (Ours) & 80Mbps & \textbf{46.75} & \textbf{17.63} & \textbf{23.01} & 592.41 \\
\midrule
Centralized & 100Gbps & 47.22 & 18.35 & 23.08 & 602.57  \\ 
\bottomrule
\end{tabular}
\end{table}




Our method enables training models over extremely low-bandwidth connections. We trained networks on both $80$ Mbps and datacenter-grade $100$ Gbps connections. Fig.~\ref{fig:validation} illustrates the train curves of an 8-layer ($2$B) model against wall-clock time. As expected, training over $80$ Mbps links in the decentralized setting significantly degrades convergence. In contrast, with our compression, the decentralized model achieves on par convergence to the network trained over $100$ Gbps connections. To demonstrate the generalization, the perplexity scores of each model over validation sets is shown in Table.~\ref{tab:perplexity}. As evident, the decentralized model with our compression even surpasses the performance of the centralized model for the same training time. To further validate test-time performance, we train models to convergence using the compute-optimal 1:20 model-to-token ratio from the Chinchilla scaling law~\cite{hoffmanntraining}, reaching compute-optimality at 12B training tokens with superior performance (Appendix~\ref{app:compute_optimal}).




\subsection{Throughput gain}

Our compression also significantly accelerates inference. As inference  requires less computation than training, bandwidth becomes the dominant bottleneck; hence, our compression yields substantial  gains. Fig.~\ref{fig:throughput_gain} illustrates gains in both training and inference: at inference, we achieve almost a $100\times$ speedup  at 80\,Mbps. Although this advantage diminishes at higher bandwidths (e.g., 100\,Gbps), we still observe about a $3\times$ improvement. \textbf{This indicates that even centralized systems benefit from reduced inference latency using our approach, which can translate into considerable cost savings with the recent trend of inference time scaling of large language models \cite{snell2024scaling, bi2024forest}} A similar trend holds for training throughput as well.



\subsection{Validation at compute-optimal}
\label{app:compute_optimal}

We conduct additional training experiments to demonstrate the effectiveness of our proposed compression method at the compute-optimal point of model training. The compute-optimal point, as defined by prior work such as Chinchilla, represents an ideal trade-off between model size and token count, optimizing performance within a fixed computational budget. Typically, for architectures like LLaMA, this optimal point corresponds to approximately a 1:20 model-to-token ratio. Consequently, we trained a 640M-parameter model for around 12 billion tokens, aligning closely with the compute-optimal recommendation.

The training was executed using 8 FSDP workers, achieving approximately 22k tokens per second throughput. The results, summarized in Table~\ref{tab:compute_optimal}, demonstrate that our compression method not only matches but marginally outperforms the non-compressed baseline in terms of validation perplexity.

\begin{table}[ht]
\centering
\caption{Validation perplexity (lower is better) at the compute-optimal point (12 billion tokens) for a 640M-parameter model. Both ours and centralized model are trained for a same number of iterations. Training the decentralized uncompressed model to compute optimal is infeasible (estimated over 200 days) so we only report the TPS.}
\label{tab:compute_optimal}
\begin{tabular}{lcccc}
\toprule
\textbf{Model} & \textbf{C4} &  \textbf{BookCorpus} & \textbf{TPS}\\ 
\midrule
Decentralized & - & - & 871 \\
Decentralized Compressed (Ours) & \textbf{12.53} & \textbf{12.67} & 22k \\
Centralized & 12.61 & 12.79 & 22.5k \\
\bottomrule
\end{tabular}
\end{table}

These results confirm that our proposed compression method maintains or slightly improves model performance at the crucial compute-optimal training regime.

\subsection{Scaling across globally distributed GPUs}

To further explore the scalability, we trained a LLaMA 8B parameter variant~\cite{dubey2024llama} with $2048$ context length, using TorchTitan~\cite{liang2024torchtitan} on the C4 dataset across 64 L4 GPUs distributed across 8 instances. We use a pipeline parallel setup with $32$ stages running in $2$ FSDP dimensions, where the 32 transformer layers are distributed one layer per stage. We evaluated two environment configurations: Centralized and Decentralized. In the Centralized setting all instances were located in the same cloud region and the bandwidth spans between $16$Gbps-$27$Gbps. For the Decentralized case  the 8 instances were distributed across 4 distinct regions (North America, Europe, and Asia). Additionally, no two consecutive stages were placed in the same region for the decentralized setup, hence, bandwidth spans from $60$Mbps-$350$Mbps. 
As shown in Fig.~\ref{fig:llama_8b},  our compression method in the decentralized configuration matches  the wall-clock time (even slightly improving) and TPS with the centralized setting. In contrast, the decentralized setting w/o compression was \textit{13x slower}.

\subsection{Ablations}

We conduct ablations on the C4 dataset to evaluate the robustness of our method. Fig.~\ref{fig:depth} compares performance across model depths. If our compression were lossy, deeper models would accumulate errors, degrading performance relative to the centralized baseline \cite{bian2024does} (see also Theorem~\ref{th:error_accumulation}). However, our results show that even as depth increases from 8 (2B) to 16 (3.5B) layers, convergence remains on par with centralized baselines. Further, our large-scale experiment (Fig.~\ref{fig:llama_8b}) confirms that 32-layer models scale effectively, demonstrating decentralized training of large models with MP for the first time.  Fig.~\ref{fig:depth} also highlights that in the 16-layer model, assigning two layers per GPU (A100, 40GB VRAM) increases per-GPU computation, reducing bandwidth bottlenecks and narrowing the gap between decentralized and centralized models. This  validates the square-cube law \cite{ryabinin2023swarm}, showing how computation-to-communication balance impacts decentralized training. However, note that decentralized training primarily targets low-end GPUs rather than high-end hardware. Ablations over other design choices and the negligible memory overhead are discussed in Appendix \ref{app:memory_overhead}.

\subsection{Comparison against lossy compressions}

As per Statement~\ref{st:1}, standard compression methods used in DDP do not effectively extend to MP. We train an 8-layer model on the WikiText, comparing our compression method with TopK, quantization, and low-rank projection. As shown in Fig.~\ref{fig:compression_abl}, with an aggressive compression rate of $\times 100$, such  compression schemes fail to converge. In contrast, our method achieves convergence on par with the non-compressed model. Note that for quantization, the best compression rate we can achieve is $16 \times$ for 16bit precision.

\subsection{Memory Overhead Analysis}
\label{app:memory_overhead}

As our method requires addition of fixed high-rank and dynamic low-rank embeddings, one potential concern is the  memory overhead, as sequence length $L$ grows. We empirically and theoretically demonstrate that the absolute memory overhead introduced by our approach remains constant and neglegible, while the relative overhead decreases with increasing sequence lengths.

We first empirically validate this and then give a theoretical explanation. To this end, we performed experiments using a 2B-parameter Transformer model (8 layers, 4k model dimension, 16 attention heads) distributed across eight NVIDIA H100 GPUs under varying sequence lengths $L$. Table~\ref{tab:memory_overhead} summarizes these results. Remarkably, the absolute overhead consistently remains around 400 MB, irrespective of sequence length. As $L$ grows from 8k to 24k, the relative memory overhead correspondingly drops from 4.0\% to 0.6\%.

\begin{table}[ht]
\centering
\caption{Peak memory usage comparison between baseline and our subspace method as sequence length ($L$) scales.}
\label{tab:memory_overhead}
\begin{tabular}{lcccc}
\toprule
$L$ & Baseline (GB) & Ours (GB) & Overhead & Relative Overhead \\ 
\midrule
8k & 9.66 & 10.06 & $\sim$400MB & $\sim$4.0\% \\[2pt]
16k & 36.51 & 36.91 & $\sim$400MB & $\sim$1.1\% \\[2pt]
24k & 76.00 & 76.46 & $\sim$400MB & $\sim$0.6\% \\
\bottomrule
\end{tabular}
\end{table}

This consistent and negligible overhead can be explained through PyTorch’s memory management behavior. Firstly, fixed embedding lookups in our method are ephemeral since they are non-trainable; PyTorch does not store activation gradients for them. After they are added to the activations, PyTorch’s caching allocator reuses the memory (instead of invoking cudaMalloc/cudaFree), allowing temporary tensors to be released prior to attention. Thus, their memory footprint doesn’t persist into later stages. Secondly, embedding lookups ($\mathcal{O}(B \times L \times D)$ memory usage) are inherently minor relative to attention ($\mathcal{O}(B \times L^2 \times D)$) and MLP layers ($\mathcal{O}(B \times L \times D^2)$), making their temporary storage impact negligible in peak memory usage.

Further, the cached embedding tables themselves do not pose a meaningful threat to scalability. For instance, even a large embedding table (10k-dimensional embedding and 50k vocabulary size) would consume approximately 0.93 GB in 16-bit precision, requiring only impractically large embedding dimensions (e.g., 100k+) to surpass 10 GB. State-of-the-art frontier models (e.g., DeepSeek~\cite{deepseek2023}, LLAMA-405B~\cite{dubey2024llama}) remain far below such thresholds (7K and 16K, respectively). Therefore, even with significantly increased embedding dimensions, memory overhead remains dominated by MLP and attention computations, not embedding storage.

In conclusion, our method demonstrates robust and scalable memory behavior, effectively managing memory overhead even at extremely large sequence lengths, validating its practicality for decentralized training at scale.

\subsection{Memory Overhead with Increasing Workers}

Another key consideration is how our compression scheme scales in terms of memory overhead with an increasing number of workers. One might suspect that distributing long contexts across more workers could lead to increased overhead per worker. We clarify that this concern is unfounded due to the way embeddings and attention computations are handled in our setup.

To elaborate, in a context-parallel training setting, each worker handles a distinct segment of the input sequence. Embedding lookups are performed \emph{locally on each worker} before attention computations. The resulting key-value (KV) tensors maintain the same shape and size as a standard, non-compressed model, ensuring no extra storage requirement. Importantly, the ephemeral embedding additions are discarded immediately after their local usage, well before the memory-intensive attention and MLP layers, thus never contributing to peak memory usage.

To empirically validate our claims, we extend our experiments to include varying numbers of workers using Ring Attention, a state-of-the-art context-parallel mechanism where each worker exchanges KV tensors only with its immediate neighbors. As embedding operations occur locally and KV tensors remain unchanged in size, our compression scheme integrates seamlessly.

Table~\ref{tab:memory_workers} illustrates that, irrespective of the increase in sequence length ($L$) and corresponding increase in the number of workers, our per-worker memory overhead remains constant (around 400MB) and does not scale with either sequence length or worker count.

\begin{table}[ht]
\centering
\caption{Peak memory per worker with increasing sequence lengths and workers using our compression scheme.}
\label{tab:memory_workers}
\resizebox{\columnwidth}{!}{
\begin{tabular}{lccccc}
\toprule
$L$ & Num. Workers & Baseline (GB) & Ours (GB) & Overhead/Worker & Relative Overhead \\
\midrule
8k  & 1 & 9.66  & 10.06 & $\sim$400MB & $\sim$4.0\% \\[2pt]
16k & 1 & 36.51 & 36.91 & $\sim$400MB & $\sim$1.1\% \\[2pt]
24k & 1 & 76.00 & 76.46 & $\sim$400MB & $\sim$0.6\% \\[2pt]
50k & 2 & 76.13 & 76.55 & $\sim$400MB & $\sim$0.55\% \\[2pt]
65k & 3 & 78.19 & 79.62 & $\sim$400MB & $\sim$0.54\% \\
\bottomrule
\end{tabular}
}
\end{table}

Thus, our design effectively maintains low, constant memory overhead per worker, demonstrating excellent scalability in decentralized contexts. The embedding memory per-worker overhead remains negligible and constant, independent of the scaling of sequence length or the number of distributed workers, validating our approach for large-scale decentralized training scenarios.

\section{Broader impact}

Our work presents what we believe to be the first  compression method that allows practical, large-scale, decentralized model-parallel training, enabling efficient scaling in low-bandwidth environments. This marks a significant step toward democratizing access to large-scale AI by reducing dependence on costly datacenter infrastructure and expanding opportunities for deep learning research. Our approach has far-reaching implications for making AI more accessible, cost-effective, and environmentally sustainable across academia, startups, and industry.

\section{Conclusion}
\label{sec:conclusion}

We propose a novel compression technique that, for the first time, enables aggressive compression in MP without harming convergence. By leveraging structured subspace constraints, we achieve up to $100\times$ communication efficiency while preserving convergence. Our compression enhances both inference and training efficiency, reducing latency even in centralized settings. Extensive experiments across varying model depths, bandwidth conditions, and large-scale deployments validate the effectiveness and scalability of our approach. Notably, we demonstrate its real-world applicability by successfully training an 8B-parameter LlaMa model on low-end GPUs distributed across multiple global regions, connected solely via internet-grade (60 Mbps) links, while achieving convergence comparable to a centralized setup.  Our results establish that decentralized training of large-scale models using model parallelism—previously hindered by severe communication bottlenecks—is now practical. 



{
\bibliographystyle{plainnat}
\bibliography{main}
}

\newpage
\appendix

\begin{center}
    {\LARGE \textbf{Appendix}}
\end{center}

\section{Gradient compression in backpropagation}
\label{sec:abl_gradients}
In this section, we show that compressing the gradients by projecting them on to $\mathcal{S}$ does not induce any approximation error.

\label{sec:abl_gradient}
Let \(\nabla_l(\X^{l+1})\) denote the gradient of the loss with respect to \(\X^{l+1}\), the output of layer \(l\). The gradient with respect to \(\W^l_{p_2}\) is then
\begin{equation}
    \nabla_l(\W^l_{p_2}) 
    \;=\; (\X^l_{\text{hidden}})^\top \, \nabla_l(\X^{l+1}).
\end{equation}
The residual gradient from the skip connection is given by
\begin{equation}
    \bigl(\nabla_l(\X^l_{\text{attn}})\bigr)^{\text{residual}} 
    \;=\; \nabla_l(\X^{l+1}).
\end{equation}
Hence, the update rule for \(\W^l_{p_2}\) becomes
\begin{equation}
    \W^l_{p_2}(t+1) 
    \;=\; \W^l_{p_2}(t) \;-\; \gamma \,\bigl(\X^l_{\text{hidden}}\bigr)^\top \,\nabla_l(\X^{l+1}),
\end{equation}
    
where \(\gamma\) is the learning rate, and \(t\) denotes the training timestep.

If we project \(\nabla_l(\X^{l+1})\) onto \(\mathrm{Col}(\U_k) = \mathcal{S}\), the update can be written as
\begin{equation}
    \W^l_{p_2}(t+1) 
    \;=\; \W^l_{p_2}(t) 
    \;-\; \gamma \,\bigl(\X^l_{\text{hidden}}\bigr)^\top 
                 \bigl(\nabla_l(\X^{l+1}) \,\U_k \,\U_k^\top\bigr).
\end{equation}
If \(\gamma\) acts as a row-wise constant, then it is straightforward to see that if 
\(\mathrm{Row}\bigl(\W^l_{p_2}(t)\bigr) \,\subseteq\, \mathcal{S}\),
it follows that 
\(\mathrm{Row}\bigl(\W^l_{p_2}(t+1)\bigr) \,\subseteq\, \mathcal{S}\)
as well, because vector spaces are closed under linear operations. 

Recall that we modify AdamW so that its adaptive learning rate is constant across each row (see Section~\ref{sec:adam}), ensuring this property holds. \textbf{Therefore, by induction, if 
\(\mathrm{Row}\bigl(\W^l_{p_2}(0)\bigr) \,\subseteq\, \mathcal{S}\) 
at initialization, then 
\(\mathrm{Row}\bigl(\W^l_{p_2}(t)\bigr) \,\subseteq\, \mathcal{S}\)
for all subsequent updates when the incoming gradients \(\nabla_l(\X^{l+1})\) are projected onto \(\mathcal{S}\)}, removing the need for iterative projection of $\W_{p_2}^l$ on to $\mathcal{S}$.

Continuing backpropagation, the gradient with respect to \(\X^l_{\text{hidden}}\) is
\begin{equation}
    \nabla_l\bigl(\X^l_{\text{hidden}}\bigr)  
    \;=\; \nabla_l\bigl(\X^{l+1}\bigr)\,\U_k\,\U_k^\top \,\bigl(\W^l_{p_2}\bigr)^\top.
\end{equation}
Since \(\mathrm{Row}\bigl(\W^l_{p_2}\bigr) = \mathcal{S}\), we can write 
\(\W^l_{p_2} = \W^l_{p_2}\,\U_k\,\U_k^\top\). Substituting this in, we get
\begin{align}
    \nabla_l\bigl(\X^l_{\text{hidden}}\bigr)
    &= \nabla_l\bigl(\X^{l+1}\bigr)\,\U_k\,\U_k^\top 
       \Bigl(\W^l_{p_2}\,\U_k\,\U_k^\top\Bigr)^\top \\
    &= \nabla_l\bigl(\X^{l+1}\bigr)\,\U_k\,\U_k^\top\,\U_k\,\U_k^\top 
       \bigl(\W^l_{p_2}\bigr)^\top.
\end{align}
Because \(\U_k\) is orthonormal, \(\U_k^\top\,\U_k = \I\), so
\begin{align}
    \nabla_l\bigl(\X^l_{\text{hidden}}\bigr)
    &= \nabla_l\bigl(\X^{l+1}\bigr)\,\U_k\,\U_k^\top 
       \bigl(\W^l_{p_2}\bigr)^\top \\
    &= \nabla_l\bigl(\X^{l+1}\bigr)\,\bigl(\W^l_{p_2}\,\U_k\,\U_k^\top\bigr)^\top \\
    &= \nabla_l\bigl(\X^{l+1}\bigr)\,\bigl(\W^l_{p_2}\bigr).
\end{align}
This shows that projecting \(\nabla_l\bigl(\X^{l+1}\bigr)\) onto \(\mathcal{S}\) does not introduce any approximation error in \(\nabla_l\bigl(\X^l_{\text{hidden}}\bigr)\).

\smallskip

Similarly, the remaining gradients are computed as follows. First,
\begin{equation}
    \nabla_l\bigl(\W^l_1\bigr)
    \;=\; \bigl(\X^l_{\text{attn}}\bigr)^\top \,\Bigl(\nabla_l\bigl(\X^l_{\text{hidden}}\bigr) \,\circ\, \nabla f_{\text{relu}}\Bigr),
\end{equation}
where \((\cdot \circ \cdot)\) denotes elementwise (Hadamard) multiplication. The gradient for \(\X^l_{\text{attn}}\) then becomes
\begin{align}
    \nabla_l\bigl(\X^l_{\text{attn}}\bigr) 
    &= \Bigl(\nabla_l\bigl(\X^l_{\text{hidden}}\bigr) \,\circ\, \nabla f_{\text{relu}}\Bigr)\,\bigl(\W^l_1\bigr)^\top 
       \;+\; \bigl(\nabla_l\bigl(\X^l_{\text{attn}}\bigr)\bigr)^{\text{residual}} \\
    &= \Bigl(\nabla_l\bigl(\X^l_{\text{hidden}}\bigr) \,\circ\, \nabla f_{\text{relu}}\Bigr)\,\bigl(\W^l_1\bigr)^\top 
       \;+\; \nabla_l\bigl(\X^{l+1}\bigr).
\end{align}
Next,
\begin{align}
    \nabla_l\bigl(\W^l_{p_1}\bigr)
    &= \bigl(\X^l_{\text{concat}}\bigr)^\top\,\nabla_l\bigl(\X^l_{\text{attn}}\bigr)\\
    &= \bigl(\X^l_{\text{concat}}\bigr)^\top \,\Bigl(\bigl(\nabla_l\bigl(\X^l_{\text{hidden}}\bigr)\,\circ\,\nabla f_{\text{relu}}\bigr)\,\bigl(\W^l_1\bigr)^\top 
       \;+\; \nabla_l\bigl(\X^{l+1}\bigr)\Bigr).
\end{align}
Recall that \(\mathrm{Row}\bigl(\W^l_1\bigr) \subseteq \mathcal{S}\). Therefore, if \(\nabla_l\bigl(\X^{l+1}\bigr)\) is projected onto \(\mathcal{S}\), it follows that 
\(\mathrm{Row}\bigl(\nabla_l\bigl(\W^l_{p_1}\bigr)\bigr)\,\subseteq\,\mathcal{S}\).

\smallskip

Finally, the update rule for \(\W^l_{p_1}\) is
\begin{equation}
    \W^l_{p_1}(t+1)
    \;=\; \W^l_{p_1}(t)\;-\;\gamma\,\nabla_l\bigl(\W^l_{p_1}(t)\bigr),
\end{equation}

The gradient updates for \(\X^l_{\text{concat}}\) involve
\begin{equation}
    \X^l_{\text{concat}} 
    \;=\; \X^l_{\text{attn}}\,\bigl(\W^l_{p_1}\bigr)^\top.
\end{equation}
Because all gradients up to \(\X^l_{\text{concat}}\) are preserved without loss, it follows via the chain rule that the gradient flow to the blocks below \(\X^l_{\text{concat}}\) is also lossless. Consequently, when compressing \(\nabla_L(\X^{l+1})\) as
\begin{align}\label{eq:grad_compress}
    \bigl(\nabla_L (\X^{l+1})\bigr)_{\text{compressed}} 
    \;=\; \nabla_L\bigl(\X^{l+1}\bigr)\,\U_k 
    \;\in\;\mathbb{R}^{b \times n \times k},
\end{align}
we can fully recover it in the previous layer:
\begin{align}
    \bigl(\nabla_L (\X^{l+1})\bigr)_{\text{recovered}} 
    &= \bigl(\nabla_L (\X^{l+1})\bigr)_{\text{compressed}}\,\U_k^\top \\ 
    &= \nabla_L\bigl(\X^{l+1}\bigr),
\end{align}
\noindent thereby incurring no approximation error. Notably, since \(k \ll d\), the compressed gradient \(\bigl(\nabla_L (\X^{l+1})\bigr)_{\text{compressed}}\) is significantly lower-dimensional than \(\nabla_L (\X^{l+1})\), yielding substantial communication savings over low-bandwidth links. This shows that by projecting gradients onto the same subspace used for forward-pass compression,  we can achieve \emph{lossless gradient compression} during backpropagation. 
This strategy further eliminates the need for frequent projections of \(\W^l_{p_2}\) 
maintaining its confinement to the intended subspace.



\section{On the error accumulation of lossy compressions}
\label{sec:abl_error_accum}

A key difference between Distributed Data Parallel (DDP) training and Model Parallel (MP) training lies in how gradients are exchanged and how compression is applied. In DDP, gradients of model parameters are exchanged after each training step, enabling compression to be applied across the entire gradient vector in one operation. In contrast, PP training requires the exchange of activations and activation gradients between model partitions, leading to a layer-wise compression approach. Unlike DDP, where compression affects the gradient as a whole, MP training introduces independent compression errors at each layer. These errors accumulate progressively during training, impacting gradient computations layer by layer and posing unique challenges in maintaining model accuracy and stability.

Given these distinctions, it is critical to analyze MP training within the constraints of activation gradient compression. The layer-by-layer compression in MP can lead to compounded errors that significantly influence convergence behavior and overall model performance. The following result focuses on investigating the relative error with respect to the weight gradients, occurred by a compression error induced by a particular layer, on an arbitrary layer below it. 

\begin{theorem}
\label{th:error_accumulation}

Consider a feedforward neural network with $L$ layers, where layer $l$ applies a (differentiable) function 
\[
   x_{l+1} \;=\; f_{l}(x_l),
   \quad
   l = 1, \dots, L.
\]
Let $\nabla_L(x_l)$ denote the gradient of the final loss $\mathcal{L}$ with respect to the layer’s input $x_l$.  
Suppose that:
\begin{enumerate}
    \item The spectral norm of the Jacobian $\nabla f_l(x_l)$ is bounded above by $\nu > 0$ for all $l$, i.e.\ $\|\nabla f_l(x_l)\|\le \nu$.
    \item In backpropagation, an additional error $e_l$ is introduced at each layer $l$, with $\|e_l\|\le e$ for some constant $e>0$.
\end{enumerate}
Define $\varepsilon_l$ to be the \emph{cumulative} error in the gradient at layer $l$. Then for $\nu>1$, $\varepsilon_l$ can grow exponentially with the total number of layers $L$; in particular,
\[
   \|\varepsilon_l\|
   \;\;\le\;
   e\,\frac{\nu^{\,L - l + 1}\;-\;1}{\nu-1},
\]
which is an exponential function of $L$ when $\nu>1$.

\end{theorem}

\begin{proof}
Recall the usual chain rule for the gradient of the final loss $\mathcal{L}$ with respect to the input $x_l$ of layer $l$:
\[
  \nabla_L(x_l)
  \;=\;
  \nabla_L\bigl(x_{l+1}\bigr)\;\nabla f_l\bigl(x_l\bigr).
\]

Assume that at each layer we introduce an error in the gradient. Let 
\[
  \varepsilon_l
  \;=\;
  \bigl(\text{true gradient at layer } l\bigr)
  \;-\;
  \bigl(\text{observed/propagated gradient at layer } l\bigr).
\]
When moving from layer $l$ to layer $l-1$, the error recursion becomes:
\[
  \varepsilon_{l-1}
  \;=\;
  \varepsilon_{l}\,\nabla f_{l-1}(x_{l-1})
  \;+\;
  e_{l-1},
\]
where $e_{l-1}$ is the \emph{newly introduced} error at layer $l-1$.  Unfolding this backwards gives a general expansion:
\[
  \varepsilon_l
  \;=\;
  e_l
  \;+\;
  \sum_{j=l+1}^{L} \Biggl(\,\prod_{i=l}^{\,j-1} \nabla f_i(x_i)\Biggr)\,e_j.
\]

Taking the norm and using the assumption $\|\nabla f_i(x_i)\|\le \nu$ and $\|e_j\|\le e$, we get:
\[
  \|\varepsilon_l\|
  \;\le\;
  \sum_{j=l}^{L}
    \Bigl(\prod_{i=l}^{\,j-1} \|\nabla f_i(x_i)\|\Bigr)\;\|\;e_j\|
  \;\le\;
  \sum_{j=l}^{L} \nu^{\,j-l}\;e
  \;=\;
  e \sum_{k=0}^{L-l} \nu^{\,k},
\]
where $k = j-l$.  This geometric sum is
\[
  \sum_{k=0}^{L-l} \nu^{k}
  \;=\;
  \frac{\nu^{\,L-l+1}-1}{\nu-1}
  \quad(\text{valid for } \nu \neq 1).
\]
Hence,
\[
  \|\varepsilon_l\|
  \;\le\;
  e\,\frac{\nu^{\,L - l + 1}-1}{\nu-1}.
\]
Since $\nu>1$, $\nu^{\,L - l + 1}$ grows exponentially in $L$.  
\end{proof}

This is an important result. In particular, this result indicates that in cases where the spectral norm of the weight matrices is large enough, the upper bound for the error for the lower layers can grow exponentially as the depth of the network increases.

\section{Vanishing of gradient updates outside a subspace}
\label{sec:app:vanishing_grad}

The result in this section highlights an important property of AdamW: the optimizer’s weight updates progressively align with the subspace in which the gradients are constrained. In particular, if the gradients of a weight matrix lie within a specific low-dimensional subspace \(\mathcal{S}\), the updates made to the weight matrix are increasingly confined to \(\mathcal{S}\) over time. This behavior ensures that any weight components orthogonal to \(\mathcal{S}\) are asymptotically suppressed, with the weight matrix ultimately converging to \(\mathcal{S}\). To this end, we first prove Lemma \ref{thm:distortion-bound} and then Theorem \ref{th:weight_update}.

\begin{lemma}[Orthogonal Distortion Bound]\label{thm:distortion-bound}
  Let $a = (a_1,\dots,a_n)\in \mathbb{R}^n$ be a nonzero vector, and let $v_1,\dots,v_n>0$ be positive scalars.  
  Define 
  \[
    m \;=\;\min_{1\le i\le n}\, v_i,
    \quad
    M \;=\;\max_{1\le i\le n}\, v_i,
    \quad
    \Gamma \;=\;\frac{M}{m}.
  \]
  Form the vector
  \[
    b \;=\;\bigl(v_1 a_1,\;v_2 a_2,\;\dots,\;v_n a_n\bigr)\in \mathbb{R}^n,
  \]
  and let $b_{\perp}$ be the component of $b$ orthogonal to $a$.  Then the following two bounds hold:
  \[
    \|b_{\perp}\|
    \;\le\;
    \frac{M - m}{2}\,\|a\|
    \;=\;
    \frac{m\,(\Gamma - 1)}{2}\,\|a\|
    \quad\text{and}\quad
    \|b_{\perp}\|
    \;\le\;
    \frac{\Gamma - 1}{2}\,\|b\|.
  \]
\end{lemma}

\begin{proof}
  First, decompose $b$ into its projection onto $a$ plus its component orthogonal to $a$:
  \[
    b
    \;=\;
    \mathrm{proj}_a(b)\;+\;b_{\perp},
    \quad
    \text{where}
    \quad
    b_{\perp} \;\perp\; a.
  \]
  The norm of $b_{\perp}$ can be expressed by the well‐known identity
  \[
    \|b_{\perp}\|^2
    \;=\;
    \|b\|^2
    \;-\;
    \frac{(b \cdot a)^2}{\|a\|^2}.
  \]
  Since $b=(v_1 a_1,\dots,v_n a_n)$, we have
  \[
    \|b\|^2
    \;=\;
    \sum_{i=1}^n (v_i\,a_i)^2
    \;=\;
    \sum_{i=1}^n v_i^2\,a_i^2,
    \quad
    \text{and}
    \quad
    b \cdot a
    \;=\;
    \sum_{i=1}^n v_i\,a_i^2.
  \]
  Thus
  \[
    \|b_{\perp}\|^2
    \;=\;
    \sum_{i=1}^n v_i^2\,a_i^2
    \;-\;
    \frac{\Bigl(\sum_{i=1}^n v_i\,a_i^2\Bigr)^2}{\sum_{i=1}^n a_i^2}.
  \]
  Let $A^2 = \|a\|^2 = \sum_{i=1}^n a_i^2$.  Define weights 
  \[
    s_i \;=\;\frac{a_i^2}{A^2}
    \quad\text{so that}\quad
    s_1 + s_2 + \cdots + s_n \;=\;1.
  \]
  Then
  \[
    \sum_{i=1}^n v_i^2 a_i^2
    \;=\;
    A^2\;\sum_{i=1}^n v_i^2\,s_i,
    \quad
    \sum_{i=1}^n v_i\,a_i^2
    \;=\;
    A^2\;\sum_{i=1}^n v_i\,s_i.
  \]
  Hence
  \[
    \|b_{\perp}\|^2
    \;=\;
    A^2
    \biggl[\sum_{i=1}^n v_i^2\,s_i
      \;-\;\Bigl(\sum_{i=1}^n v_i\,s_i\Bigr)^2\biggr].
  \]
  Recognize that 
  \[
    \sum_{i=1}^n v_i^2\,s_i
    \;-\;
    \Bigl(\sum_{i=1}^n v_i\,s_i\Bigr)^2
  \]
  is precisely the variance $\mathrm{Var}_s(V)$ of the values $v_i$ with respect to the probability distribution $\{s_i\}$.  Since $m \le v_i \le M$ for all~$i$, the variance of any random variable confined to $[m,M]$ is at most $\frac{(M - m)^2}{4}$.  (This is realized, for instance, by taking a two‐point distribution at $m$ and $M$ with equal probabilities.)  Therefore,
  \[
    \sum_{i=1}^n v_i^2\,s_i
    \;-\;
    \Bigl(\sum_{i=1}^n v_i\,s_i\Bigr)^2
    \;\le\;
    \frac{(M - m)^2}{4}.
  \]
  Consequently,
  \[
    \|b_{\perp}\|^2
    \;\le\;
    A^2 \cdot \frac{(M - m)^2}{4}
    \;=\;
    \frac{(M - m)^2}{4}\;\|a\|^2,
  \]
  and hence
  \[
    \|b_{\perp}\|
    \;\le\;
    \frac{M - m}{2}\;\|a\|.
  \]
  Substituting $M = m\,\Gamma$ (so $M - m = m(\Gamma-1)$) gives 
  \[
    \|b_{\perp}\|
    \;\le\;
    \frac{m\,(\Gamma - 1)}{2}\,\|a\|.
  \]
  For the alternative form in terms of $\|b\|$, note that
  \[
    \|b\|
    \;=\;
    \sqrt{\sum_{i=1}^n (v_i a_i)^2}
    \;\ge\;
    \sqrt{\sum_{i=1}^n m^2\,a_i^2}
    \;=\;
    m\,\|a\|.
  \]
  Hence $\|a\|\le \frac{1}{m}\,\|b\|$, and
  \[
    \|b_{\perp}\|
    \;\le\;
    \frac{(M - m)}{2}\,\|a\|
    \;\le\;
    \frac{(M - m)}{2\,m}\,\|b\|
    \;=\;
    \frac{\Gamma - 1}{2}\,\|b\|.
  \]
  This completes the proof.
\end{proof}

\begin{theorem}Let \( \mathbf{W}(t) \in \mathbb{R}^{m \times n} \) be the weight matrix of a neural network optimized using AdamW at the \( t \)-th iteration. Suppose the rows of the gradient matrix \( \nabla_L (\mathbf{W}(t)) \) lie within a subspace \( \mathcal{S} \subseteq \mathbb{R}^n \). Let $\Delta_t$ be defined as

\begin{align}
    \W(t+1) = \W(t) - \eta\big(\Delta_t + \alpha \W(t)\big)
\end{align}

where $\eta$ is the learning rate and,

\begin{align}
\Delta_t = \M_t \oslash(\sqrt{\V} + \epsilon) 
\label{eq:update}
\end{align}

where $(\cdot \oslash \cdot)$ denotes the element wise Hadamard division and $(\alpha, \epsilon) > 0$ where : $\alpha$ is the weight decay and $\epsilon$ is a constant added to avoid division by 0. $\M_t$ is the first momentum matrix defined as 

\begin{equation}
    \M_t = \beta_1\M_{t-1} + (1-\beta_1)\nabla_L(\W(t))
\end{equation}

and $\V$ is the second moment,

\begin{equation}
     \V_t = \beta_2\V_{t-1} + (1-\beta_2)\nabla_L(\W(t))^2
\end{equation}

Also, let \( \Delta_t^\perp \) be the matrix whose rows consist of the components orthogonal to \(\mathcal{S}\) from the rows of \( \Delta_t \). Assume that $g_t^2$ follows a sub-Gaussian distribution with a variance proxy $(\sigma(t))^2$ where $g_t$ is any element of \( \nabla_L (\mathbf{W}(t)) \). Then, with probability of at least $1-\delta$, $\lim_{t \to \infty} \frac{\| \Delta_t^\perp \|}{\| \Delta_t \|} = 0$ if  $\lim_{t \to \infty} \sigma(t) = 0$.



\label{th:weight_update}
\end{theorem}

\begin{proof}
\label{proof:subspacebound}










 Since the first-moment estimate $\M_t$ in AdamW remains a (biased) linear combination of past gradients $\nabla_L (\W(t))$, its rows stays within \(\mathcal{S}\) (recall that vector spaces are closed under addition). However, the rows of second-moment estimate \(\V_t\) are not constant-value vectors, and thus, the element-wise division of $\M_t$ by $(\sqrt{\V} + \epsilon)$ distorts and pushes the rows of $\M_t$ outsied $\mathcal{S}$. Consequently, the rows of $\Delta_t$ are not confined to $\mathcal{S}$. 
Our goal is to show that \(\frac{\|\Delta_t^\perp\|}{\| \Delta_t \|}\) asymptotically goes to zero with the variance decay.

\paragraph{Step 1: Bounding the second-moment ratio.}

Since the distribution of $g_t^2$ is sub-Gaussian, then for each $g_t$, we have 

\begin{equation}
    P(g_t^2 \geq x) \leq 2\exp{\Big(-\frac{x^2}{2\sigma^2}\Big)}.
\end{equation}

Note that we drop the coordinate indexes $(i,j)$ from $g_t$ for brevity. Using a union bound over $m \times n$ coordinates and at any given time $t$, we have

\begin{equation}
    P(\max_{i,j} g_t^2 \geq x) \leq 2mn\exp{\Big(-\frac{x^2}{2(\sigma(t))^2}\Big)}
\end{equation}

Then, with probability at least $1-\delta$, we get

\begin{equation}
    \max_{i,j}(g_t^2)  \leq \sigma(t) \sqrt{2\ln \Big(\frac{2mn}{\delta}\Big)}
    \label{eq:maxgt}
\end{equation}

Since any element $v_t$ of $\V_t$ is a moving average of $g_t^2$, we can expand it as,

\begin{equation}
    v_t = \beta_2^{(t-1)}(1-\beta_2)g_1^2 + \beta_2^{(t-2)}(1-\beta_2)g_2^2 + \dots \beta_2^2(1-\beta_2)g_{t-2}^2 + \beta_2(1-\beta_2)g_{t-1}^2 + (1-\beta_2)g_t^2
\end{equation}

(assuming $v_0 = 0$). Note that for large $t$, we can have a $\tau(k)$ such that $v_t - \tau(k) < \mu$ for arbitrary small $\mu$, where 

\begin{equation}
    \tau(k) = \beta_2^{(t-k)}(1-\beta_2)g_k^2 + \beta_2^{(t-k-1)}(1-\beta_2)g_{k+1}^2 + \dots \beta_2^2(1-\beta_2)g_{t-2}^2 + \beta_2(1-\beta_2)g_{t-1}^2 + (1-\beta_2)g_t^2
\end{equation}

Plugging the above result to Eq. \ref{eq:maxgt},  with probability at least $1-\delta$, we again have

\begin{equation}
    \max_{i,j} [v_t]_{i,j} \leq \max_{t \in [t-k, T]} \sigma(t) \sqrt{2\ln \Big(\frac{2mn}{\delta}\Big)}
\end{equation}

Then,

\begin{align}
\label{eq:upper_bound}
    \frac{\sqrt{\max_{i,j} [v_t]_{i,j}} + \epsilon}{\sqrt{\min_{i,j} [v_t]_{i,j}} + \epsilon} \leq 1 +  \frac{\sqrt{\max_{i,j} [v_t]_{i,j}}}{\epsilon}
  \;\le\; 1 + \frac{\sqrt{\max_{t \in [t-k, T]} \sigma(t) \sqrt{2\ln \Big(\frac{2mn}{\delta}\Big)}}}{\epsilon},
\end{align}

since $\sqrt{\min_{i,j} [v_t]_{i,j}} = 0$.  Let $\kappa(t) = \frac{\sqrt{\max_{t \in [t-k, T]} \sigma(t) \sqrt{2\ln \Big(\frac{2mn}{\delta}\Big)}}}{\epsilon}$. Then, the coordinate-wise learning-rate scaling 
\(
   \frac{1}{\sqrt{v_t} + \varepsilon}
\)
cannot differ among coordinates by more than \(1 + \kappa(t)\) with a high probability. 





\paragraph{Step 2: Relating off-subspace updates to the ratio.}  Consider any row vector of $\Delta_t$ to be $(\Delta_t)_{i,:}$. Then, by Applying the update Eq.~\ref{eq:upper_bound} to Lemma 1, and by definition of $\Delta_t$, we have


\begin{align}
\label{eq:kappa_inequal}
     \|(\Delta_t)_{i,:}^{\perp}\| \leq \frac{\kappa(t)}{2}\| (\Delta_t)_{i,:} \|
\end{align}


\paragraph{Step 3: Taking the limit \(t \to \infty\).}
Because we have 

\begin{equation}
\lim_{t \to \infty} \sigma(t) =0,
\end{equation}

trivially, we also have

\begin{equation}
\lim_{T \to \infty} \max_{t \in [T-k, T]} \sigma(t) =0.
\end{equation}

Then, it follows from the definition of $\kappa(t)$ that 

\begin{equation}
\lim_{t \to \infty} \kappa(t) =0.
\end{equation}

Therefore, from Eq.~\ref{eq:kappa_inequal}, assuming $\|(\Delta_t)_{i,:}\| \neq 0$ we finally have

\begin{align}
  0 \leq  \lim_{t \to \infty} \frac{\|(\Delta_t)_{i,:}^{\perp}\|}{\|(\Delta_t)_{i,:}\|} \leq 0
\end{align}

\begin{align}
    =\lim_{t \to \infty} \frac{\|(\Delta_t)_{i,:}^{\perp}\|}{\|(\Delta_t)_{i,:}\|} = 0
\end{align}

 In simpler terms, the distortion of the rows of $\Delta_t$ outside \(\mathcal{S}\) is ultimately goes to zero. 

\end{proof}

\noindent{\textbf{Discussion:} Above theorem gives us an important insight. In particular, it ensures that if the gradients of an unrestricted network predominantly lie in a specific low-dimensional subspace $\mathcal{S}$, then the AdamW optimizer restricts updates outside that subspace. On the other hand, as training progresses, gradient norms decrease and become more uniform across coordinates, which has  been rigorously proved in previous works within reasonable bounds \cite{li2024convergence, zhang2024convergence, defossez2020simple}.   Consequently, the variance of the gradients also tends to diminish, justifying our assumption of $\lim_{t \to  \infty} \sigma(t) \to 0$. This reduction further curtails the extent of any \emph{off-subspace} updates, thereby indicating that AdamW naturally remains focused on the directions most important to learning.}

Further, empirical evidence supports the notion that gradients of unconstrained networks often lie predominantly within a low-dimensional subspace (we also empirically validate this in Fig.~\ref{fig:gradients}). This observation aligns with the hypothesis that many deep learning problems exhibit significant redundancy in parameter updates, with key directions of optimization being confined to a smaller subspace. By leveraging this inherent structure, it becomes possible to guide optimization dynamics in a manner that preserves computational efficiency while maintaining convergence. We conduct an extended discussion next.

\subsection{On the assumptions of Theorem \ref{th:weight_update}}
\label{sec:app_assumptions}

For Theorem \ref{th:weight_update}, we employ two assumptions 1) the gradients of weight matrices are confined to a low-dimensional subspace and 2) the variance of the gradients diminish as the model converges. Next, we provide justifications for these assumptions.

\subsubsection{ The gradients of weight matrices are confined to a low-dimensional subspace}

To empirically validate this behavior, we train an 8-layer network on C4 and track the stable rank progression of its projection matrices throughout training. As shown in Fig.~\ref{fig:gradients}, the gradients maintain an exceptionally low stable rank relative to the maximum possible rank from initialization to convergence. This confirms that the gradients primarily focus on a few dominant directions. 

\begin{figure*}[ht]
    \centering
    \begin{subfigure}[b]{0.25\textwidth}
        \includegraphics[width=\textwidth]{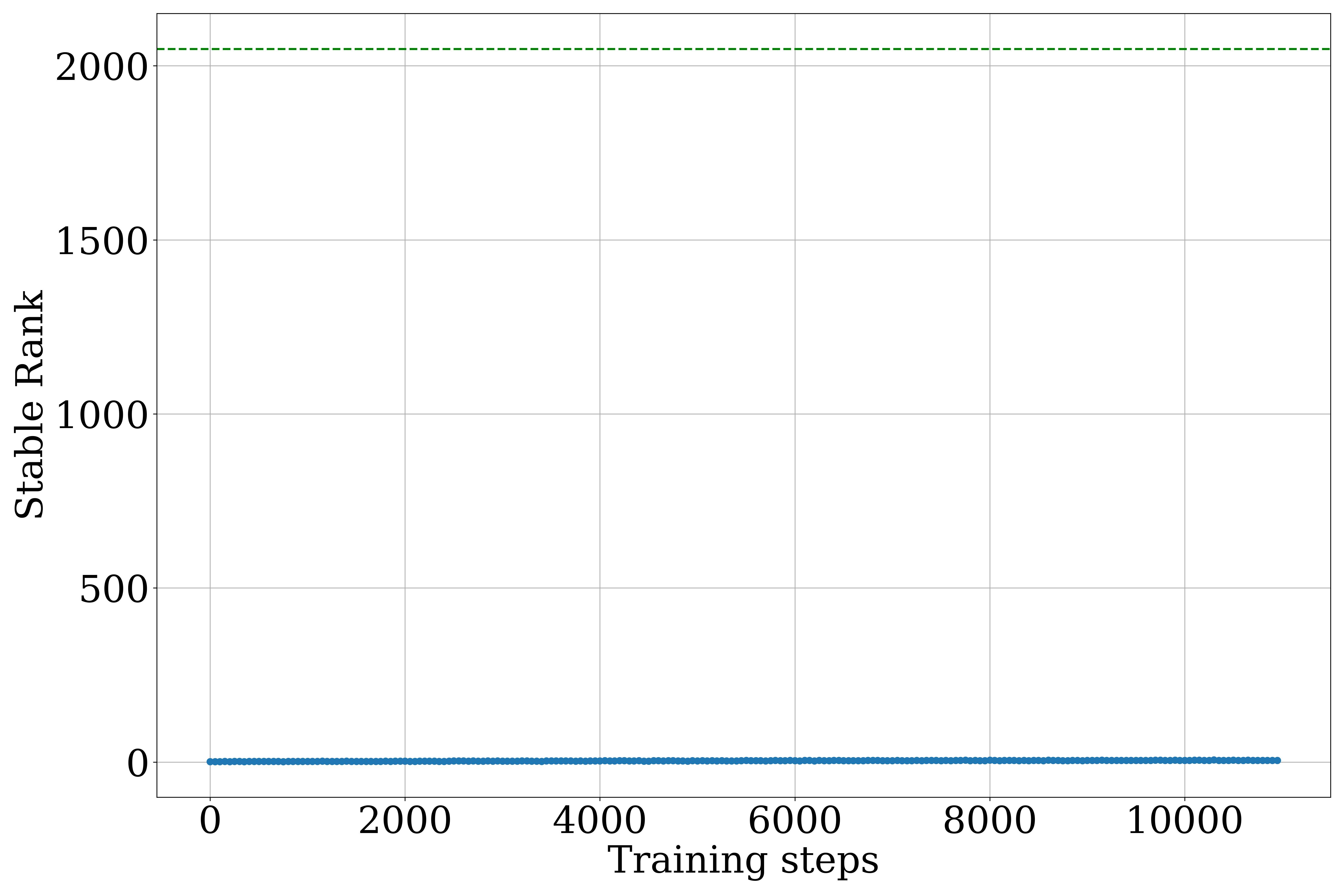}
        \caption{$\nabla_{l}\W^4_{p_2}$}
        \label{fig:sub1b}
    \end{subfigure}
    \hfill
    \begin{subfigure}[b]{0.25\textwidth}
        \includegraphics[width=\textwidth]{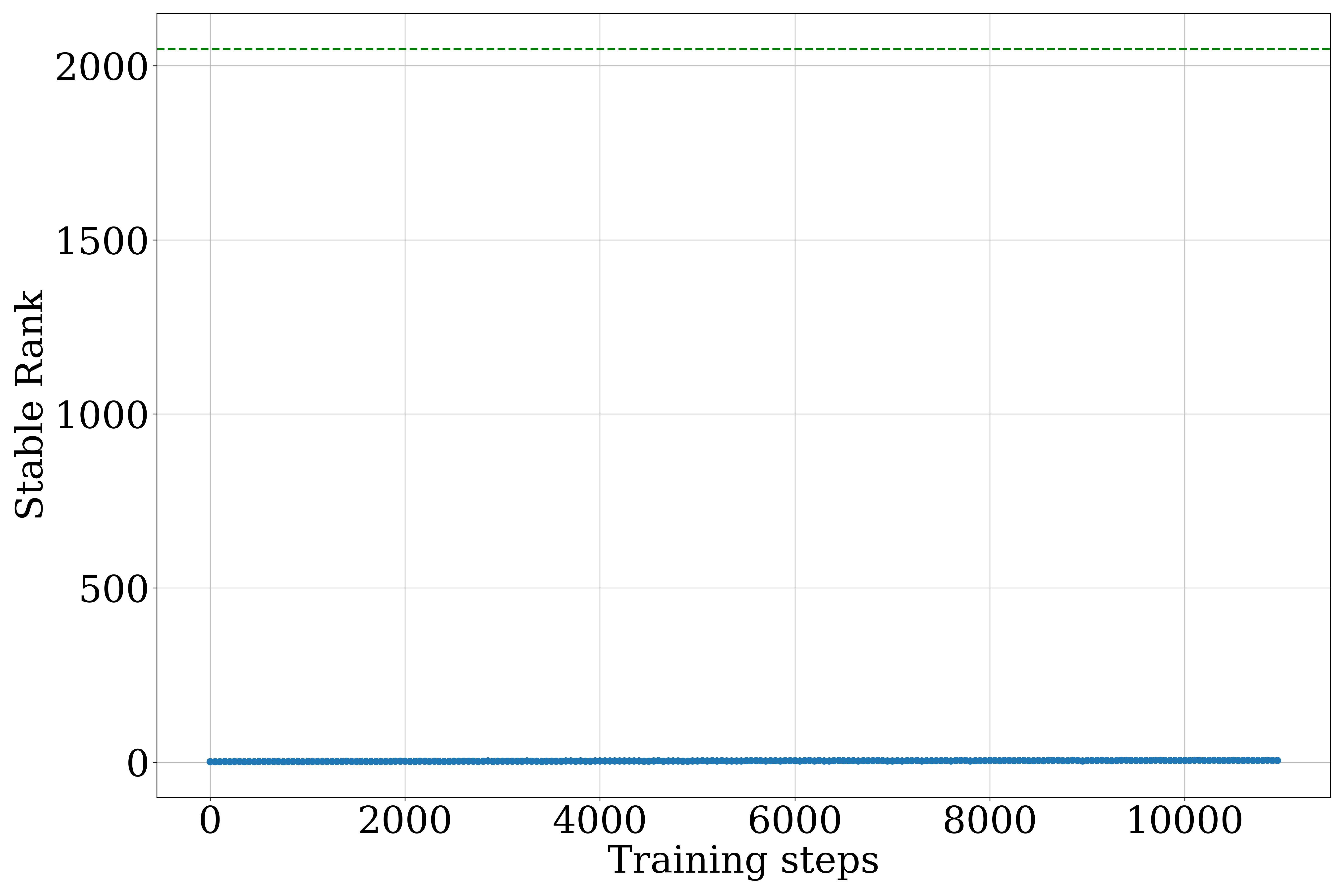}
        \caption{$\nabla_{l}\W^7_{p_2}$}
        \label{fig:sub2b}
    \end{subfigure}
    \hfill
    \begin{subfigure}[b]{0.25\textwidth}
        \includegraphics[width=\textwidth]{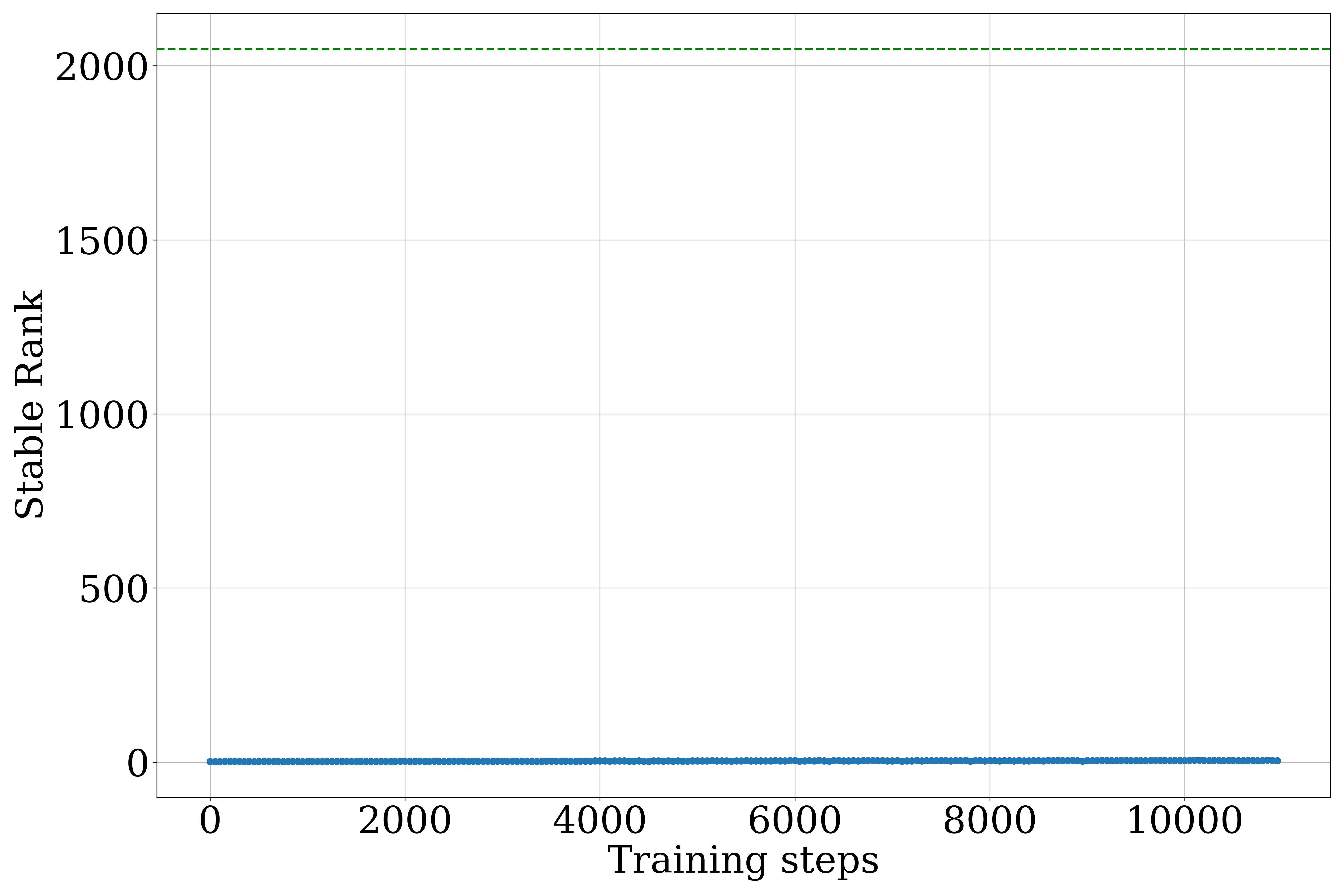}
        \caption{$\nabla_{l}\W^7_{p_1}$}
        \label{fig:sub3}
    \end{subfigure}
    \begin{subfigure}[b]{0.25\textwidth}
        \includegraphics[width=\textwidth]{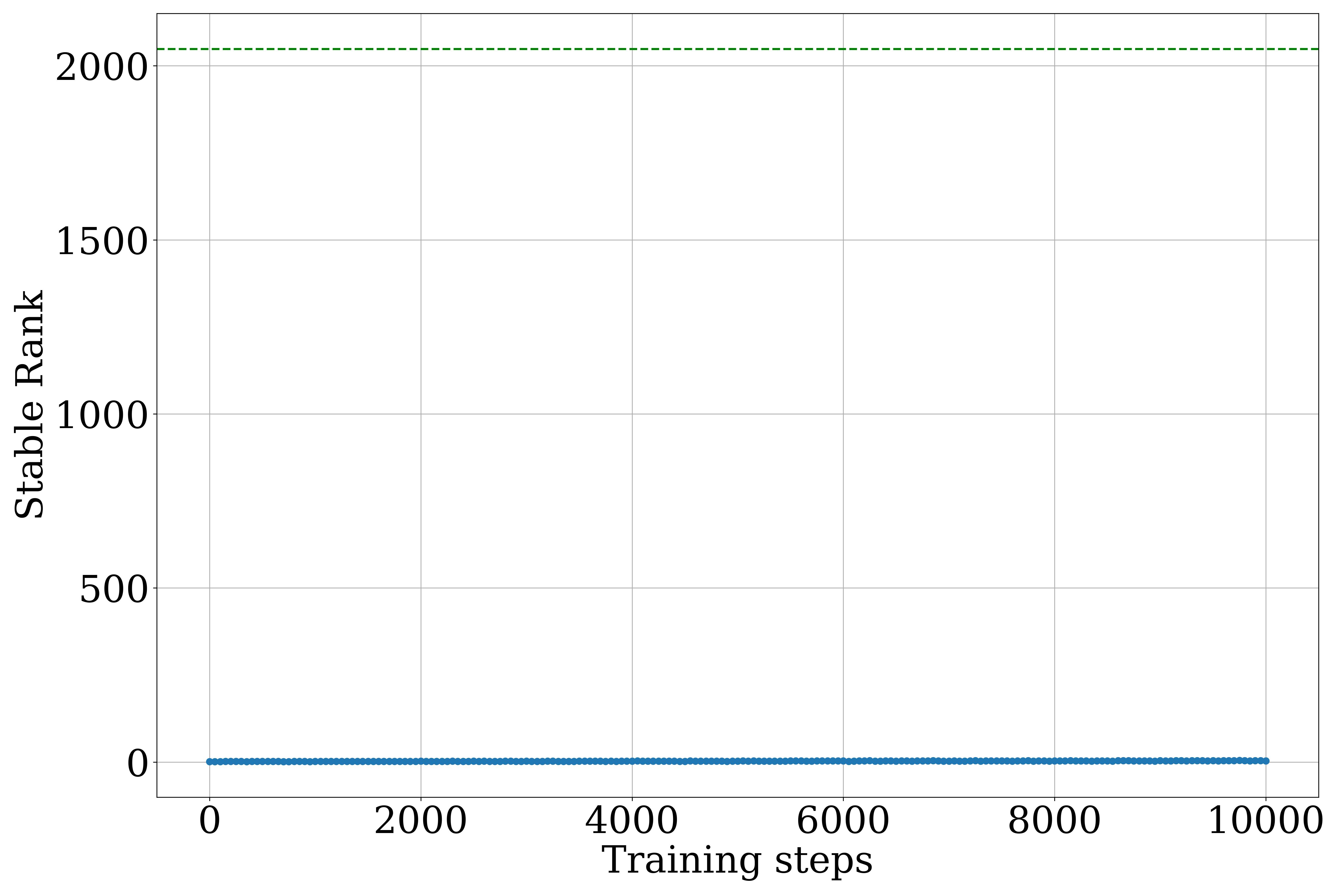}
        \caption{$\nabla_{l}\W^4_{p_1}$}
        \label{fig:sub3}
    \end{subfigure}
    \caption{\textbf{Low dimensionality of the gradients of projection matrices.} We measure the stable ranks of gradients of the projection matrices over the course of training. In each plot, the green dashed line indicates the maximum rank. As shown, all the weight matrices consistently maintain a significantly lower rank ($\sim~< 5$) throughout the training. This validates the assumption of our theorem \ref{th:weight_update} that the weight gradients of the projection matrices are confined to a low-dimensional subspace.}
    \label{fig:gradients}
\end{figure*}


The assumption that weight gradients span a low-dimensional subspace is well-supported by prior research. \cite{gur2018gradient} demonstrated that gradient descent in deep learning does not explore the full parameter space but instead operates within a constrained subspace, largely dictated by the spectral structure of the Hessian. The Hessian matrix features a dominant subspace associated with large eigenvalues, where most gradient updates occur, and a bulk subspace characterized by near-zero eigenvalues that correspond to flat directions in the loss landscape. Empirical evidence indicates that early in training, gradients align with this dominant Hessian subspace and remain confined within it throughout the optimization process.

This phenomenon extends to the learned representations themselves. \cite{zhao2021zero} showed that gradient descent naturally gravitates toward low-rank solutions, a behavior arising from implicit regularization within the optimization process. Instead of initializing with full-rank representations, training typically begins in a low-rank regime and expands rank as necessary. This trend has been consistently observed across multiple architectures, including MLPs, CNNs, and transformers. The concept of spectral bias further supports this observation, suggesting that neural networks initially learn broad, low-rank structures before refining finer details.

A practical application of this low-rank property is evident in \cite{cosson2023low}'s introduction of Low-Rank Gradient Descent (LRGD), which explicitly leverages low-rank subspaces to expedite training. By identifying dominant gradient directions and confining updates to this reduced subspace, LRGD significantly lowers computational costs without sacrificing optimization effectiveness. Theoretical results demonstrate that in strongly convex settings, LRGD reduces oracle complexity from \(O(p \log(1/\epsilon))\) to \(O(r \log(1/\epsilon) + rp)\), where \(r \ll p\). Similar improvements are observed in non-convex environments. Moreover, \cite{wangni2018gradient} illustrated that in deep networks trained with cross-entropy loss, the Hessian spectrum exhibits a clear separation. The top subspace, whose dimensionality approximates the number of classes in the dataset, captures the majority of gradient activity. In contrast, the bulk subspace, filled with small eigenvalues, corresponds to nearly flat directions in the loss landscape. This structure implies that significant changes in the loss function occur primarily within the top subspace. Empirical findings also reveal that gradients quickly align with the top Hessian subspace early in training and persist within it. This effect has been documented across diverse architectures such as fully connected networks, convolutional networks, ResNets, and transformers. While individual Hessian eigenvectors evolve during training, the dominant subspace remains relatively stable, guiding the optimization path. The low-rank nature of gradients is further validated by \cite{vogels2019powersgd}, who observed a rapid decay in the singular values of gradient matrices. This decay implies that the most meaningful updates are confined to a small, structured subspace. Additionally, the gradient covariance matrix in deep networks shows exponential spectral decay, suggesting that full-rank updates are not essential.

Above findings provide robust empirical and theoretical support for our low-rank gradient assumption, directly motivating our approach. By exploiting low-rank constraints, we achieve substantial reductions in both computation and communication overhead while preserving performance. Given that deep learning models inherently exhibit implicit regularization, explicitly enforcing a low-rank structure in our compression could even induce a beneficial regularization effect—potentially enhancing performance beyond that of non-compressed models, as indicated in Table~\ref{tab:perplexity}.

\subsubsection{The variance of the gradients diminish as the model converges. }

The reduction in gradient variance as a model converges is a well-observed phenomenon in optimization literature. As training progresses, stochastic gradient estimates become more stable, leading to smoother updates and accelerated convergence. This effect is particularly pronounced in adaptive learning rate optimizers such as Adam \cite{kingma2014adam} and AdamW \cite{loshchilov2017decoupled}. In particular, \cite{balles2018dissecting} showed that these methods inherently reduce the variance of parameter updates over time, contributing to stable convergence. \cite{liu2020adam} further refined this idea by introducing a variant of Adam designed to enforce adaptive variance reduction, ensuring that the variance of the stochastic gradient estimator decreases as the model approaches optimality.

Beyond adaptive optimization, variance reduction techniques have been widely studied in the context of stochastic gradient descent (SGD). Traditional SGD suffers from high variance in gradient estimates, particularly in early training stages, which can hinder convergence. To mitigate this, methods such as Stochastic Variance Reduced Gradient (SVRG) and SAGA \cite{wang2013variance} have been proposed. These approaches reduce variance by incorporating control variates or maintaining a memory of past gradients, yielding more accurate updates and improved convergence rates. Additional techniques, including Kalman-based filtering \cite{vuckovic2018kalman} and stochastic filtering methods \cite{yang2020stochastic}, further enhance stability by adaptively reducing noise in gradient estimates.

Normalization layers have also been shown to smooth parameter spaces and gradients, effectively reducing variance. This smoothness, characterized by a lower Lipschitz constant, results in more predictive gradients, which can be interpreted as a form of implicit variance reduction. The effect is particularly beneficial in deep networks, where sharp loss landscapes can otherwise lead to erratic updates.

In summary, diminishing gradient variance is a natural outcome of optimization dynamics. As training progresses and model parameters approach optimality, stochastic gradients become increasingly stable, even in the absence of explicit variance reduction techniques. This effect can be attributed to the curvature of the loss landscape: near an optimal solution, the loss function tends to be flatter, leading to more consistent gradient estimates across mini-batches. Consequently, as the model converges, gradient variance inherently diminishes, reinforcing the validity of this assumption in our work.

\section{Continuous-Time Gradient Flow with Weight Decay Converges to a Low-Dimensional Subspace}
\label{sec:app_weightconvergence}

The result in this section reveals a fundamental property of optimization under AdamW: the optimizer inherently drives weight matrices to converge asymptotically to a shared subspace, regardless of their initial values. This behavior stems from the decoupled weight decay mechanism, which systematically suppresses components of the weight matrix that receive negligible or noisy gradient updates during training. Over time, the weight matrix becomes confined to a low-dimensional subspace spanned by the significant directions of the gradient updates, leading to an effectively low-rank representation.

\begin{theorem}
\label{th:weight_convergence}

Consider a continuous-time dynamics for $\mathbf{W}(t) \in \mathbb{R}^{m \times n}$ governed by
\[
  \frac{\mathrm{d}\mathbf{W}}{\mathrm{d}t}
  \;=\;
  -\,\eta\,\Bigl(\,\Delta_t \;+\;\alpha\,\mathbf{W}(t)\Bigr),
\]
where $\eta,\;\alpha>0$ are constants, and $\Delta_t$ represents the gradient-like update without weight decay (i.e.\ the ``AdamW'' part, ignoring the $\alpha\,\mathbf{W}(t)$ term).
Assume that for large $t$, $\Delta_t$ has its rows confined predominantly to a subspace $\mathcal{S}\subseteq\mathbb{R}^n$ (so that the orthogonal component of $\Delta_t$ vanishes asymptotically).
Then, as $t \to \infty$, the component of $\mathbf{W}(t)$ lying in $\mathcal{S}^\perp$ decays exponentially to zero.  Consequently, the rows of the final solution $\mathbf{W}(\infty)$ lie in (or arbitrarily close to) $\mathcal{S}$.
\end{theorem}

\begin{proof}
Let $\Pi_{\mathcal{S}}$ denote the orthogonal projector onto $\mathcal{S}$ (applied rowwise to any $m\times n$ matrix).  Decompose:
\[
  \mathbf{W}(t) 
  \;=\;
  \mathbf{W}_{\parallel}(t)
  \;+\;
  \mathbf{W}_{\perp}(t),
  \quad
  \text{where}
  \quad
  \mathbf{W}_{\parallel}(t) \;=\;\Pi_{\mathcal{S}}\bigl[\mathbf{W}(t)\bigr],
  \quad
  \mathbf{W}_{\perp}(t) \;=\;\bigl(\mathrm{I}-\Pi_{\mathcal{S}}\bigr)\mathbf{W}(t).
\]

By hypothesis (see also Theorem~\ref{th:weight_update}), the rows of $\Delta_t$ become confined to $\mathcal{S}$ as $t\to\infty$, i.e.
\[
  \lim_{t\to\infty} \bigl(\mathrm{I} - \Pi_{\mathcal{S}}\bigr)\Delta_t
  \;=\; 
  0.
\]
Thus, for large $t$, the orthogonal component of $\Delta_t$ is negligible.  Project the ODE onto $\mathcal{S}^\perp$:
\[
  \frac{\mathrm{d}\mathbf{W}_{\perp}(t)}{\mathrm{d}t}
  \;=\;
  \bigl(\mathrm{I}-\Pi_{\mathcal{S}}\bigr)\frac{\mathrm{d}\mathbf{W}}{\mathrm{d}t}
  \;=\;
  -\,\eta\,\bigl(\mathrm{I}-\Pi_{\mathcal{S}}\bigr)\!\Bigl(
    \Delta_t + \alpha\,\mathbf{W}(t)
  \Bigr).
\]
For sufficiently large $t$, $(\mathrm{I}-\Pi_{\mathcal{S}})\Delta_t$ is negligible, so asymptotically,
\[
  \frac{\mathrm{d}\mathbf{W}_{\perp}(t)}{\mathrm{d}t}
  \;\approx\;
  -\,\eta\,\alpha\,\bigl(\mathrm{I}-\Pi_{\mathcal{S}}\bigr)\!\bigl(\mathbf{W}(t)\bigr)
  \;=\;
  -\,\eta\,\alpha\,\mathbf{W}_{\perp}(t).
\]
This is a linear ODE:
\[
  \frac{\mathrm{d}\mathbf{W}_{\perp}(t)}{\mathrm{d}t}
  \;=\;
  -\,\eta\,\alpha\,\mathbf{W}_{\perp}(t),
\]
whose solution is
\[
  \mathbf{W}_{\perp}(t)
  \;=\;
  \exp\bigl(-\,\eta\,\alpha\,t\bigr)\,\mathbf{W}_{\perp}(0).
\]
Hence $\mathbf{W}_{\perp}(t)$ decays exponentially to $\mathbf{0}$.  Therefore,
\[
  \mathbf{W}(t)
  \;=\;
  \mathbf{W}_{\parallel}(t) + \mathbf{W}_{\perp}(t)
  \;\longrightarrow\;
  \mathbf{W}_{\parallel}(t),
  \quad
  \text{with}
  \quad
  \mathbf{W}_{\parallel}(t)\;\in\;\mathcal{S}.
\]
Thus, as $t\to\infty,$ the rows of $\mathbf{W}(t)$ lie in $\mathcal{S}$ (or arbitrarily close), proving the claim.
\end{proof}

\noindent{\textbf{Discussion:}} Unlike standard stochastic gradient descent (SGD), which applies weight decay as a modification to the gradient updates, AdamW separates the weight decay term from the gradient computation. This decoupled weight decay has two key effects: 1.) \textbf{Suppression of Insignificant Components:} Components of the weight matrix orthogonal to the subspace spanned by significant gradient updates decay exponentially due to weight decay. This ensures that the optimization focuses on meaningful directions of the gradient. 2) \textbf{Low-Dimensional Convergence:} By amplifying the impact of directions where the gradient signal is strongest, the optimizer effectively projects the weight matrix onto a subspace defined by these directions, leaving other components asymptotically negligible.


To empirically validate this property, we analyzed the spectral structure of weight matrices trained under AdamW. As shown in Fig.~\ref{fig:rank_collapse}, the spectrum of these matrices exhibits rapid decay, with the majority of the spectral mass concentrated in a few dominant directions.

\section{Convergence of Partially-Projected Gradient Descent}
\label{sec:app_partialconvergence}

This section establishes that constraining a subset of weights in a neural network to lie within a low-dimensional subspace does not degrade the theoretical convergence guarantees of constrained optimization. Specifically, the network still achieves convergence to a first-order stationary point at the rate of \(O(1/T)\), where \(T\) is the number of optimization steps. This is consistent with standard results in constrained optimization and provides a preliminary theoretical foundation for employing subspace constraints in practical applications.

\begin{proposition}
\label{thm:convergence}
Let $f(\W^{(s)}, \W^{(u)})$ be a possibly non-convex function that is $L$-smooth in all its parameters $\W=(\W^{(s)},\W^{(u)})$, where $\W^{(s)} \in \mathbb{R}^{k}$ is a \emph{constrained} block and $\W^{(u)} \in \mathbb{R}^{d}$ is \emph{unconstrained} and $k < d$. 
Suppose we have a closed subspace $\mathcal{S} \subseteq \mathbb{R}^{k}$, and we define the feasible set
\[
  \W \;=\; \mathcal{S} \times \mathbb{R}^{d_u}.
\]
Consider the following \textbf{partially-projected gradient step}:
\[
  \widetilde{\W}_{t+1}^{(s)} 
  \;=\; \W_{t}^{(s)} - \eta \,\nabla_{(s)}f(\W_t), 
  \quad
  \widetilde{\W}_{t+1}^{(u)} 
  \;=\; \W_{t}^{(u)} - \eta \,\nabla_{(u)}f(\W_t),
\]
\[
  \W_{t+1}^{(s)} 
  \;=\; \Pi_{\mathcal{S}}\Bigl(\widetilde{\W}_{t+1}^{(s)}\Bigr),
  \quad
  \W_{t+1}^{(u)} 
  \;=\; \widetilde{\W}_{t+1}^{(u)},
\]
where $\Pi_{\mathcal{S}}(\cdot)$ is the Euclidean projection onto $\mathcal{S}$, $\nabla_{(S)}f$ and $\nabla_{(u)}f$ are the gradients of $f$ w.r.t.\ the constrained and unconstrained blocks, respectively, and $\eta>0$ is a fixed step size. 

Then, for any $T\ge1$ and $\eta \le \frac{1}{L}$, the iterates $\{\W_t\}$ satisfy the following \emph{stationary} guarantee:
\[
\min_{0 \,\le\, t < T}
\bigl\|\nabla f(\W_t)\bigr\|^2
\;\;\le\;\; 
\frac{2\bigl(f(\W_0) \;-\; f^*\bigr)}{\eta\,T}
\]
where $f^* = \min_{\W\in \W} f(\W)$ and $\W_0$ is the initial parameter vector.

\end{proposition}

\begin{proof}

Since $f(\W^{(s)}, \W^{(u)})$ is $L$-smooth in all parameters, for any $\W,\W'$ we have
\begin{equation}
      f(\W') 
  \;\le\; 
  f(\W) 
  \;+\; \nabla f(\W)^\top(\W' - \W)
  \;+\; \tfrac{L}{2}\|\W' - \W\|^2.
  \label{eq:convergence_1}
\end{equation}

Setting $\W = \W_t$ and $\W' = \W_{t+1}$, we get

\begin{equation}
      f(\W_{t+1}) 
  \;\le\; 
  f(\W_t) 
  + \nabla f(\W_t)^\top (\W_{t+1} - \W_t)
  + \tfrac{L}{2} \,\|\W_{t+1} - \W_t\|^2.
  \label{eq:convergence_2}
\end{equation}

Recall the update:
\begin{equation}
    \W_{t+1}^{(s)} 
  = \Pi_{\mathcal{S}}\Bigl(\W_{t}^{(S)} - \eta\,\nabla_{(S)} f(\W_t)\Bigr),
  \quad
  \W_{t+1}^{(u)} 
  = \W_{t}^{(u)} - \eta\,\nabla_{(u)} f(\W_t).
  \label{eq:convergence_3}
\end{equation}
  
\begin{equation}
    \W_{t+1} 
  \;=\; 
  \Bigl(\,\Pi_{\mathcal{S}}(\W_t^{(s)} - \eta\,\nabla_{(s)} f(\W_t)),
         \; \W_t^{(u)} - \eta\,\nabla_{(u)} f(\W_t)\Bigr).
         \label{eq:convergence_4}
\end{equation}

Projecting  the $\widetilde{\W}_{t+1}^{(s)}$ block onto $\mathcal{S}$, is a 1-Lipschitz (nonexpansive) operation in the full Euclidean space.  Concretely, consider $ \|\bigl(\Pi_{\mathcal{S}}(\widetilde{\W}_{t+1}^{(s)}) - \W_{t}^{(s)})\|$. Also recall that 
 \(\W_t^{s} \in\mathcal S\). It follows from the classical projection lemma that,

\begin{equation}
   \|(\W_{t+1}^{(s)}) - \W_{t}^{(s)})\| =  \|\bigl(\Pi_{\mathcal{S}}(\widetilde{\W}_{t+1}^{(s)}) - \W_{t}^{(s)})\| \leq \|\big(\widetilde{\W}_{t+1}^{(s)} - \W_{t}^{(s)})\|
    \label{eq:proj_lemma}
\end{equation}

This happens because the projection minimizes the Frobenius norm to the subset, so the projected point is closer than the unprotected one.

We also leave the disjoint set $\W^{(u)}$ untouched. So we have,

\begin{equation}
\widetilde{\W}_{t+1}^{(u)} = \W_{t+1}^{(u)}
    \label{eq:u_equal}
\end{equation}

Since $\W^{(s)}, \W^{(u)}$ are disjoint blocks, We can combine the blocks (Pythagorean identity):
The Frobenius norm is the Euclidean norm on the concatenated vector, so
\[
\bigl\|\W_{t+1} - \W_t\bigr\|^2
=
\|\big(\W_{t+1}^{(s)} - \W_{t}^{(s)})\|^2 
+
\|\big(\W_{t+1}^{(u)} - \W_t^{(u)}\big)\|^2
\]

From Eq.~\ref{eq:proj_lemma} and \ref{eq:u_equal} we have,

\begin{equation} 
\bigl\|\W_{t+1} - \W_t\bigr\|^2
\leq
\|\big(\widetilde{\W}_{t+1}^{(s)} - \W_{t}^{(s)})\|^2 
+
\|\big(\widetilde{\W}_{t+1}^{(u)} - \W_t^{(u)}\big)\|^2
\end{equation}

As $\W^{(s)}, \W^{(u)}$ are disjoint, we get,

\begin{equation}
     \bigl\|\W_{t+1} - \W_t\bigr\|^2 \leq \bigl\|\widetilde{\W}_{t+1} - \W_t\bigr\|^2
\end{equation}

\begin{equation}
     \bigl\|\W_{t+1} - \W_t\bigr\| \leq \bigl\|\widetilde{\W}_{t+1} - \W_t\bigr\|
\end{equation}

 $(\widetilde{\W}_{t+1} - \W_t)$ is the update step which is equal to $\eta ( \nabla f(\W_t) )$. Thus,

\begin{equation}
    \|\W_{t+1} - \W_t\| 
  \;\le\;
  \eta\,\bigl\|\nabla f(\W_t)\bigr\|.
  \label{eq:convergence_8}
\end{equation}

Plugging this bound back into the smoothness inequality (Eq.~\ref{eq:convergence_2}), we can get

\begin{equation}
      f(\W_{t+1}) 
  \;\le\; 
  f(\W_t) 
  + \nabla f(\W_t)^\top (\W_{t+1} - \W_t)
  + \tfrac{L}{2} \,\eta^2\,\bigl\|\nabla f(\W_t)\bigr\|^2.
  \label{eq:convergence_12}
\end{equation}

Further, from the optimality condition, it follows that,

\begin{equation}
   \big( \W_{t+1} - (\W_t - \eta \nabla f(\W_t)) \big)^{\top} (\W_{t+1} - \W_t) \leq 0
\end{equation}

\begin{align}
   \big(  \eta \nabla f(\W_t)) ^{\top} (\W_{t+1} - \W_t) \leq  -\|\W_{t+1} - \W_t\|^2
\end{align}

\begin{align}
   \ \nabla f(\W_t) ^{\top} (\W_{t+1} - \W_t) \leq  - \frac{1}{\eta}\|\W_{t+1} - \W_t\|^2
\end{align}

Substituting above in Eq.~\ref{eq:convergence_12},

\begin{equation}
     f(\W_{t+1}) 
  \;\le\; 
  f(\W_t) 
    - \frac{1}{\eta}\|\W_{t+1} - \W_t\|^2
  + \tfrac{L}{2} \,\eta^2\,\bigl\|\nabla f(\W_t)\bigr\|^2.
\end{equation}

\begin{equation}
     f(\W_{t+1}) 
  \;\le\; 
  f(\W_t) 
    - \frac{1}{\eta}\|\W_{t+1} - \W_t\|^2
  + \tfrac{L}{2} \,\|\W_{t+1} - \W_t\|^2.
\end{equation}

\begin{equation}
     f(\W_{t+1}) 
  \;\le\; 
  f(\W_t) 
  + \|\W_{t+1} - \W_t\|^2 \Big( \frac{L}{2} - \frac{1}{\eta} \Big).
\end{equation}

\begin{equation}
     f(\W_{t+1}) -  f(\W_t) 
  \;\le\; 
   \|\W_{t+1} - \W_t\|^2 \Big( \frac{L}{2} - \frac{1}{\eta} \Big).
\end{equation}

\begin{equation}
     f(\W_{t+1}) -  f(\W_t) 
  \;\le\; 
   \eta^2 \| \nabla f(\W_t) \|^2 \Big( \frac{L}{2} - \frac{1}{\eta} \Big).
\end{equation}

Choosing $\eta < \frac{2}{L}$

\begin{equation}
     f(\W_{t+1}) -  f(\W_t) 
  \;\le\; 
   -\gamma \| \nabla f(\W_t) \| ^2.
\end{equation}

Taking a telescoping summation, we have,

\begin{equation}
    \sum_{t=0}^{T-1} \Big( f(\W_{t+1}) - f(\W_t) \Big)
  \;\le\; 
  -\gamma \sum_{t=0}^{T-1}  \| \nabla f(\W_t) \|^2
\end{equation}

\begin{equation}
     \frac{\Big( f(\W_0) - f(\W_{t+1}) \Big)}{\gamma}
  \geq
  \sum_{t=0}^{T-1}  \| \nabla f(\W_t) \|^2. 
\end{equation}

Let us set $\W_{t+1} = \W^*$, where $\W^*$ is a minimizer of $f$. Then, we have $\Big( f(\W_0) - f(\W_{t+1}) \Big) \leq \Big( f(\W_0) - f(\W^*) \Big)$, which gives,

\begin{equation}
     \frac{\Big( f(\W_0) - f(\W^*) \Big)}{\gamma}
  \geq
  \sum_{t=0}^{T-1}  \| \nabla f(\W_t) \|^2. 
\end{equation}

Dividing both sides by $T$, we get

\begin{equation}
     \frac{\Big( f(\W_0) - f(\W^*) \Big)}{\gamma T}
  \geq
 \frac{1}{T} \sum_{t=0}^{T-1}  \| \nabla f(\W_t) \|^2. 
\end{equation}

If the gradient norm is upper bounded as below, then, trivially we have,

\begin{equation}
    \min_{0 \,\le\, t < T}
\bigl\|\nabla f(\W_t)\bigr\|^2
\;\;\le\;\; 
\frac{2\bigl(f(\W_0) \;-\; f(\W^*)\bigr)}{\gamma\,T}
\end{equation}

This shows that the squared minimum gradient norm goes to zero at a $O(1/T)$ rate, implying \emph{convergence to a stationary point} in the non-convex sense. 
\end{proof}



\noindent\textbf{Discussion: }The result is a direct extension of the convergence guarantees for proximal gradient descent methods on non-convex functions. Proximal gradient methods are widely used in constrained optimization problems, where the update rule involves projecting the parameters onto a feasible set at each iteration. In the case of subspace-constrained networks, the feasible set corresponds to the predefined low-dimensional subspace, and the projection operator ensures that the parameters \emph{partially} remain within this subspace. The convergence rate of \(O(1/T)\) in terms of stationarity reflects that, even in the presence of non-convexity and subspace constraints, the optimization algorithm efficiently reduces the gradient norm over time, driving the network parameters toward a stationary point. A possible implication of subspace restriction could be that it acts as an implicit regularizer, encouraging the network to learn more compact and robust representations. This can lead to improved generalization performance on unseen data. We hypothesize that this might be the reason for improved performance of the compressed models over their centralized counterparts (see Fig.~\ref{fig:llama_8b}).


Next result extends the above theorem to the stochastic mini-batch gradient descent optimization.
\begin{theorem}
\label{thm:sgd_ppgd}
Let
\(f(\W^{(s)},\W^{(u)})\) be \emph{$L$–smooth} in
\(\W=(\W^{(s)},\W^{(u)})\) with
\(\W^{(s)}\in\R^{k}\) (\emph{constrained block}),
\(\W^{(u)}\in\R^{d}\) (\emph{unconstrained block}),
and let
\(\mathcal S\subseteq\R^{k}\) be a closed subspace.
Define the feasible set \(\W:=\mathcal S\times\R^{d}\).

At each iteration \(t\) draw a mini-batch and compute an
\emph{unbiased} stochastic gradient
\(
      \g_t=(\g_t^{(s)},\g_t^{(u)})
\)
satisfying
\begin{align}
   \E[\g_t\mid\W_t] &= \nabla f(\W_t),
   &
   \E\bigl[\|\g_t-\nabla f(\W_t)\|^2\mid\W_t\bigr]
   &\le \sigma^2               \quad(\text{bounded variance}).
\end{align}

Update the parameters with the
\emph{stochastic partially-projected gradient} (SPPG) step
\[
\begin{aligned}
   \widetilde{\W}_{t+1}^{(s)} &= \W_t^{(s)}-\eta_t\g_t^{(s)}, &
   \widetilde{\W}_{t+1}^{(u)} &= \W_t^{(u)}-\eta_t\g_t^{(u)},\\
   \W_{t+1}^{(s)} &= \Pi_{\mathcal S}\!\bigl(\widetilde{\W}_{t+1}^{(s)}\bigr), &
   \W_{t+1}^{(u)} &= \widetilde{\W}_{t+1}^{(u)},
\end{aligned}
\]
where each stepsize satisfies \(\eta_t\in(0,1/L]\).

Let
\(
      f^{*}:=\min_{\W\in\W}f(\W)
\)
and denote
\(
      \bar\eta:=\frac1T\sum_{t=0}^{T-1}\eta_t
\).
Then for every horizon \(T\ge1\)
\begin{equation}
\label{eq:sgd_ppgd_rate}
      \frac1T\sum_{t=0}^{T-1}\E\bigl[\|\nabla f(\W_t)\|^2\bigr]
      \;\le\;
      \frac{2\bigl(f(\W_0)-f^{*}\bigr)}{\bar\eta\,T}
      \;+\;L\,\bar\eta\,\sigma^{2}.
\end{equation}
In particular, with the \emph{constant} stepsize
\(
      \eta_t\equiv\eta\le1/L
\)
it holds that
\[
      \min_{0\le t<T}\E\bigl[\|\nabla f(\W_t)\|^2\bigr]
      \;\le\;
      \frac{2\bigl(f(\W_0)-f^{*}\bigr)}{\eta T}
      +L\,\eta\,\sigma^{2},
\]
which attains the optimal
\(
      \mathcal O\!\bigl(1/\sqrt{T}\bigr)
\)
rate when
\(
      \eta=\min\!\bigl\{1/L,\,
      \sqrt{2(f(\W_0)-f^{*})/(L\sigma^{2}T)}\bigr\}
\).
\end{theorem}

\begin{proof}
The argument follows the deterministic proof but replaces exact
gradients with the stochastic estimator and takes conditional
expectations at the appropriate points.

Because the Euclidean projection onto a closed subspace is
\(1\)-Lipschitz, for every \(t\)
\(
      \|\W_{t+1}^{(s)}-\W_t^{(s)}\|
      \le\|\widetilde{\W}_{t+1}^{(s)}-\W_t^{(s)}\|
      =\eta_t\|\g_t^{(s)}\|.
\)
The unconstrained coordinates remain unprojected, hence
\(
      \|\W_{t+1}^{(u)}-\W_t^{(u)}\|
      =\eta_t\|\g_t^{(u)}\|.
\)
Using the Pythagorean theorem for the two disjoint blocks,
\begin{equation}\label{eq:step_norm_bound}
      \|\W_{t+1}-\W_t\|
      \;\le\;\eta_t\|\g_t\|.
\end{equation}

\(L\)-smoothness of \(f\) implies
\[
      f(\W_{t+1})
      \le
      f(\W_t)
      +\nabla f(\W_t)^{\!\top}(\W_{t+1}-\W_t)
      +\frac L2\|\W_{t+1}-\W_t\|^{2}.
\]

Introduce the zero-mean “noise” term
\(\epsilon_t:=\g_t-\nabla f(\W_t)\).  Conditioned on \(\W_t\),
\(\E[\epsilon_t]=0\) and
\(\E[\|\epsilon_t\|^2]\le\sigma^2\).
Using \eqref{eq:step_norm_bound} and \(\g_t=\nabla f(\W_t)+\epsilon_t\),
\begin{align*}
   \E\!\bigl[f(\W_{t+1})\mid\W_t\bigr]
   &\le
   f(\W_t)
   -\eta_t\|\nabla f(\W_t)\|^2
   +\frac{L\eta_t^{2}}2\,\E[\|\g_t\|^{2}\mid\W_t]\\
   &\le
   f(\W_t)
   -\eta_t\Bigl(1-\tfrac{L\eta_t}2\Bigr)\|\nabla f(\W_t)\|^{2}
   +\frac{L\eta_t^{2}}2\,\sigma^{2},
\end{align*}
where the last line used
\(
      \E[\|\g_t\|^{2}]
      =\|\nabla f(\W_t)\|^{2}
       +\E[\|\epsilon_t\|^{2}]
      \le \|\nabla f(\W_t)\|^{2}+\sigma^{2}.
\)
Because \(\eta_t\le1/L\),
\(1-L\eta_t/2\ge1/2\).  Hence
\begin{equation}\label{eq:one_step_descent}
   \E[f(\W_{t+1})]
   \;\le\;
   \E[f(\W_t)]
   -\tfrac{\eta_t}{2}\,
      \E\bigl[\|\nabla f(\W_t)\|^{2}\bigr]
   +\tfrac{L\eta_t^{2}}{2}\sigma^{2}.
\end{equation}

Sum \eqref{eq:one_step_descent} over \(t=0,\dots,T-1\)
and rearrange:
\[
      \frac1T\sum_{t=0}^{T-1}
      \E\bigl[\|\nabla f(\W_t)\|^{2}\bigr]
      \;\le\;
      \frac{2\bigl(\E[f(\W_0)]-f^{*}\bigr)}{\sum_{t=0}^{T-1}\eta_t}
      +L\sigma^{2}\,\frac{\sum_{t=0}^{T-1}\eta_t^{2}}
                             {\sum_{t=0}^{T-1}\eta_t}.
\]
Noting that
\(
      \sum_{t}\eta_t = T\bar\eta
\)
and
\(
      \sum_{t}\eta_t^{2} \le T\bar\eta^{2}
\)
(by Jensen’s inequality) yields
\[
      \frac1T\sum_{t=0}^{T-1}\E[\|\nabla f(\W_t)\|^{2}]
      \;\le\;
      \frac{2\bigl(f(\W_0)-f^{*}\bigr)}{\bar\eta\,T}
      +L\,\bar\eta\,\sigma^{2},
\]
which is precisely \eqref{eq:sgd_ppgd_rate}.

\end{proof}

\section{Ablations}
\label{sec:app_ablations}

To validate the utility of the proposed method, we stress-test it with different network and training configurations. All the models are trained on C4. 

\begin{figure*}[h]
    \centering
    \begin{subfigure}[b]{0.3\textwidth}
        \includegraphics[width=\textwidth]{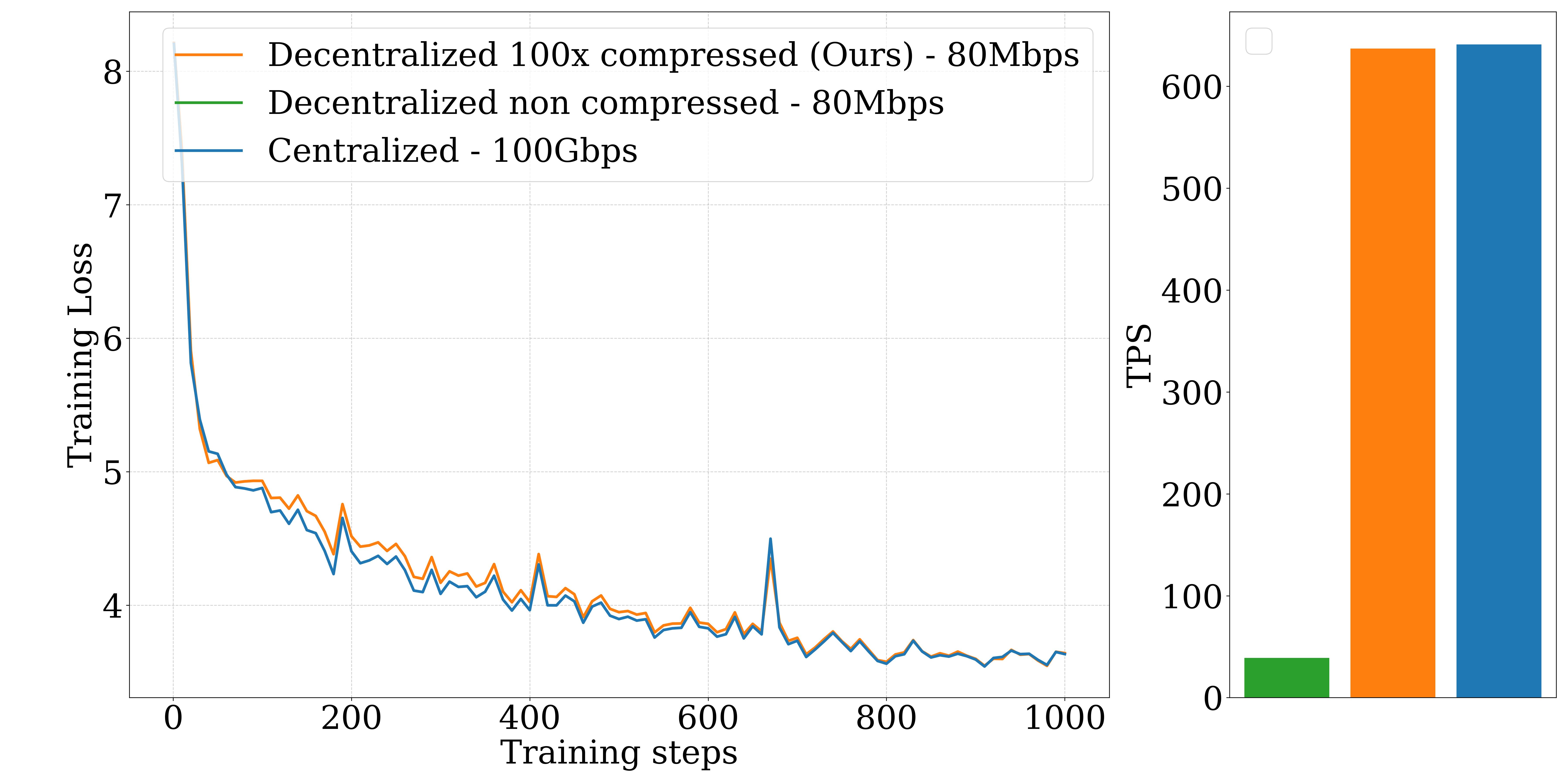}
        \caption{Batch size 8}
        \label{fig:sub1}
    \end{subfigure}
    \hfill
    \begin{subfigure}[b]{0.3\textwidth}
        \includegraphics[width=\textwidth]{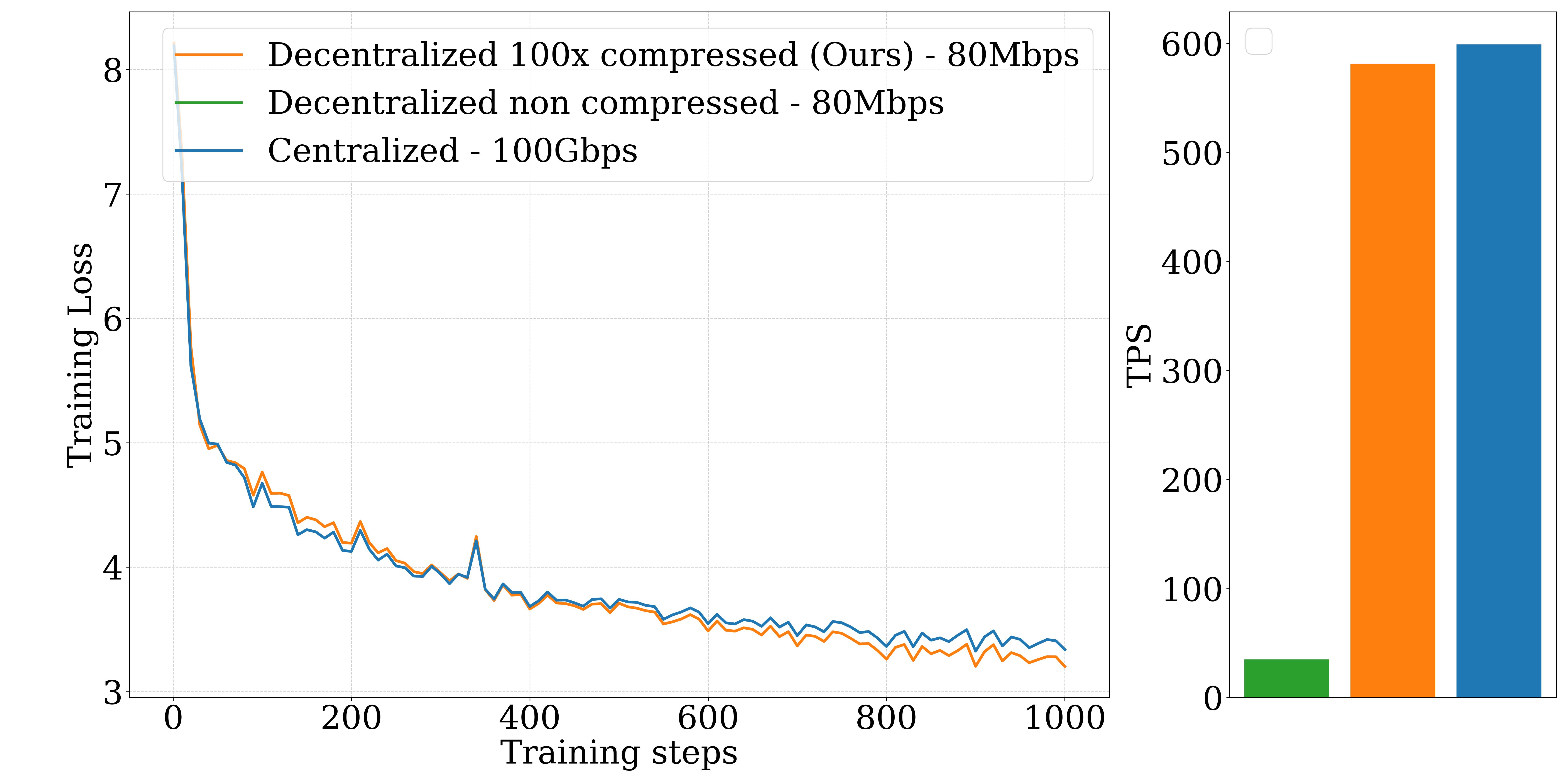}
        \caption{Batch size 16}
        \label{fig:sub2}
    \end{subfigure}
    \hfill
    \begin{subfigure}[b]{0.3\textwidth}
        \includegraphics[width=\textwidth]{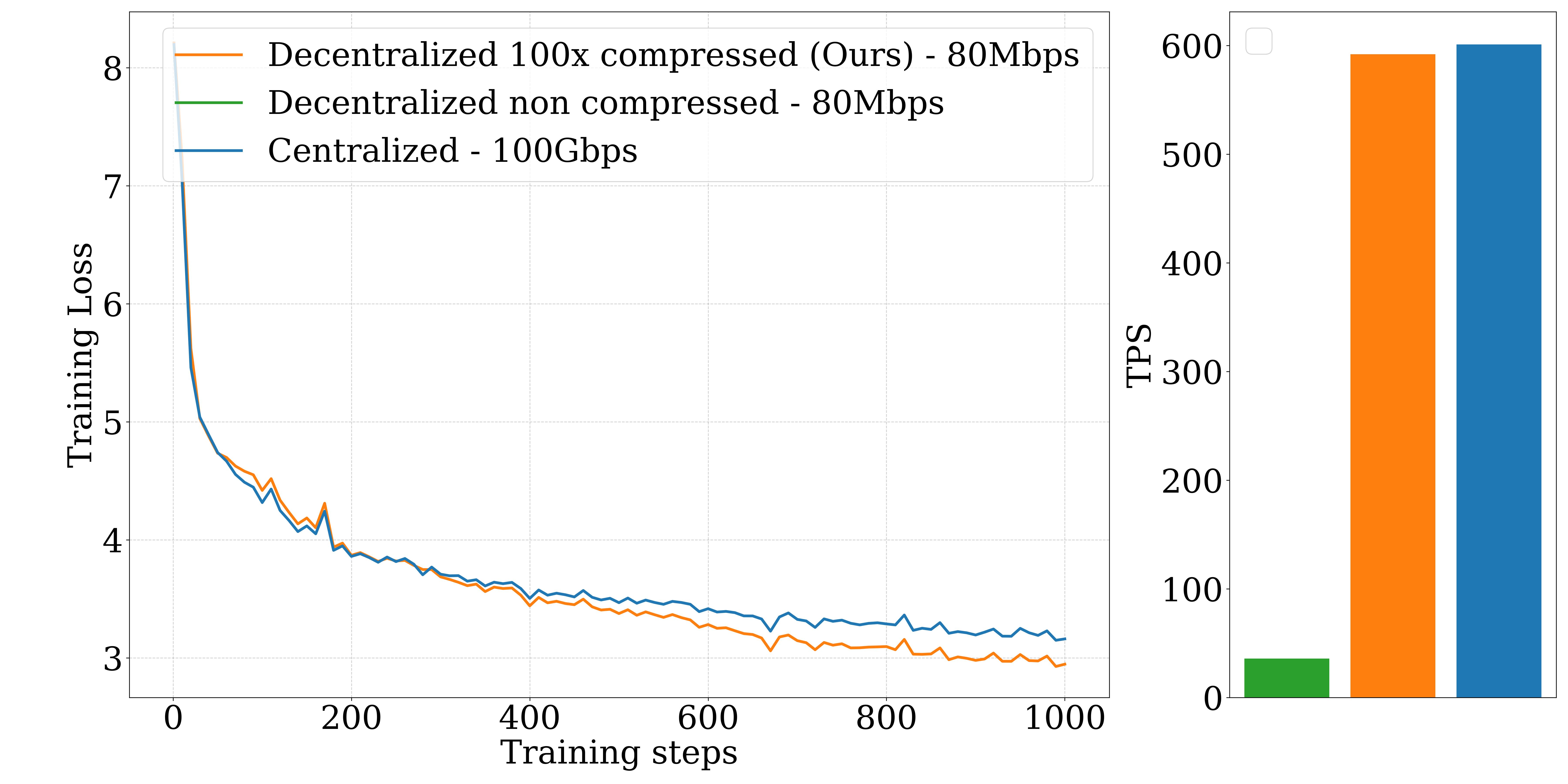}
        \caption{Batch size 32}
        \label{fig:sub3}
    \end{subfigure}
    \caption{\textbf{Convergence with different batch sizes.} Each block shows the training curves (with respect to training steps) and throughput for an 8-layer, 2B-parameter model on C4. Decentralized configurations use 80\,Mbps connections, while the centralized configuration uses datacenter-grade 100\,Gbps links. Even under a 80\,Mbps bandwidth budget, our compressed model achieves convergence comparable to---and sometimes exceeding---the centralized configuration across varying batch sizes. Note that as the batch size increases, the compressed model achieves increasingly better results compared to the centralized model. The iteration-wise dynamics of the uncompressed decentralized model match those of the centralized model; hence, we omit its curve for clarity. Despite severe bandwidth constraints, our compressed model attains throughput on par with the centralized setting.  In contrast, the uncompressed decentralized setup suffers significantly lower throughput due to communication bottlenecks. All models share the same network architecture. }
    \label{fig:compression_abl_train}
\end{figure*}

\begin{figure*}[h]
    \centering
    \begin{subfigure}[b]{0.3\textwidth}
        \includegraphics[width=\textwidth]{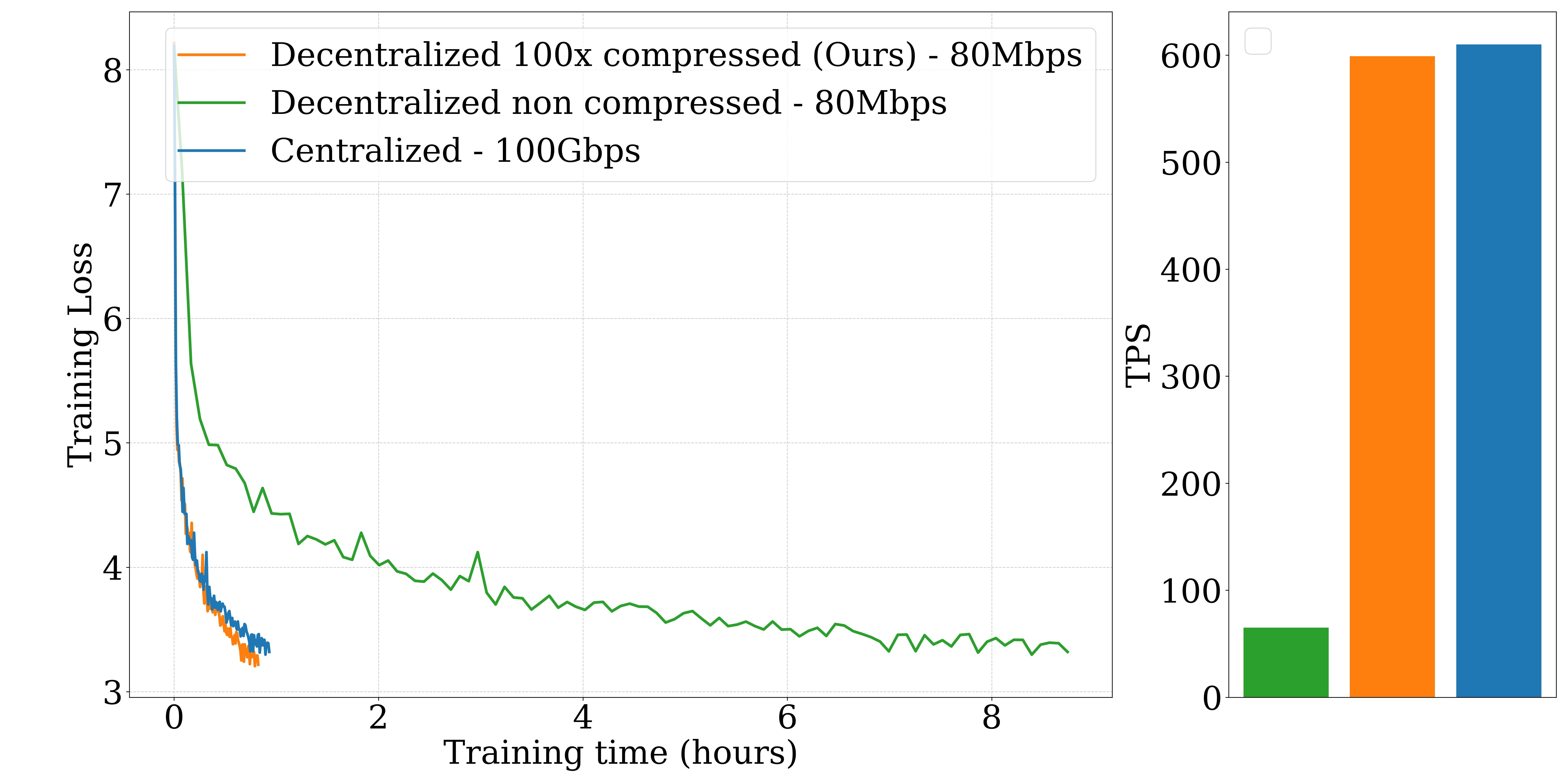}
        \caption{Batch size 8}
        \label{fig:sub1}
    \end{subfigure}
    \hfill
    \begin{subfigure}[b]{0.3\textwidth}
        \includegraphics[width=\textwidth]{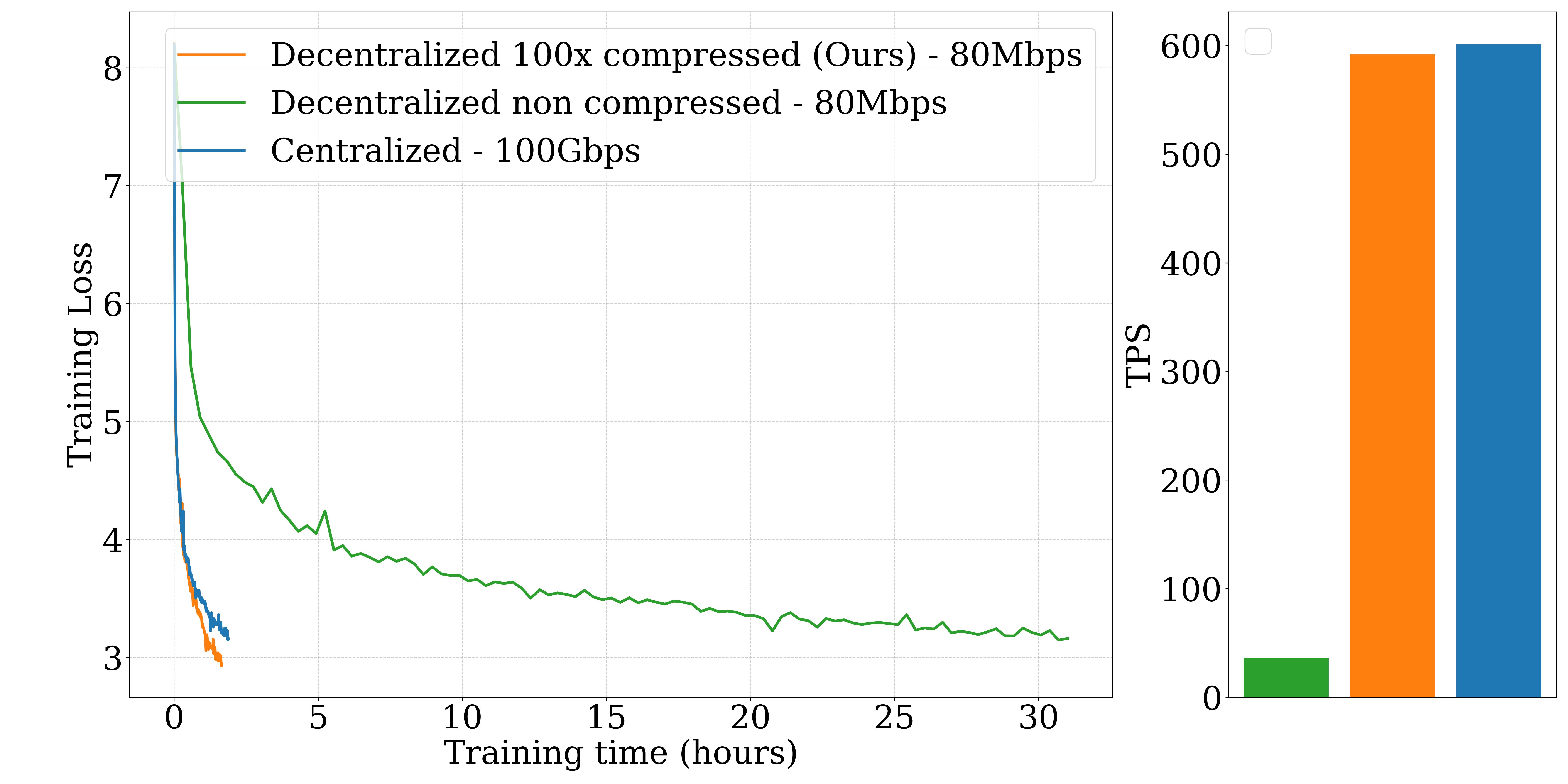}
        \caption{Batch size 16}
        \label{fig:sub2}
    \end{subfigure}
    \hfill
    \begin{subfigure}[b]{0.3\textwidth}
        \includegraphics[width=\textwidth]{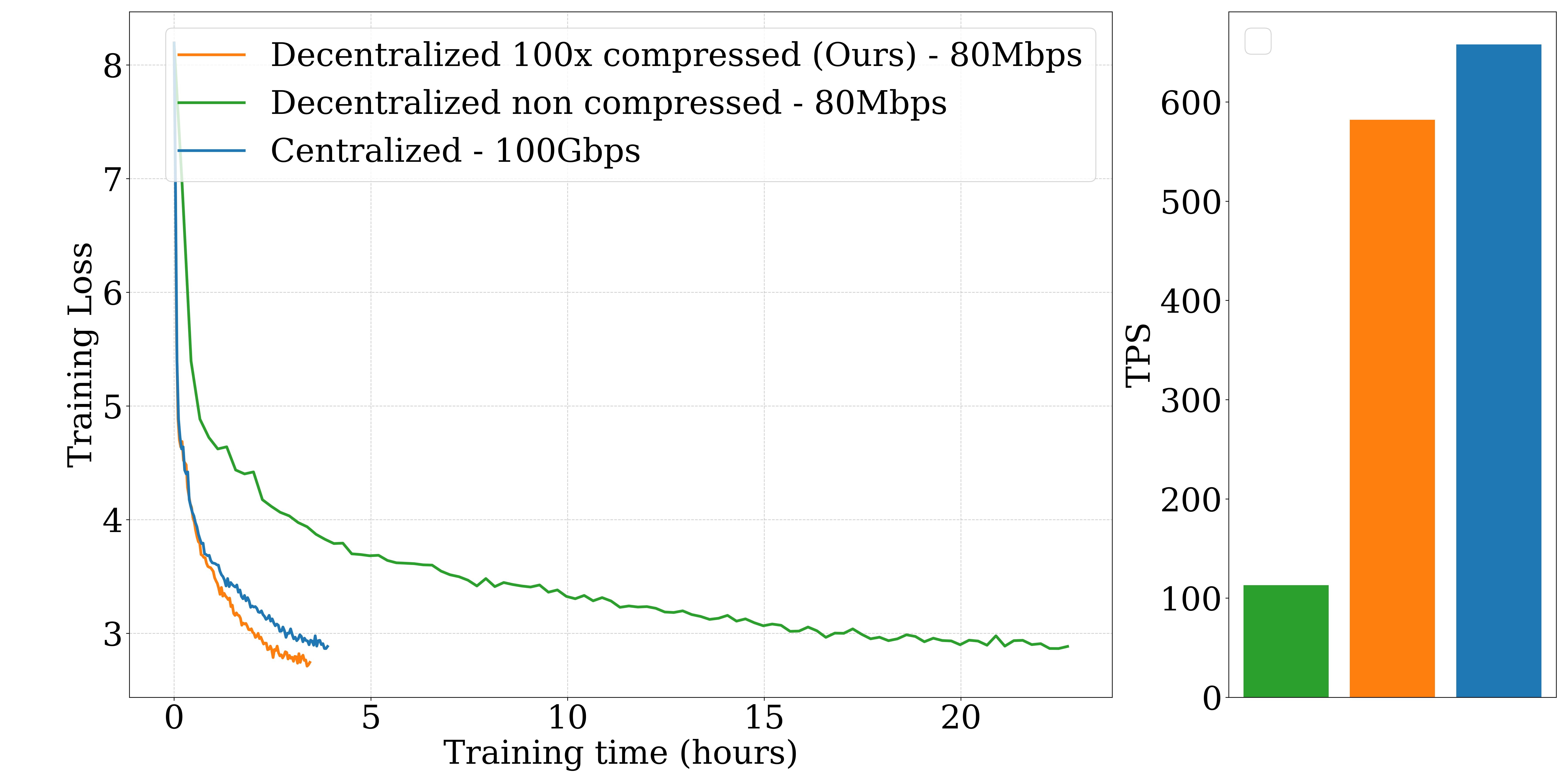}
        \caption{Batch size 32}
        \label{fig:sub3}
    \end{subfigure}
    \caption{\textbf{Convergence with different batch sizes.} Each block shows the training curves (with respect to wall-clock time) and throughput for an 8-layer, 2B-parameter model on C4. Decentralized configurations use 80\,Mbps connections, while the centralized configuration uses datacenter-grade 100\,Gbps links. Even under a 80\,Mbps bandwidth budget, our compressed model achieves convergence comparable to---and sometimes exceeding---the centralized configuration across varying batch sizes while the decentralized non compressed model demonstrates extremely slow convergence.  }
    \label{fig:compression_abl_train}
\end{figure*}

\begin{figure*}[h]
    \centering
    \begin{subfigure}[b]{0.3\textwidth}
        \includegraphics[width=\textwidth]{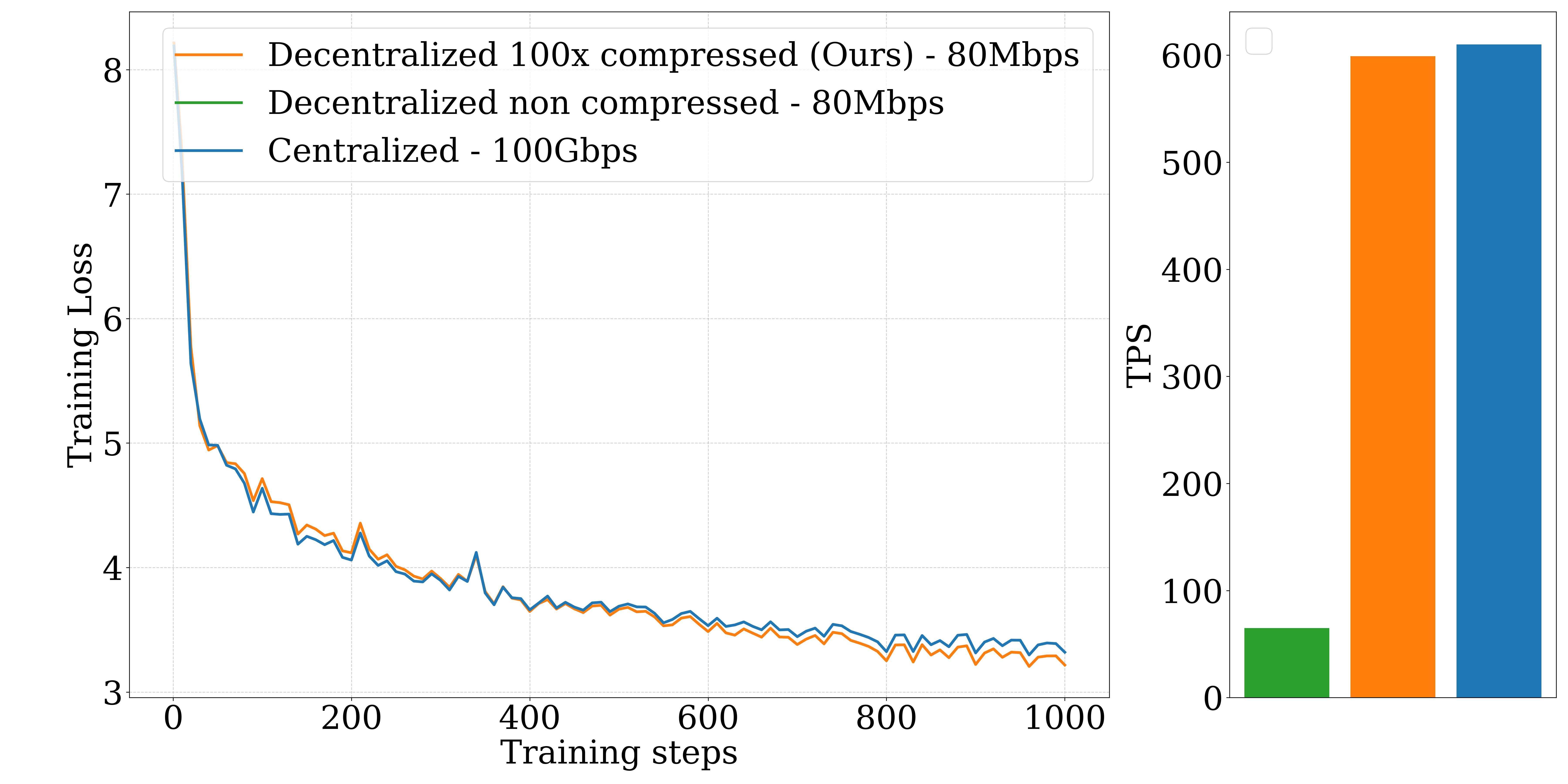}
        \caption{Context length 512}
        \label{fig:sub1}
    \end{subfigure}
    \hfill
    \begin{subfigure}[b]{0.3\textwidth}
        \includegraphics[width=\textwidth]{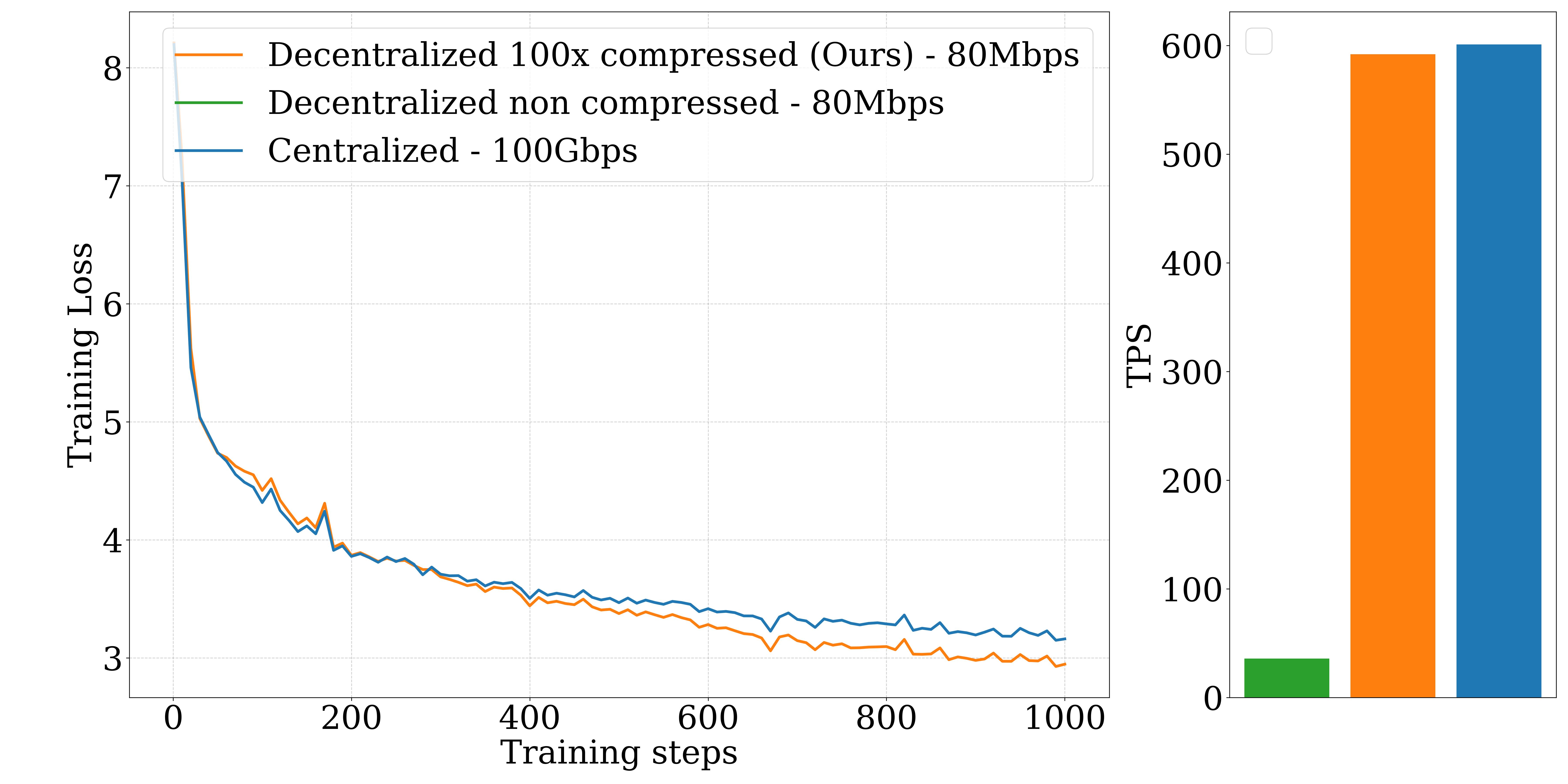}
        \caption{Context length 1024}
        \label{fig:sub2}
    \end{subfigure}
    \hfill
    \begin{subfigure}[b]{0.3\textwidth}
        \includegraphics[width=\textwidth]{figures/abl_seq/batch32seq1024_layers8.jpg}
        \caption{Context length 2048}
        \label{fig:sub3}
    \end{subfigure}
    \caption{\textbf{Convergence with different context lengths.} Each block shows the training curves (with respect to training steps) and throughput for an 8-layer, 2B-parameter model on C4. Decentralized configurations use 80\,Mbps connections, while the centralized configuration uses datacenter-grade 100\,Gbps links. Even under a 80\,Mbps bandwidth budget, our compressed model achieves convergence comparable to---and sometimes exceeding---the centralized configuration across varying context lengths.  Note that as the context length increases, the compressed model achieves increasingly better results compared to the centralized model. The iteration-wise dynamics of the uncompressed decentralized model match those of the centralized model; hence, we omit its curve for clarity. Despite severe bandwidth constraints, our compressed model attains throughput on par with the centralized setting. In contrast, the uncompressed decentralized setup suffers significantly lower throughput due to communication bottlenecks. }
    \label{fig:compression_abl_train}
\end{figure*}

\begin{figure*}[ht]
    \centering
    \begin{subfigure}[b]{0.3\textwidth}
        \includegraphics[width=\textwidth]{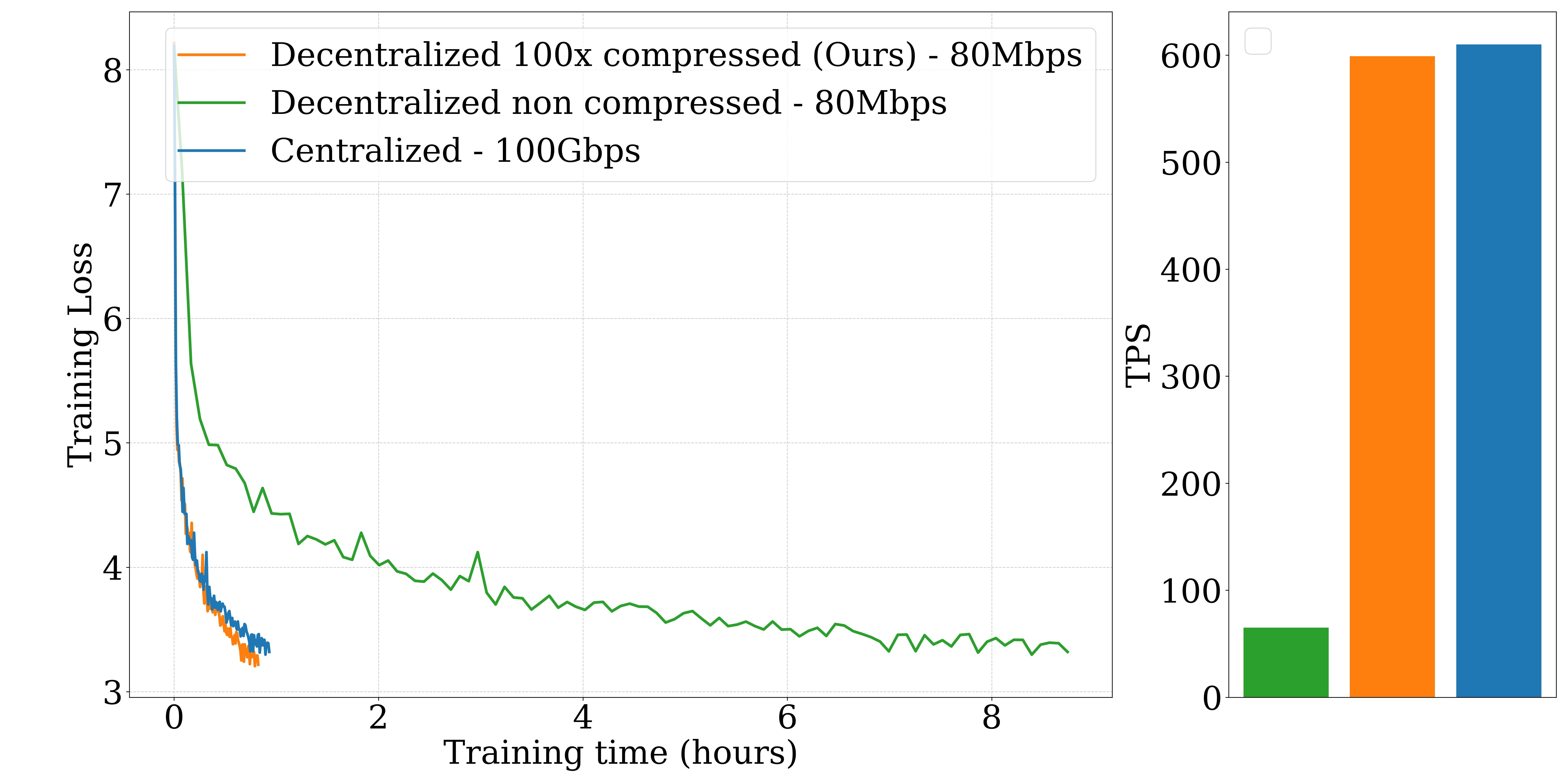}
        \caption{Context length 512}
        \label{fig:sub1}
    \end{subfigure}
    \hfill
    \begin{subfigure}[b]{0.3\textwidth}
        \includegraphics[width=\textwidth]{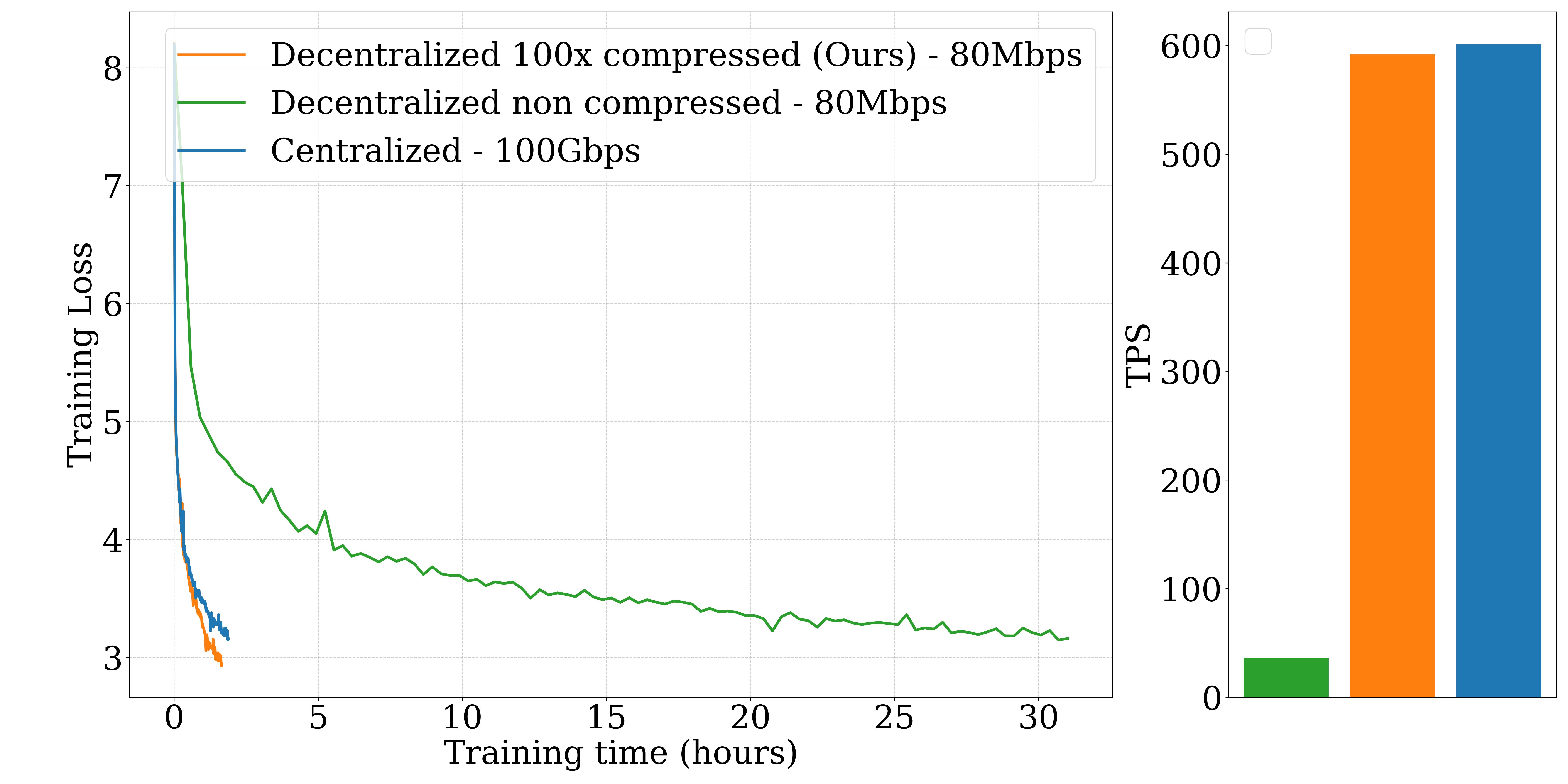}
        \caption{Context length 1024}
        \label{fig:sub2}
    \end{subfigure}
    \hfill
    \begin{subfigure}[b]{0.3\textwidth}
        \includegraphics[width=\textwidth]{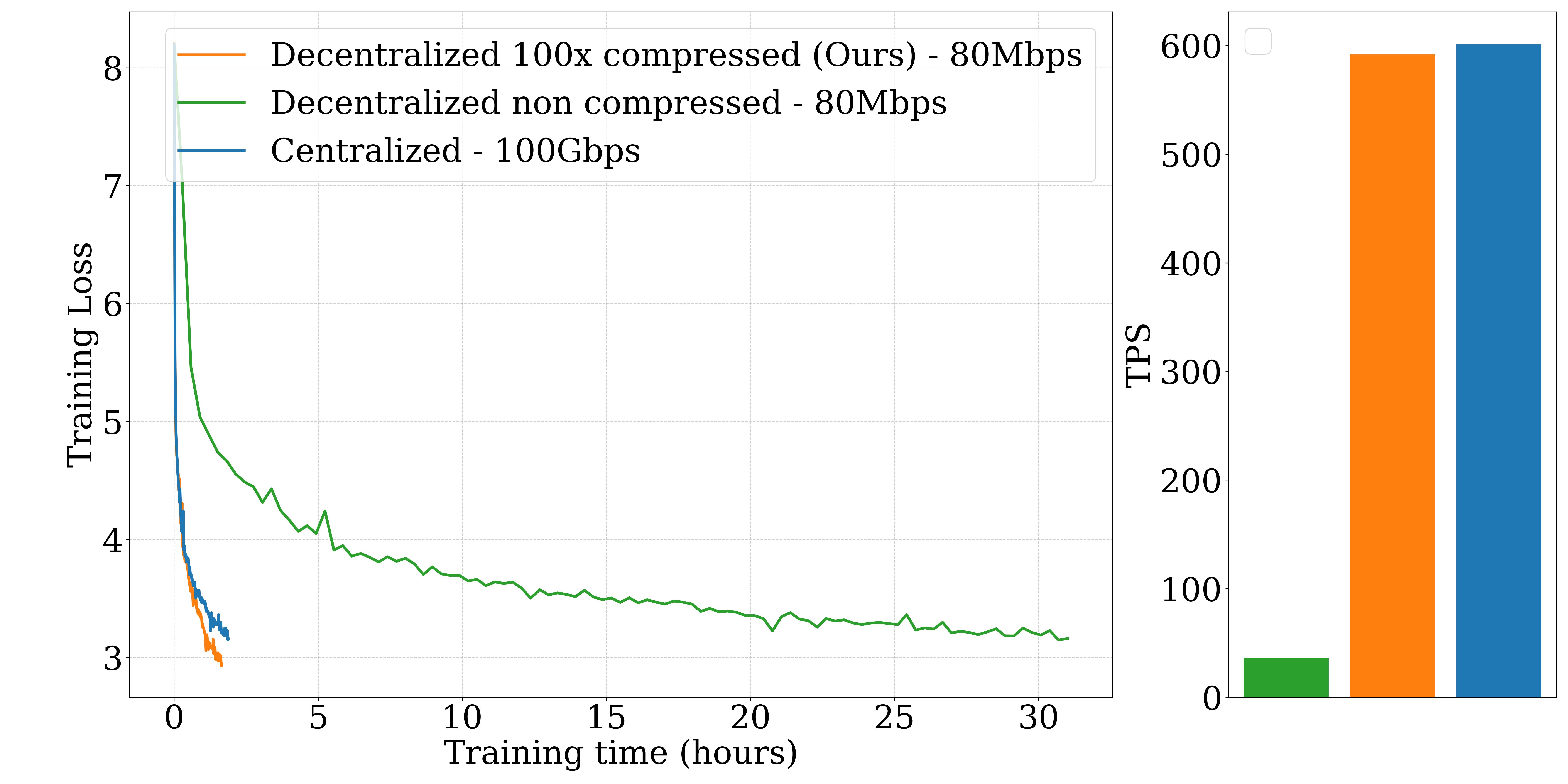}
        \caption{Context length 2048}
        \label{fig:sub3}
    \end{subfigure}
    \caption{\textbf{Convergence with different context lengths.} Each block shows the training curves (with respect to wall-clock time) and throughput for an 8-layer, 2B-parameter model on C4. Decentralized configurations use 80\,Mbps connections, while the centralized configuration uses datacenter-grade 100\,Gbps links. Even under a 80\,Mbps bandwidth budget, our compressed model achieves convergence comparable to---and sometimes exceeding---the centralized configuration across varying context lengths while the decentralized non compressed model demonstrates extremely slow convergence.  }
    \label{fig:compression_abl_train}
\end{figure*}

\begin{figure*}[ht]
    \centering
    \begin{subfigure}[b]{0.45\textwidth}
        \includegraphics[width=\textwidth]{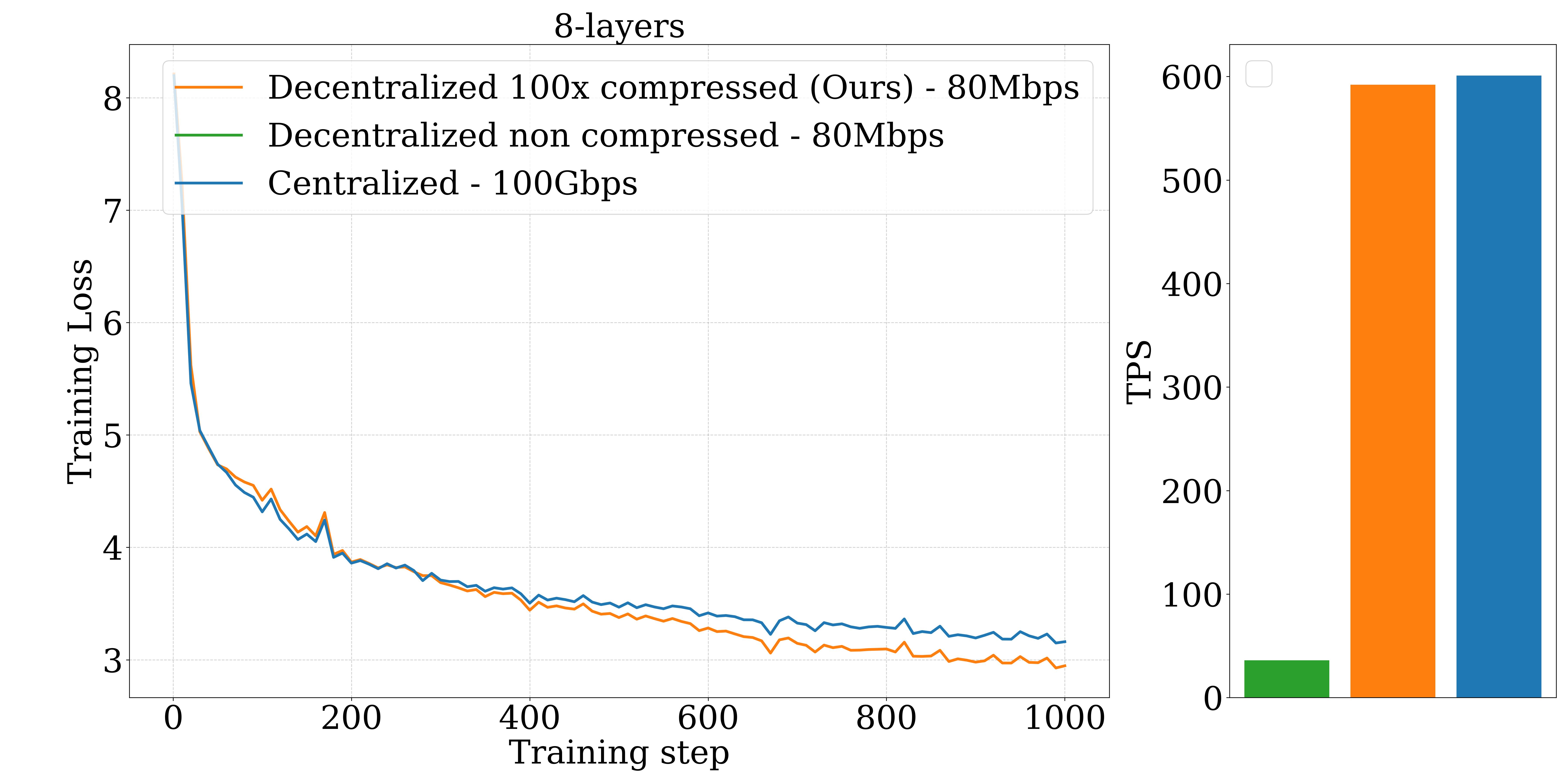}
        \caption{8 layers}
        \label{fig:sub2}
    \end{subfigure}
     \hfill
    \begin{subfigure}[b]{0.45\textwidth}
        \includegraphics[width=\textwidth]{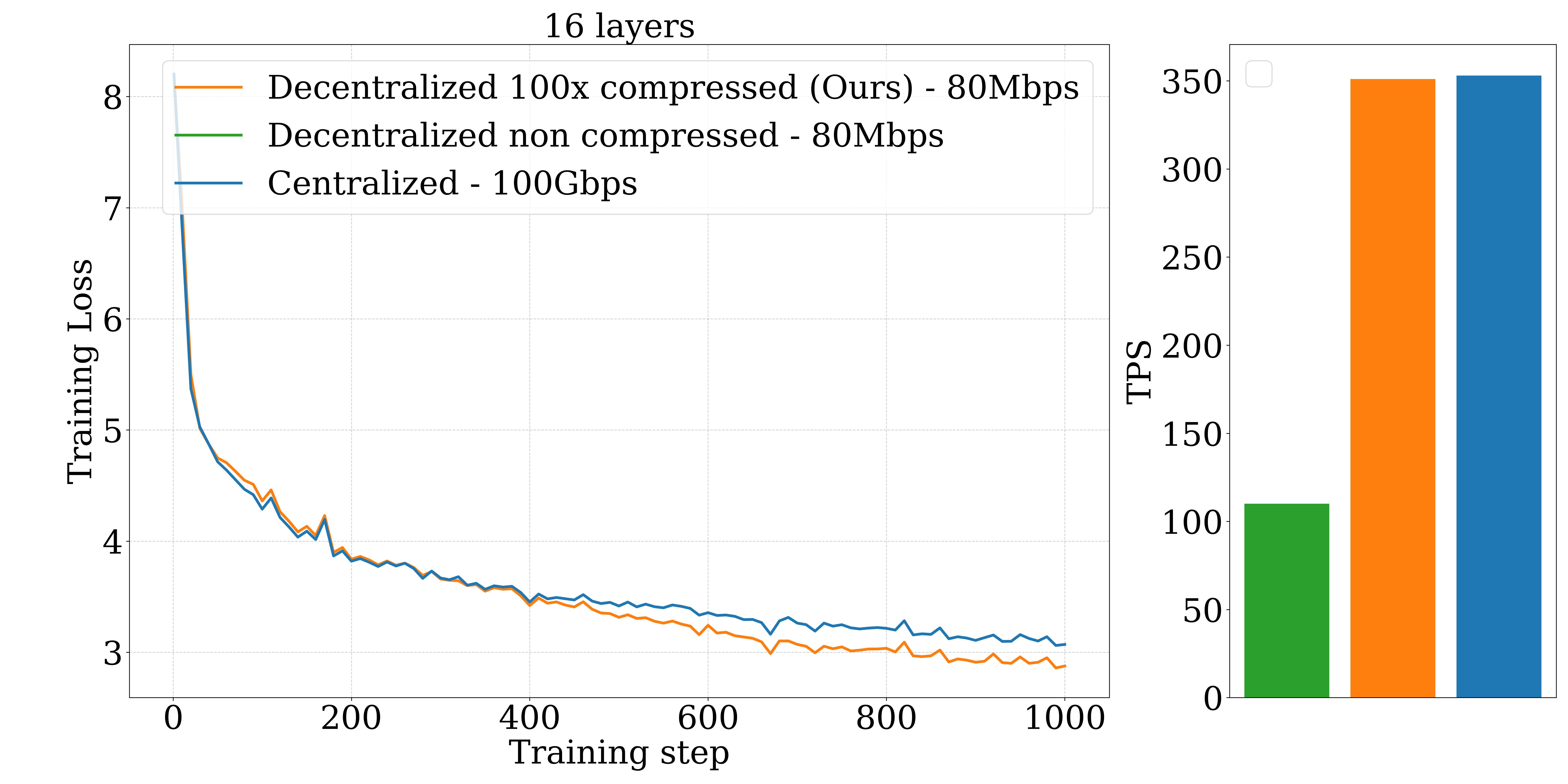}
        \caption{16 layers (2 layers per GPU)}
        \label{fig:sub1}
    \end{subfigure}   
    \caption{\textbf{Convergence with increasing number of layers.} Note that as the number of layer increase, our model consistently matches (even slightly exceeds) the centralized model. This is in stark contrast to lossy compression schemes, where the model convergence severely degrades as the model depth increases \cite{bian2024does}.}
    \label{fig:compression_abl_train}
\end{figure*}

\begin{figure*}[ht]
    \centering
    \begin{subfigure}[b]{0.45\textwidth}
        \includegraphics[width=\textwidth]{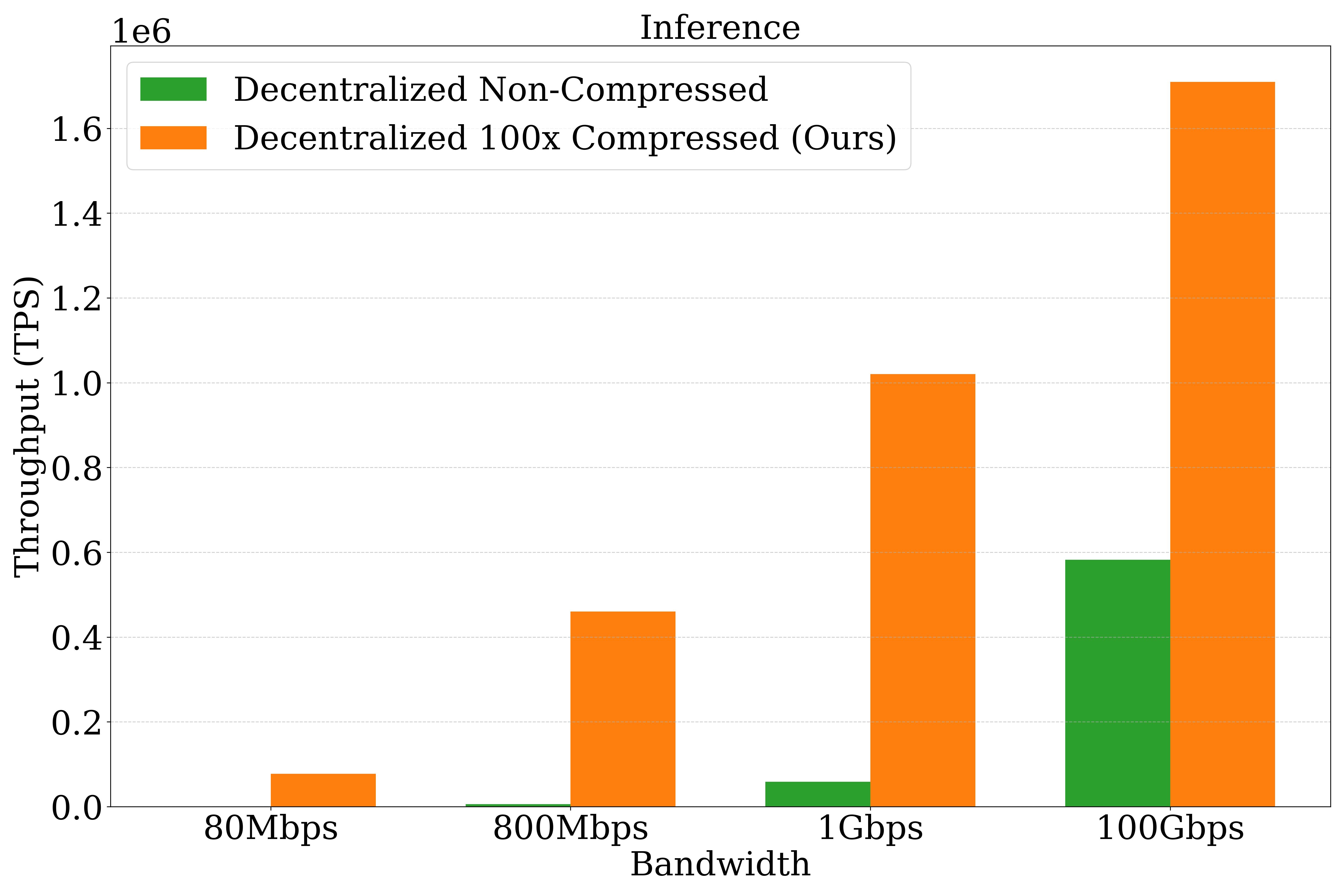}
        \caption{Inference}
        \label{fig:sub1}
    \end{subfigure}
    \hfill
    \begin{subfigure}[b]{0.45\textwidth}
        \includegraphics[width=\textwidth]{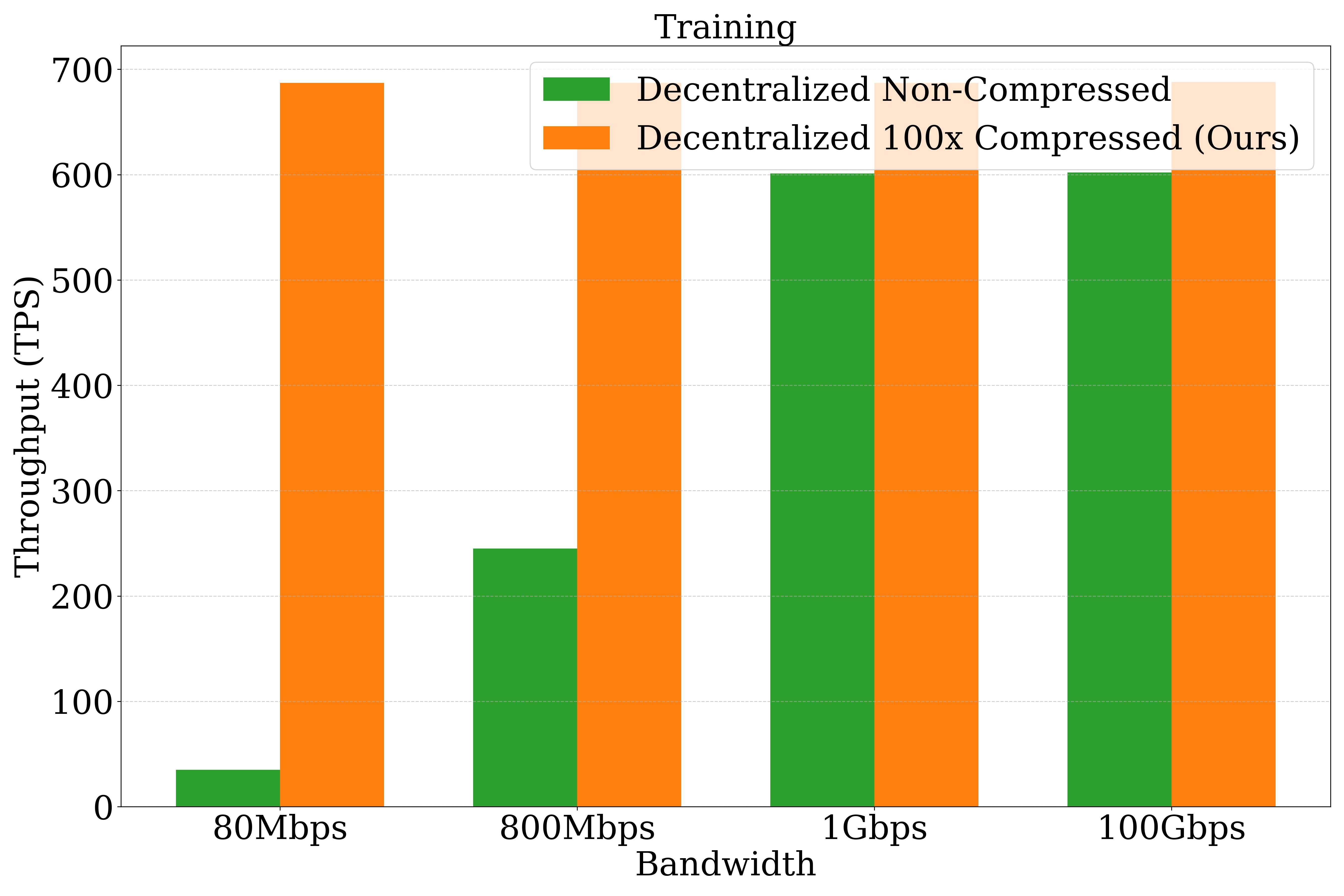}
        \caption{Training}
        \label{fig:sub2}
    \end{subfigure}
    \caption{\textbf{Throughput across bandwidth constraints.} We limit network bandwidth between GPUs and measure throughput during inference and training. The compressed model consistently achieves significantly higher throughput than the non-compressed model. Notably, even at 100Gbps, compression improves inference throughput by 3×, demonstrating benefits for even  centralized setups.
  }
    \label{fig:compression_abl_train}
\end{figure*}

\begin{figure}
    \centering
    \includegraphics[width=0.5\linewidth]{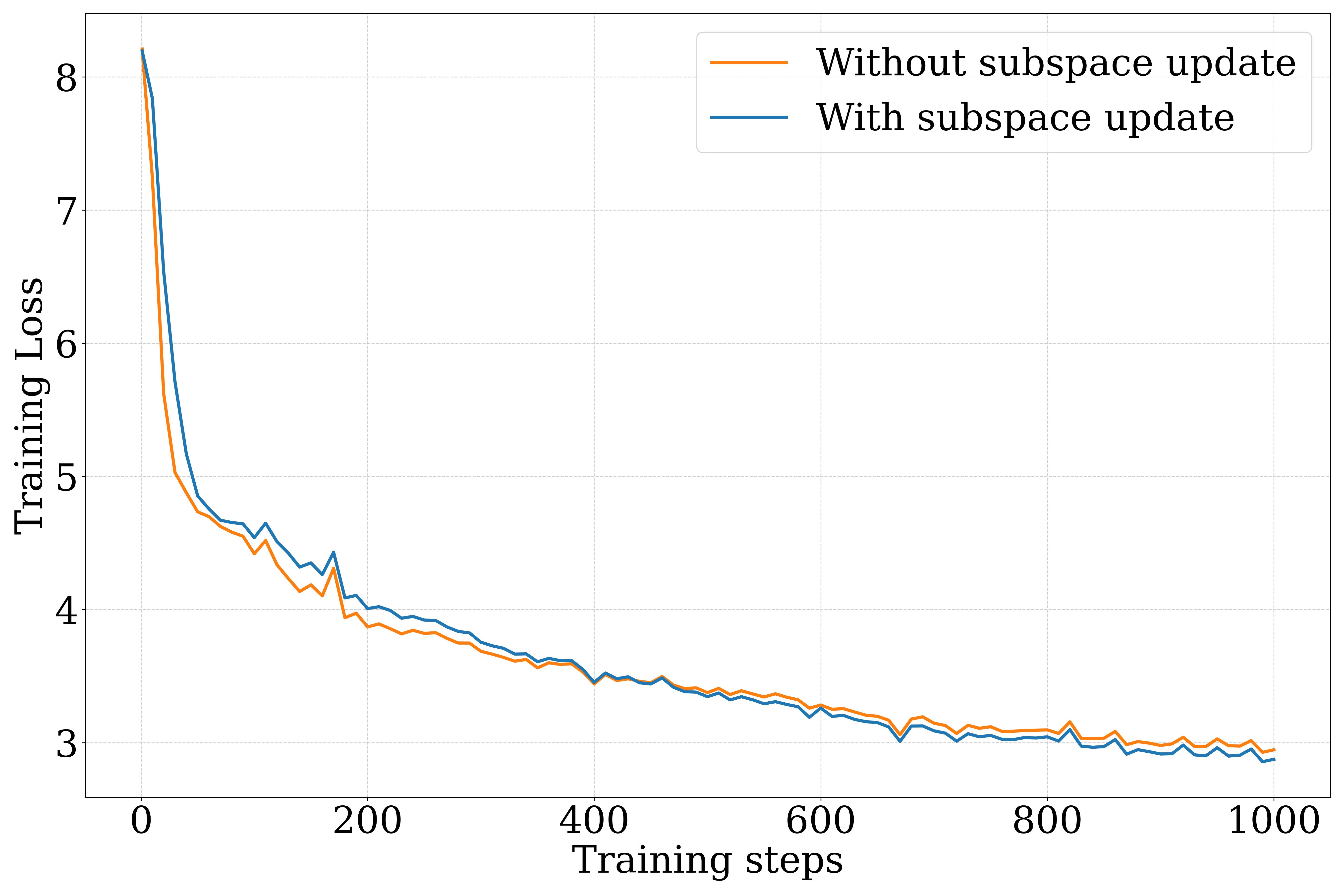}
    \caption{\textbf{Effect of the (Grassman) subspace updates}. An 8-layer model is trained on C4 for this experiment. Since this performance gap seems to keep increasing towards the end of training, we emphasize the importance of infrequent subspace updates.}
    \label{fig:enter-label}
\end{figure}

\begin{figure}
    \centering
    \includegraphics[width=0.5\linewidth]{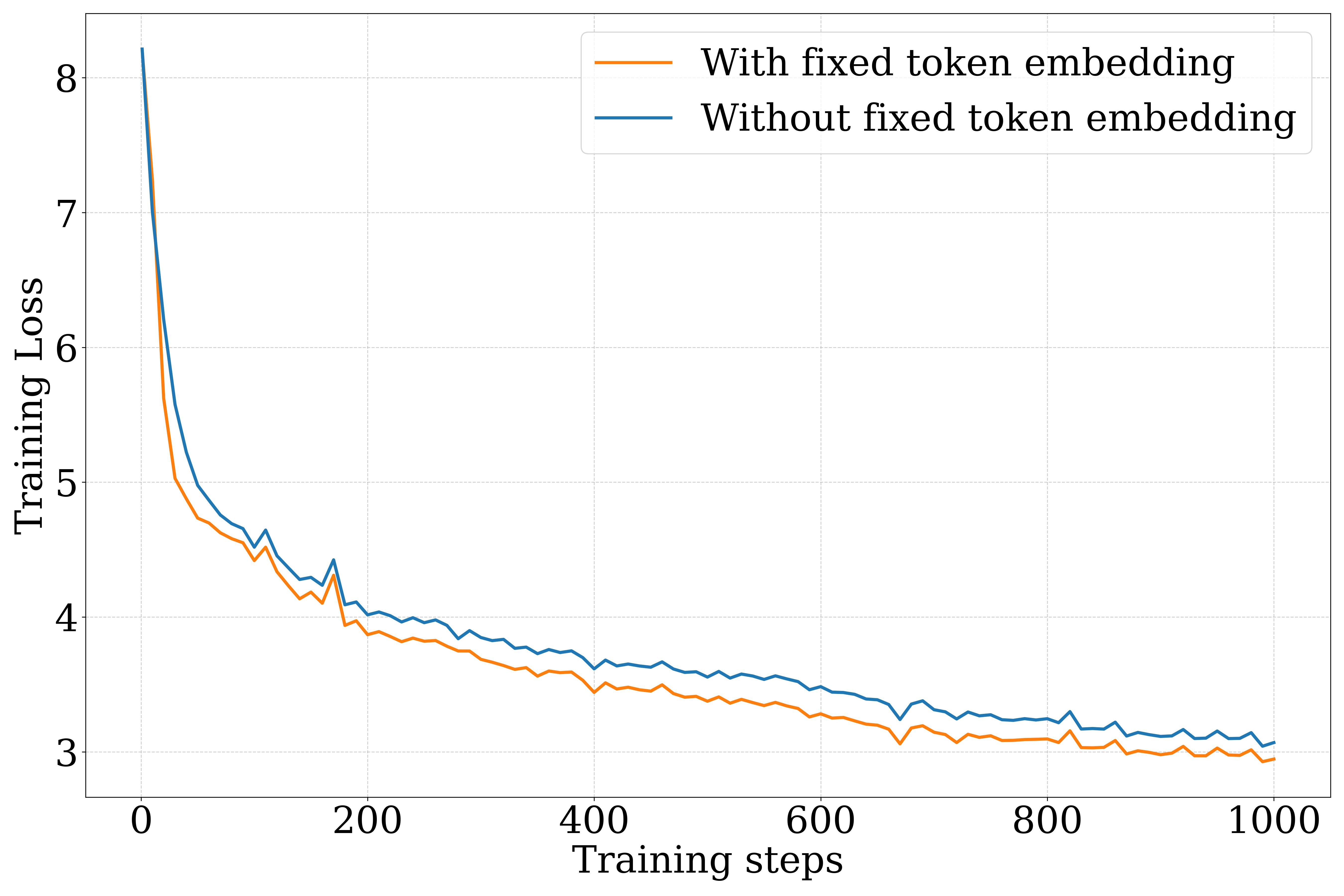}
    \caption{\textbf{Effect of fixed token embedding.} We decompose the token embeddings to a fixed high rank embedding and a dynamic low rank embedding. As shown in the figure, we observe inferior convergence when there is no such decomposition.  An 8-layer model is trained on C4 for this experiment.}
    \label{fig:enter-label}
\end{figure}

\section{Analysis of pretrained checkpoints}
\label{app:checkpoints}

We investigate the stable ranks of the output projection layers of the official checkpoints of frontier open-weight LLMs (LlaMA, Qwen, Olmo, Phi). The observations are reported in Fig.~\ref{fig:low-rank-analysis}. As shown, the weights demonstrate extremely low ranks across all the layers and models.

\begin{figure}[th]
  \centering

  \begin{minipage}[b]{0.3\linewidth}
    \centering
    \includegraphics[width=\linewidth]{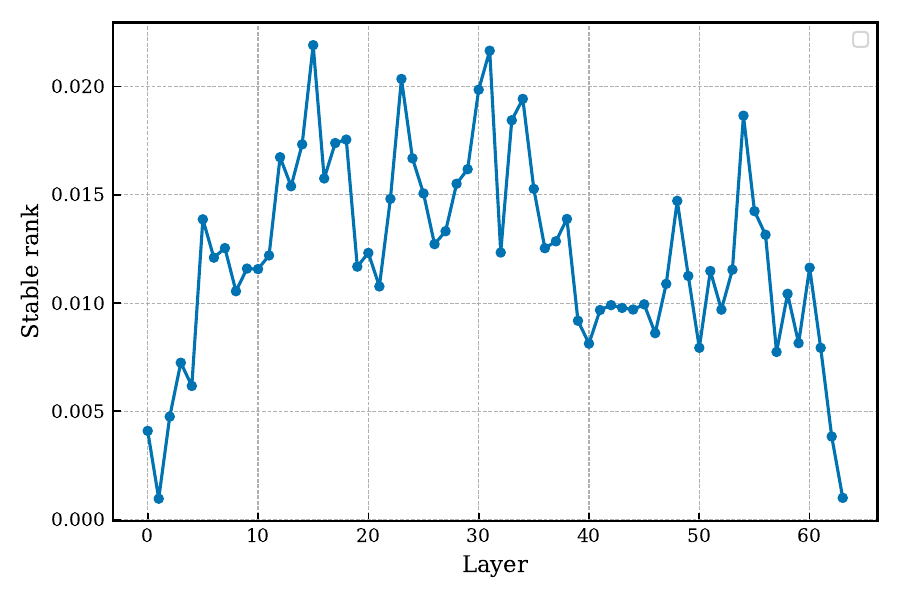}
    \caption*{\textsc{DeepSeek-R1-670B}}
  \end{minipage}
  \hfill
  \begin{minipage}[b]{0.3\linewidth}
    \centering
    \includegraphics[width=\linewidth]{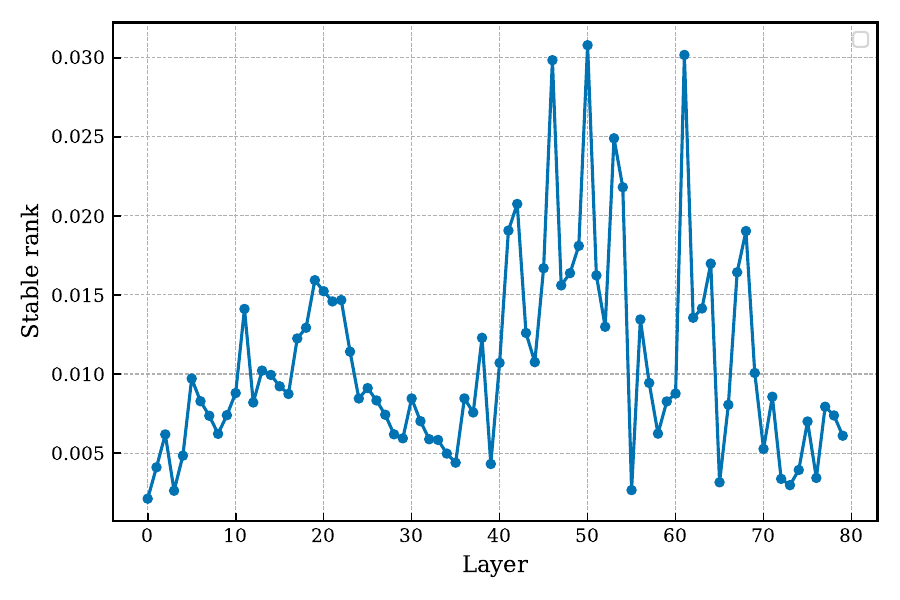}
    \caption*{\textsc{LlaMa3-70B}}
  \end{minipage}
  \hfill
  \begin{minipage}[b]{0.3\linewidth}
    \centering
    \includegraphics[width=\linewidth]{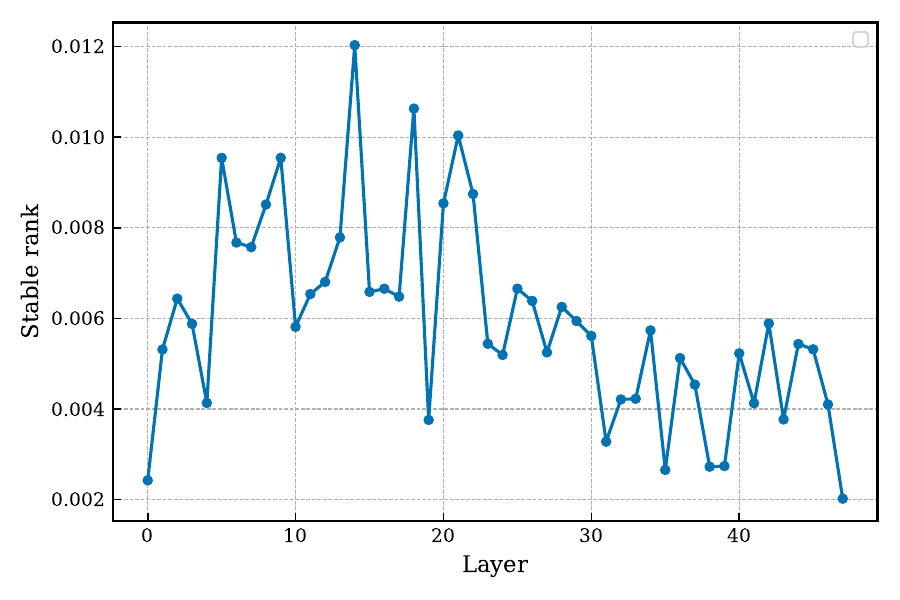}
    \caption*{\textsc{LlaMa4-Scout-17B}}
  \end{minipage}

  \vspace{0.5em}  

  \begin{minipage}[b]{0.3\linewidth}
    \centering
    \includegraphics[width=\linewidth]{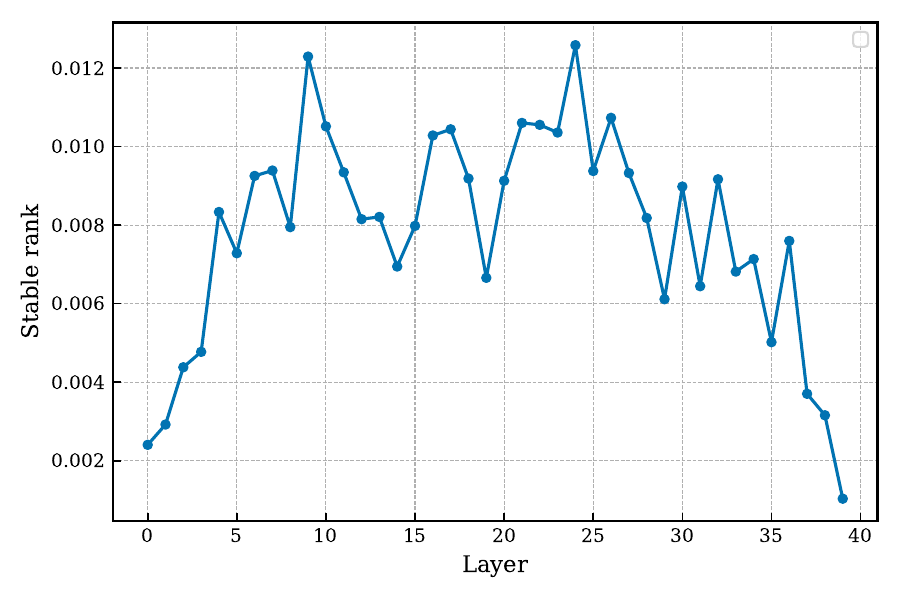}
    \caption*{\textsc{Phi-4-14B}}
  \end{minipage}
  \hfill
  \begin{minipage}[b]{0.3\linewidth}
    \centering
    \includegraphics[width=\linewidth]{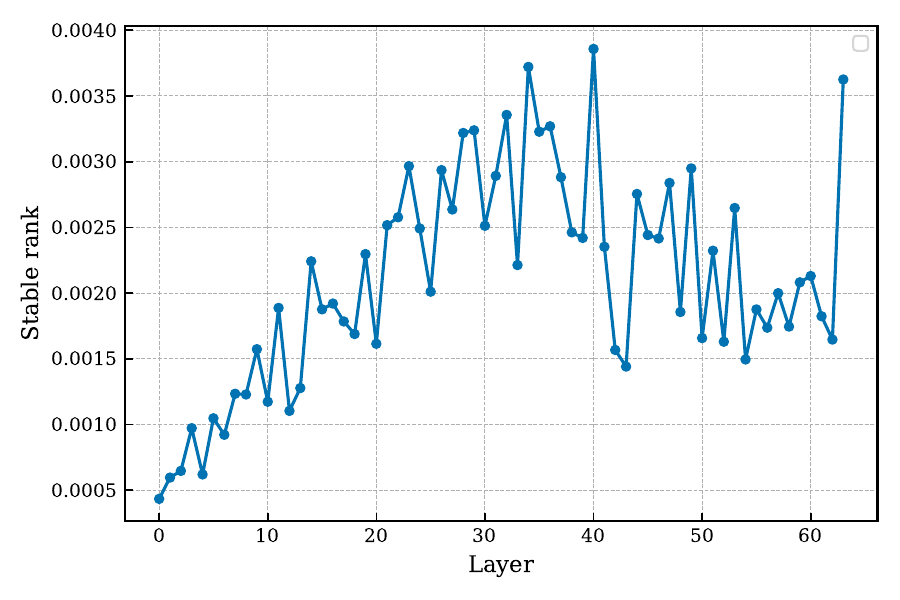}
    \caption*{\textsc{Olmo-2-32B}}
  \end{minipage}
  \hfill
  \begin{minipage}[b]{0.3\linewidth}
    \centering
    \includegraphics[width=\linewidth]{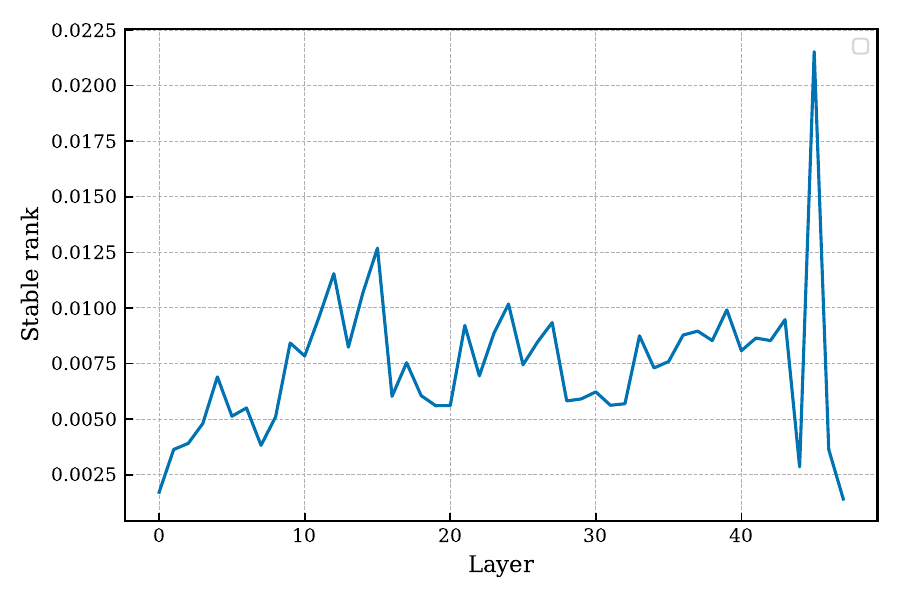}
    \caption*{\textsc{Qwen3-30B}}
  \end{minipage}

  \caption{\textbf{Stable ranks of output projection matrices (normalized by the maximum possible rank) across different frontier models.} Statistics are computed on official fully pre-trained checkpoints. All the models  demonstrate extremely low ranks, solidifying our theoretical arguments. }
  \label{fig:low-rank-analysis}
\end{figure}

\end{document}

%% file: math_commands.tex
\usepackage{amsmath,amsfonts,bm}

\def\eqref#1{equation~\ref{#1}}

\def\1{\bm{1}}

\DeclareMathAlphabet{\mathsfit}{\encodingdefault}{\sfdefault}{m}{sl}
\SetMathAlphabet{\mathsfit}{bold}{\encodingdefault}{\sfdefault}{bx}{n}

\newcommand{\E}{\mathbb{E}}

\newcommand{\R}{\mathbb{R}}

